\documentclass{article}
\usepackage[utf8]{inputenc} 
\usepackage[T1]{fontenc}    
\usepackage{booktabs}       
\usepackage{nicefrac}       
\usepackage{microtype}      
\usepackage{times}             

\usepackage[round]{natbib}
\usepackage{amssymb,amsmath,amsthm,bbm}
\usepackage[margin=1in]{geometry}
\usepackage{verbatim,float,url,dsfont}
\usepackage{graphicx,subfigure,psfrag}
\usepackage{algorithm,algorithmic}
\usepackage{mathtools,enumitem}
\usepackage{xr-hyper}
\usepackage[colorlinks=true,citecolor=blue,urlcolor=blue,linkcolor=blue]{hyperref}
\usepackage{multirow}

\newtheorem{theorem}{Theorem}
\newtheorem{lemma}{Lemma}
\newtheorem{corollary}{Corollary}
\newtheorem{proposition}{Proposition}
\theoremstyle{definition}
\newtheorem{remark}{Remark}

\newtheorem*{assumption*}{\assumptionnumber}
\providecommand{\assumptionnumber}{}


\makeatletter
\newcommand*\rel@kern[1]{\kern#1\dimexpr\macc@kerna}
\newcommand*\widebar[1]{%
  \begingroup
  \def\mathaccent##1##2{%
    \rel@kern{0.8}%
    \overline{\rel@kern{-0.8}\macc@nucleus\rel@kern{0.2}}%
    \rel@kern{-0.2}%
  }%
  \macc@depth\@ne
  \let\math@bgroup\@empty \let\math@egroup\macc@set@skewchar
  \mathsurround\z@ \frozen@everymath{\mathgroup\macc@group\relax}%
  \macc@set@skewchar\relax
  \let\mathaccentV\macc@nested@a
  \macc@nested@a\relax111{#1}%
  \endgroup
}
\makeatother

\newcommand{\argmin}{\mathop{\mathrm{argmin}}}

\newcommand{\minimize}{\mathop{\mathrm{minimize}}}
\newcommand{\st}{\mathop{\mathrm{subject\,\,to}}}

\def\R{\mathbb{R}}

\def\E{\mathbb{E}}
\def\P{\mathbb{P}}
\def\T{\mathsf{T}}
\def\Cov{\mathrm{Cov}}

\def\half{\frac{1}{2}}

\def\df{\mathrm{df}}

\def\nul{\mathrm{null}}

\def\nuli{\mathrm{nullity}}

\def\supp{\mathrm{supp}}
\def\diag{\mathrm{diag}}

\def\hf{\hat{f}}

\def\halpha{\hat{\alpha}}

\def\htheta{\hat{\theta}}

\def\cK{\mathcal{K}}
\def\cH{\mathcal{H}}

\def\cN{\mathcal{N}}

\def\cS{\mathcal{S}}
\def\cT{\mathcal{T}}
\def\cW{\mathcal{W}}
\def\cX{\mathcal{X}}

\usepackage{caption}
\usepackage{mathtools}
\usepackage{ragged2e}
\newcommand{\prox}{\mathsf{prox}}

\def\S{\mathbb{S}}
\def\cC{\mathcal{C}}

\def\th{^{\text{th}}}
\def\TV{\mathrm{TV}}
\def\BV{\mathrm{BV}}
\def\KTV{\mathrm{KTV}}
\def\Risk{\mathrm{Risk}}

\def\one{\mathds{1}}

\def\dop{\Delta}
\def\dmat{D}
\def\hmat{H}
\def\genmat{D}
\def\kronmat{D_{n,d}}
\def\graphmat{S_{n,d}}
\def\dmatk{\dmat^{(k+1)}}
\def\hmatk{\hmat^{(k+1)}}
\def\kronmatk{\kronmat^{(k+1)}} 
\def\graphmatk{\graphmat^{(k+1)}} 

\def\kronset{\cT_{n,d}}
\def\sobolset{\cW_{n,d}}
\def\holderset{\cC_{n,d}}
\def\graphset{\cS_{n,d}}

\makeatletter
\newcommand{\vast}{\bBigg@{4}}
\newcommand{\Vast}{\bBigg@{5}}
\newcommand{\VAST}{\bBigg@{6}}
\makeatother

\usepackage{xcolor}

\title{Multivariate Trend Filtering for Lattice Data} 
\author{
Veeranjaneyulu Sadhanala$^a$ \and
Yu-Xiang Wang$^b$ \and 
Addison J. Hu$^c$ \and 
Ryan J. Tibshirani$^c$}
\date{\normalsize
$^a$Google \quad
$^b$University of California Santa Barbara \quad
$^c$Carnegie Mellon University}

\begin{document}
\maketitle

\vspace{-10pt}
\begin{abstract}
We study a multivariate version of trend filtering, called Kronecker trend
filtering or KTF, for the case in which the design points form a lattice in
$d$ dimensions. KTF is a natural extension of univariate trend filtering
\citep{steidl2006splines, kim2009trend, tibshirani2014adaptive}, and is defined
by minimizing a penalized least squares problem whose penalty term sums the 
absolute (higher-order) differences of the parameter to be estimated along each
of the coordinate directions. The corresponding penalty operator can be
written in terms of Kronecker products of univariate trend filtering penalty
operators, hence the name Kronecker trend filtering. Equivalently, one can view
KTF in terms of an $\ell_1$-penalized basis regression problem where the basis
functions are tensor products of falling factorial functions, a piecewise
polynomial (discrete spline) basis that underlies univariate trend filtering.

This paper is a unification and extension of the results in
\citet{sadhanala2016total, sadhanala2017higher}. We develop a complete set of
theoretical results that describe the behavior of $k\th$ order Kronecker trend 
filtering in $d$ dimensions, for every $k \geq 0$ and $d \geq 1$. This reveals a 
number of interesting phenomena, including the dominance of KTF over linear
smoothers in estimating heterogeneously smooth functions, and a phase transition
at $d=2(k+1)$, a boundary past which (on the high dimension-to-smoothness side)
linear smoothers fail to be consistent entirely. We also leverage recent results
on discrete splines from \citet{tibshirani2020divided}, in particular, discrete
spline interpolation results that enable us to extend the KTF estimate to any
off-lattice location in constant-time (independent of the size of the lattice
$n$).
\end{abstract}

\section{Introduction}
\label{sec:intro}

We consider a standard nonparametric regression model, relating real-valued 
responses $y_i \in \R$, $i=1,\ldots,n$ to design points $x_i \in \cX \subseteq
\R^d$, 
$i=1,\ldots,n$,   
\begin{equation}
\label{eq:data_model1}
y_i = f_0(x_i) + \epsilon_i, \quad i=1,\ldots,n,
\end{equation}
where $f_0 : \cX\to \R$ is the (unknown) regression function to be estimated,
and $\epsilon_i$, $i=1,\ldots,n$ are mean zero stochastic errors. In this
paper, we will focus on functions $f_0$ that display heterogeneous smoothness
across the domain $\cX$, in a sense we will make precise later. We will also
focus on the case in which the design points form a $d$-dimensional lattice:
that is, we assume \smash{$n=N^d$}, and 
\begin{equation}
\label{eq:uniform_lattice}
\{x_1,\ldots,x_n\} = \{1/N, 2/N, \ldots, 1\}^d := Z_{n,d}.
\end{equation}
The assumption of a uniformly-spaced lattice, embedded in the unit cube
$[0,1]^d$, is used only for simplicity. Essentially all results (both methods
and theory) translate over to the case of a somewhat more general lattice
structure, a Cartesian product \smash{$\{z_{i1}\}_{i=1}^{N_1} \times
  \{z_{i2}\}_{i=1}^{N_2} \times \cdots \times \{z_{id}\}_{i=1}^{N_d}$}, where
\smash{$n = \prod_{j=1}^d N_j$}, and the sets in this product are otherwise 
arbitrary. We return to this point in Section \ref{sec:general_lattice}.  

This paper is a unification and extension of \citet{sadhanala2016total,
  sadhanala2017higher} (more will be said about the relationship to these papers
in Section \ref{sec:related_work}). The models of smoothness for $f_0$ that we
will study are based on \emph{total variation} (TV). For a univariate function
$g : [a,b] \to \R$, recall that its total variation is defined as    
$$
\TV(g; [a,b]) = \sup_{a < z_1 < \cdots < z_{m+1} < b} \; \sum_{i=1}^m 
|g(z_i) - g(z_{i+1})|.
$$
For a multivariate function $f : [0,1]^d \to \R$, we will consider notions of
smoothness that revolve around the following \emph{discrete} version of
multivariate total variation:  
$$
\TV(f; Z_{n,d}) = \sum_{j=1}^d \sum_{\substack{x,z \in Z_{n,d} \\ z = x+e_j/N}}   
|f(x) - f(z)|.
$$
Here we use $e_j$ to denote the $j\th$ standard basis vector in $\R^d$, and
hence the inner sum is taken over all pairs of lattice points $x,z \in Z_{n,d}$
that differ in the $j\th$ coordinate by $1/N$ (and match in all other
coordinates). We will also consider higher-order versions of discrete
multivariate TV, which are based on higher-order differences in the summands in
the above display. We will connect our discrete notions of TV smoothness to
standard continuum notions of total variation in Section \ref{sec:tv_lines}.

Broadly speaking, there are many multivariate nonparametric regression methods
available. Many of the methods in common use are \emph{linear smoothers}:
estimators of the form \smash{$\hf(x) = w(x)^\T y$}, for a suitable weight 
function $w :\cX \to \R^n$ (which can depend on the design $x_1,\ldots,x_n$),
where we use $y = (y_1,\ldots,y_n) \in \R^n$ for the response vector. Examples
include kernel smoothing, thin-plate splines, and reproducing kernel Hilbert
space estimators. A critical shortcoming of linear smoothers is that they cannot 
be \emph{locally adaptive}---they cannot adapt to different local levels of 
smoothness exhibited by $f_0$ over $\cX$. This is a phenomenon
that has been well-documented in various settings; see Section
\ref{sec:related_work}. 

The limitations of linear smoothers---and the need for nonlinear adaptive
methods---is a major theme in this paper. The central method that drives this
story is a multivariate extension of trend filtering, which is indeed nonlinear
and locally adaptive, a claim that will be supported by experiments and theory
in the coming sections. We must note at the outset that all of the developments
in this paper hinge on the assumption of lattice data (meaning, a lattice
structure for the design points). A multivariate extension of trend filtering
for scattered data would require a completely different approach (unlike, say,
kernel smoothing or reproducing kernel Hilbert space methods, which apply
regardless of the structure of the design points). The intersection of
multivariate nonparametric regression methods and locally adaptive methods is
actually quite small, especially when we further intersect this with the set of
simple methods that are easy to use in practice, are well-understood
theoretically. For this reason, we see the contributions of the current paper,
though limited to lattice data, as being worthwhile. The development of new
multivariate trend filtering methods for scattered data is important, and an
extension is discussed in Section \ref{sec:scattered_data}, but a comprehensive
study is left to future work.

\subsection{Review: trend filtering}
\label{sec:tf}

Before describing the main proposal, we review \emph{trend filtering}, a 
relatively recent method for univariate nonparametric regression, independently
proposed by \citet{steidl2006splines, kim2009trend}. For a univariate design, 
equally-spaced (say) on the unit interval, $x_i=i/n$, $i=1,\ldots,n$, and an
integer $k \geq 0$, the $k\th$ order trend filtering estimate is defined by the
solution of the optimization problem: 
\begin{equation}
\label{eq:tf}
\minimize_{\theta \in \R^n} \; \frac{1}{2} \|y-\theta\|_2^2 + \lambda
\|\dmatk_n \theta\|_1. 
\end{equation}
Here $\lambda \geq 0$ denotes a tuning parameter, $y=(y_1,\ldots,y_n) \in 
\R^n$ is the response vector, and \smash{$\dmatk_n \in \R^{(n-k-1) \times n}$}
is the difference operator of order $k+1$, which we will also loosely call
discrete derivative operator of order $k+1$. This can be defined recursively in
the following manner:    
\begin{equation}
\label{eq:diff_mat}
\begin{gathered}
\dmat^{(1)}_n = 
\left[\begin{array}{rrrrrr} 
-1 & 1 & 0 & \ldots & 0 & 0 \\
0 & -1 & 1 & \ldots & 0 & 0 \\
\vdots & & & & & \\
0 & 0 & 0 & \ldots & -1 & 1 
\end{array}\right], \\
\dmatk_n = \dmat^{(1)}_{n-k} \, \dmat^{(k)}_n, \quad k = 1,2,3,\ldots.   
\end{gathered}
\end{equation}
The intuition behind problem \eqref{eq:tf} is as follows. The penalty term,
which penalizes the discrete derivatives of $\theta$ of order $k+1$, can be
equivalently seen as penalizing the \emph{differences} in $k\th$ discrete 
derivatives of $\theta$ at adjacent design points (due to
\eqref{eq:diff_mat}). By the sparsity-inducing property of the $\ell_1$ norm,
the $k\th$ discrete derivatives of the trend filtering solution \smash{$\htheta$} 
will be exactly equal at a subset of adjacent design points, and
\smash{$\htheta$} will therefore exhibit the structure of a $k\th$ degree
piecewise polynomial, with adaptively-chosen knots. 

Here is a summary of the properties of trend filtering, as a nonparametric
regression tool.\footnote{We have defined it in this subsection for an
  evenly-spaced design, for simplicity, but trend filtering can still be defined
  for an arbitrary design, and all of the following properties still hold; see
  Section \ref{sec:general_lattice}.}  

\begin{itemize}
\item The discrete trend filtering estimate, which is defined over the design
  points, can be ``naturally'' extended to a $k\th$ degree piecewise polynomial
  function, in fact, a $k\th$ degree discrete spline, on $[0,1]$. 
\item Trend filtering is computationally efficient (several fast algorithms
  exist for the structured, convex problem \eqref{eq:tf}), and is not much
  slower to compute than (say) the smoothing spline. 
\item Trend filtering is more locally adaptive than the smoothing spline (or any
  linear smoother). This not only carries theoretical backing (next point), but
  is clearly noticeable in practice as well. 
\item Trend filtering attains the minimax rate (in squared empirical norm) of
  \smash{$n^{-(2k+2)/(2k+3)}$} for estimating a function $f_0$ whose $k\th$ weak
  derivative has bounded total variation. The minimax linear rate (the best 
  worst-case risk that can be attained by a linear smoother) over this class is  
  \smash{$n^{-(2k+1)/(2k+2)}$}.    
\end{itemize}
Support for the above facts can be found in \citet{tibshirani2014adaptive}, and
the discrete spline (numerical analytic) perspective behind trend filtering
is further developed in in \citet{tibshirani2020divided}. More will be said
about all of these properties in the coming sections, as analogous properties
will be developed for a multivariate extension of trend filtering. 

To prepare for this multivariate extension, it helps to recast the discrete
problem \eqref{eq:tf} in just a slightly different form. First, some notation:
for a vector $\theta \in \R^n$, we will (when convenient) index it by the
underlying design points, and write its components as $\theta(x_i)$,
$i=1,\ldots,n$ in place of $\theta_i$, $i=1,\ldots,n$. Next, we define a
difference operator, which we will again loosely refer to as a discrete
derivative operator, by
$$
(\dop\theta)(x_i) = 
\begin{cases}
\theta(x_{i+1})-\theta(x_i) & \text{if $i \leq n-1$}, \\
0 & \text{else}.
\end{cases}
$$
Naturally, we can view $\dop\theta$ as a vector in $\R^n$ with components 
$(\dop\theta)(x_i)$, $i=1,\ldots,n$. Higher-order discrete derivatives are
obtained by repeated application of the same formula; we abbreviate
$(\dop^2\theta)(x_i) = (\dop(\dop\theta))(x_i)$, and so on. In this new
notation, we can now rewrite problem \eqref{eq:tf} as
\begin{equation}
\label{eq:tf_dd}
\minimize_{\theta \in \R^n} \; \frac{1}{2} \sum_{i=1}^n \big(y_i -
\theta(x_i)\big)^2 + \lambda \sum_{i=1}^n |(\dop^{k+1}\theta)(x_i)|.
\end{equation}

\subsection{Kronecker trend filtering} 
\label{sec:ktf}

Let us return to the setting of multivariate design points on a lattice, as in
\eqref{eq:uniform_lattice} (where recall we use \smash{$N=n^{1/d}$}, assumed 
to be integral). Building from the univariate notation and definitions at the
end of the last subsection, we now introduce multivariate analogs. For a vector 
$\theta \in \R^n$, we will (when convenient) index its components by their
lattice positions, denoted $\theta(x)$, $x \in Z_{n,d}$. For each
$j=1,\ldots,d$, we define the discrete derivative of $\theta$ in the $j\th$
coordinate direction at a location $x$ by   
$$
(\dop_{x_j} \theta)(x) = 
\begin{cases}
\theta(x+e_j/N)-\theta(x) & \text{if $x,x+e_j/N \in Z_{n,d}$}, \\ 
0 & \text{else}.
\end{cases}
$$
We write \smash{$\dop_{x_j}\theta \in \R^n$} for the vector with components 
\smash{$(\dop_{x_j}\theta)(x)$}, $x \in Z_{n,d}$. As before, higher-order
discrete derivatives are simply defined by repeated application of the above
definition; we use abbreviations  
\smash{$(\dop_{x_j^2}\theta)(x)=(\dop_{x_j}(\dop_{x_j}\theta))(x)$},
\smash{$(\dop_{x_j,x_\ell}\theta)(x)=(\dop_{x_j}(\dop_{x_\ell}\theta))(x)$},
and so on.

With this notation in place, we define a multivariate version of trend
filtering, that we call \emph{Kronecker trend filtering} (KTF). Given an integer
$k \geq 0$, the $k\th$ order KTF estimate is defined by the solution of the
optimization problem:    
\begin{equation}
\label{eq:ktf_dd}
\minimize_{\theta \in \R^n} \; \frac{1}{2} \sum_{i=1}^n \big(y_i -
\theta(x_i)\big)^2 + \lambda \sum_{j=1}^d \sum_{x \in Z_{n,d}} 
|(\dop_{x_j^{k+1}} \theta)(x)|.
\end{equation}
Note the close analogy between \eqref{eq:tf_dd} and \eqref{eq:ktf_dd}: the
latter extends the former by adding up absolute discrete derivatives of $\theta$
of order $k+1$ along each one of the $d$ coordinate directions. A similar
intuition carries over from the univariate case, regarding the role of the
penalty in \eqref{eq:ktf_dd}, and the structure of the solution. As we can see,
the KTF problem penalizes the differences in $k\th$ discrete derivatives of 
$\theta$ at lattice positions $x$ and $z$, for all $x$ and $z$ that are adjacent
along any one of the $d$ coordinate directions. By the sparsifying nature of
the $\ell_1$ norm, the KTF solution \smash{$\htheta$} will have equal $k\th$
discrete derivatives between neighboring points on the lattice (and more so for
larger $\lambda$, generally speaking). Hence, along any line segment parallel to
one of the coordinate axes, the KTF solution \smash{$\htheta$} will have the
structure of a $k\th$ degree piecewise polynomial, with adaptively-chosen
knots. This intuition will be made rigorous in Section \ref{sec:ktf_continuous}.

We can also rewrite the KTF problem \eqref{eq:ktf_dd} in a more compact form, so
that it resembles \eqref{eq:tf}: 
\begin{equation}
\label{eq:ktf}
\minimize_{\theta \in \R^n} \; \frac{1}{2} \|y-\theta\|_2^2 + \lambda
\|\kronmatk \theta\|_1, 
\end{equation} 
where we define 
\begin{equation}
\label{eq:ktf_pen_mat}
\kronmatk = \left[\begin{array}{c}
\dmatk_N \otimes I_N \otimes \cdots \otimes I_N \\    
I_N \otimes \dmatk_N \otimes \cdots \otimes I_N \\ 
\vdots \\
I_N \otimes I_N \otimes \cdots \otimes \dmatk_N  
\end{array}\right].
\end{equation}
Here \smash{$\dmatk_N \in \R^{(N-k-1) \times N}$} is the discrete 
derivative matrix from \eqref{eq:diff_mat} (as would be used in $k\th$ order
univariate trend filtering on $N$ points); $I_N \in \R^{N\times N}$ denotes the
identity matrix; and $A \otimes B$ denotes the Kronecker product of matrices
$A,B$. Each block of rows in \eqref{eq:ktf_pen_mat} is made up of a total of
$d-1$ Kronecker products (a total of $d$ matrices). The Kronecker product
structure behind the penalty matrix in \eqref{eq:ktf_pen_mat} is what inspires
the name Kronecker trend filtering. In a similar vein, we will also refer to
\smash{$\|\kronmatk \theta\|_1$} as the $k\th$ order \emph{Kronecker total
variation} (KTV) of $\theta$. Note that when $k=0$, the KTF problem
\eqref{eq:ktf} reduces to anisotropic TV denoising on the $d$-dimensional
lattice, and when $d=1$, it reduces to $k\th$ order univariate trend filtering
\eqref{eq:tf}.

Before we delve any deeper into its properties, which will start in Section
\ref{sec:properties}, we give a simple example of KTF in Figure
\ref{fig:intro}. The example portrays an underlying function $f_0$ in $d=2$
dimensions that has two peaks in opposite corners of the unit square $[0,1]^2$,
one larger and one smaller, and is otherwise very smooth. Based on noisy
observations of $f_0$ over a 2d lattice, the estimates from KTF of orders
$k=0,1,2$ are able to capture this behavior. Meanwhile, estimates from kernel
smoothing are not, and they either oversmooth the larger peak or undersmooth the
valleys, depending on the choice of the bandwidth. Advocates of kernel
smoothing might say that this problem can be solved by giving the kernel a
locally varying bandwidth (that is, modeling the bandwidth itself as a function
of $x \in [0,1]^d$). While this can work in principle, for example, using
Lepski's method (see Section \ref{sec:related_work}), it can often be hard to
implement in practice. More to the point: the comparison in Figure
\ref{fig:intro} is not meant to portray the kernel smoothing framework as
being wholly untenable; rather, it is just meant to portray the differences in 
adaptivity of KTF versus kernel smoothing when each method is allowed only a
single tuning parameter.

\begin{figure}[t]
\centering
\includegraphics[width=0.32\textwidth]{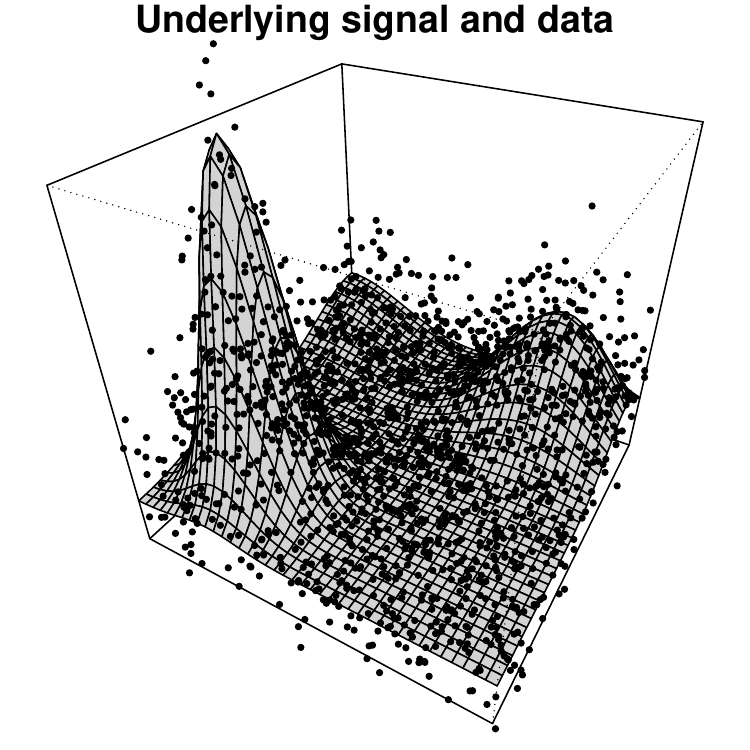}
\includegraphics[width=0.32\textwidth]{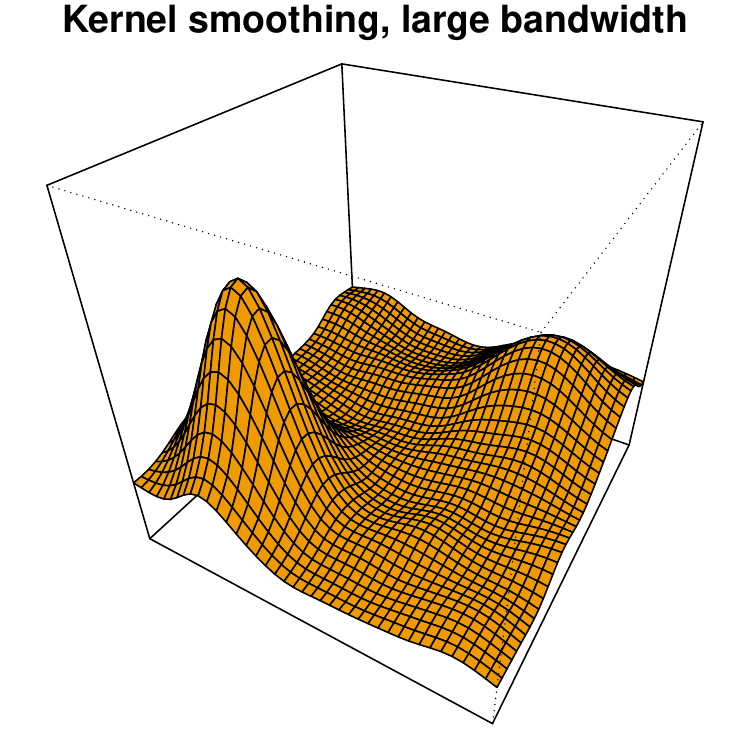} 
\includegraphics[width=0.32\textwidth]{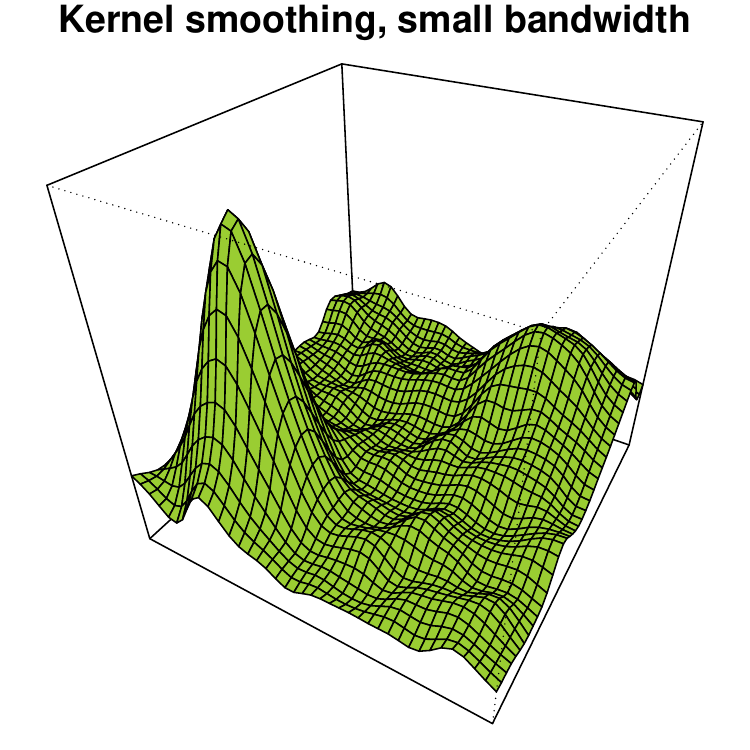} \\
\includegraphics[width=0.32\textwidth]{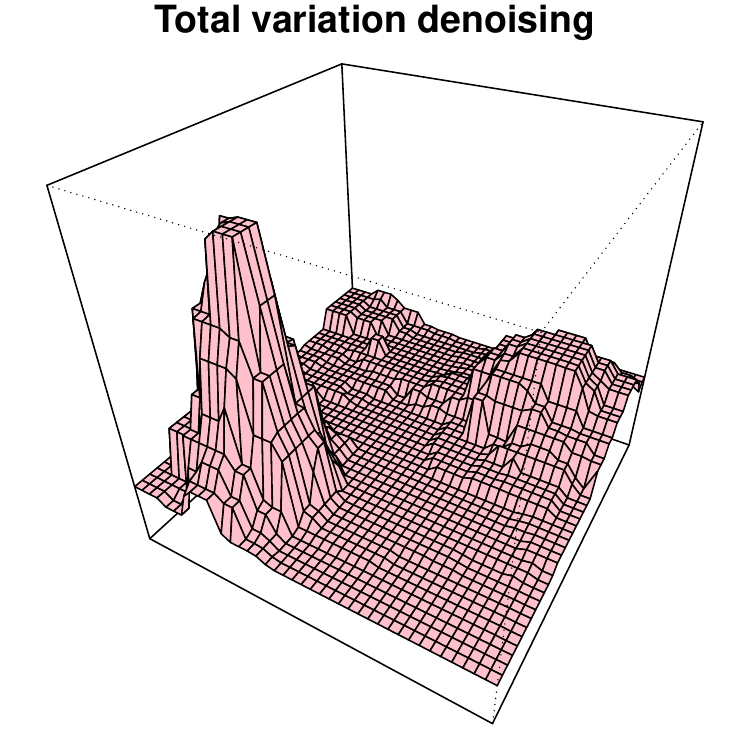} 
\includegraphics[width=0.32\textwidth]{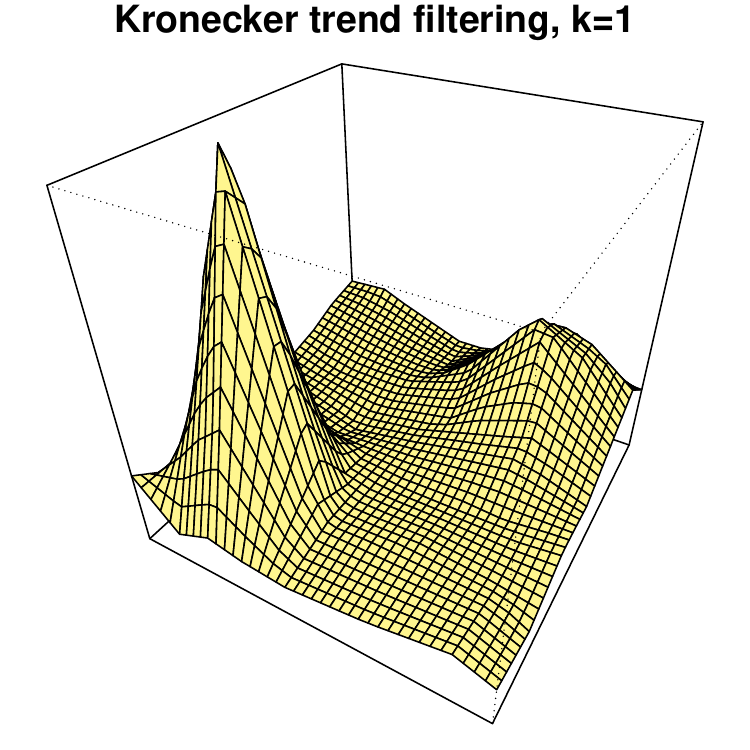} 
\includegraphics[width=0.32\textwidth]{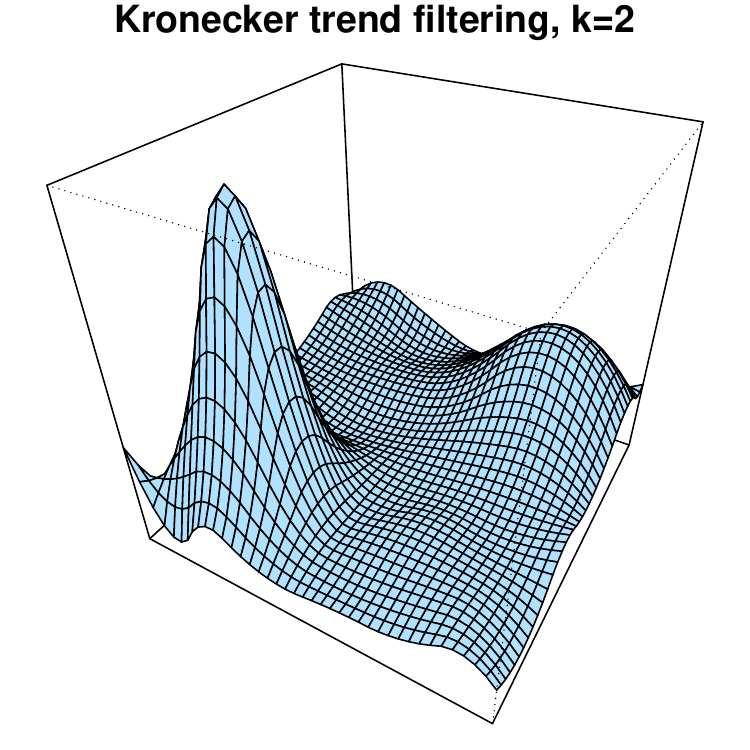}
\caption{\it\small Top left: underlying regression function $f_0$ evaluated
  over a square lattice with $n=40^2=1600$ points, and associated responses
  (formed by adding noise) shown as black points. Top middle and top right:
  kernel smoothing (with a spherical Gaussian kernel) fit to this data using
  large and small bandwidth values, respectively. Bottom left, middle, and
  right: Kronecker trend filtering estimates of orders $k=0,1,2$, respectively 
  (recall, KTF with $k=0$ reduces to anisotropic total variation denoising). We
  see that, in order to capture the larger of the two peaks in $f_0$, kernel 
  smoothing must significantly undersmooth the other peak (and surrounding
  areas); instead, with more regularization, it undersmooths throughout. The
  KTF estimates are able to adapt to heterogeneity in the smoothness of
  $f_0$. Also, each exhibits a distinct structure based on the polynomial order 
  $k$.}    
\label{fig:intro}
\end{figure}

\subsection{Related work}
\label{sec:related_work}

Our paper builds from a line of work on trend filtering, which, recall, enforces 
smoothness by penalizing the $\ell_1$ norm of discrete derivatives of a given
order \citep{steidl2006splines, kim2009trend, tibshirani2014adaptive, 
  wang2014falling, ramdas2016fast}. This can be seen as a discrete analog of the  
\emph{locally adaptive regression spline} estimator, which penalizes the TV of
the $k\th$ derivative of a function. This estimator was formally studied (and
named) by \citet{mammen1997locally}, though the same core idea---regularization
via the TV of derivatives---can also be found in \citet{koenker1994quantile},
and even earlier in \citet{nemirovski1984signal, nemirovski1985rate}. Notably,
the latter paper derives rates of convergence and proves that linear smoothers 
cannot be minimax rate optimal over TV classes, a result that was later
generalized to Besov spaces by \citet{donoho1998minimax}.  

For an overview of trend filtering, its connection to
classical theory on splines and divided differences, and related results that
bridge discrete and continuum representations (such as that connecting trend
filtering and locally adaptive regression splines), we refer to
\citet{tibshirani2020divided}. Trend filtering was extended to general graphs by
\citet{wang2016trend} (more on this below), and was analyzed in an additive
model setting by \citet{sadhanala2019additive}. The Kronecker trend filtering
estimator was proposed by \citet{sadhanala2017higher}. Rates of convergence for
KTF have been developed in different special cases, distributed among several
papers. Minimax results for TV denoising on lattices can be found in
\citet{hutter2016optimal} (upper bounds) and \citet{sadhanala2016total} (lower
bounds); with respect to KTF and KTV classes, this covers the case of $k=0$ and
all dimensions $d$. Meanwhile, \citet{sadhanala2017higher} gave minimax results
(matching upper and lower bounds) for all $k$ and $d=2$. The current paper
completes the landscape, deriving minimax theory for all smoothness orders $k$
and all dimensions $d$. The majority of this paper (including the minimax
theory) was completed in 2018 and can be found in the Ph.D.\ thesis of the first
author \citep{sadhanala2019phd}.

\paragraph{Graph-based TV methods.} 

Our paper is complementary to the line of research on locally adaptive
nonparametric estimation over graphs, such as \citet{gavish2010multiscale,  
  sharpnack2013detect, wang2016trend, gobel2018construction, padilla2018dfs,
  padilla2020adaptive, ye2021nonparametric}. The lattice structure that we
consider in this paper can be cast as a particular $d$-dimensional grid graph
with $n$ nodes (that is, with all side lengths equal to
\smash{$N=n^{1/d}$}). However, many of the methods proposed in the
aforementioned references seek to be far more general and operate over arbitrary
graph structures. This generality comes with several challenges, from conceptual
to theoretical.

For example, \citet{wang2016trend} developed \emph{graph trend filtering} (GTF),
which estimates a signal that takes values over the nodes of an arbitrary graph
by penalizing the $\ell_1$ norm its graph derivatives, which are defined via an
iterated graph Laplacian operator. Applied to the $d$-dimensional grid graph,
and translated into the notation of our paper, the penalty term used in $k\th$ 
order GTF for a signal $\theta$ at a point in the lattice $x \in Z_{n,d}$ is:    
$$
\begin{cases} 
\displaystyle
\sum_{j_1=1}^d \Bigg| \sum_{j_2,\ldots,j_q=1}^d 
\Big(\dop_{x_{j_1},x_{j_2}^2,\ldots,x_{j_q}^2} \theta \Big)(x)  
\Bigg| & \text{for $k$ even, where $q=k/2$}, \\
\displaystyle
\Bigg| \sum_{j_1,\ldots,j_q=1}^d 
\Big(\dop_{x_{j_1}^2,x_{j_2}^2,\ldots,x_{j_q}^2} \theta \Big)(x)   
\Bigg| & \text{for $k$ odd, where $q=(k+1)/2$}.
\end{cases}
$$
Compared to the analogous penalty term used in KTF \eqref{eq:ktf_dd}, which 
is just \smash{$\sum_{j=1}^d |(\dop_{x_j^{k+1}} \theta)(x)|$}, we see that the
above is much is harder to interpret. GTF clearly considers some form of mixed 
derivatives (whereas KTF does not and is anisotropic), but it is generally
unclear what kind of smoothness GTF is promoting. Moreover,
\citet{sadhanala2017higher} argue that using the GTF penalty operator in order
to define a smoothness class for the analysis of multivariate signals is
problematic, in the following sense: for any $k \geq 1$ and any dimension $d$,
there are $k\th$ order Holder smooth functions whose discretization to the
lattice is arbitrarily nonsmooth as measured by the $k\th$ order GTF penalty
operator. This is due to issues in the way the GTF penalty operator measures
smoothness on the boundaries of the lattice. 

\paragraph{Continuous-time multivariate TV methods.}

There is a rich body of work in applied mathematics on the denoising of signals
or images by promoting total variation smoothness, beginning with the seminal
paper by \citet{rudin1992nonlinear}, which gave rise to the so-called
Rudin-Osher-Fatemi (ROF) functional. This was then further developed and
extended by \citet{rudin1994total, vogel1996iterative, chambolle1997image, 
chan2000highorder, candes2002new, chan2005aspects, dong2011automated}. 
Papers in this line of work tend to be cast in continuous-time,\footnote{For $d
  \geq 2$, it may be more appropriate to call this ``continuous-space'', but we
  stick with the term continuous-time for simplicity.}   
which means that the estimand, estimator, and typically even the data itself are
each functions of one or more variables on a continuous domain (such as
$[0,1]^d$). A related line of work, inspired by ROF, considers discretization as
a step in numerical optimization, see, for example,
\citet{chambolle2004algorithm, chambolle2005total, almansa2008tv}.      

More recently, \citet{delalamo2021frameconstrained} studied estimation of a
multivariate function of bounded variation under the white noise model, in
arbitrary dimension $d$. This may be interpreted as a continuum analog of our
setting, albeit for $k=0$ only. They derive minimax rates for $L^p$ estimation
of TV bounded functions. When $p=2$ (matching our analysis), their estimator 
obtains (up to log factors) the minimax rate of  $n^{-1/d}$ on the  
squared $L^2$ error scale, for any $d \geq 2$, which agrees with the minimax
rate in our discrete setting. We note that our work, while motivated from
discrete principles, bears rigorous connections to continuous-time formulations
of multivariate TV; see Sections \ref{sec:ktf_continuous} and \ref{sec:tv_lines}.       



\paragraph{Alternative models for multivariate TV smoothness.} 

Recently, there has been a stream of work studying different generalizations 
of TV and trend filtering penalties to multiple dimensions, including
\citet{bibaut2019fast, fang2021multivariate, ortelli2021tensor,
ki2021mars}. These papers are based on notions of smoothness related to the
Hardy-Krause variation of a multivariate function. For a function to be smooth
in the Hardy-Krause sense, it must exihibit an order of smoothness that scales
with the dimension $d$,\footnote{For example, for a smooth function $f$, its
  Hardy-Krause variation can be expressed in terms of the $L^1$ norm of
  \smash{$\partial^d f / \prod_{j=1}^d \partial x_j$}.} 
which leads to error rates that are (nearly) dimension-free.  However,
practically speaking, assuming that the inherent smoothness of the regression 
function is on par with the ambient dimension may not be reasonable in some
applications. A distinct feature of our work is that the smoothness order $k$
(which translates into the max degree of the local polynomial in the fitted
model) is a user-defined parameter, and is not tied in any way to the dimension
$d$.

One may adopt a different perspective and model the regression function as being
multivariate piecewise constant, or more generally multivariate piecewise
polynomial, which can be broadly interpreted as a ``strong sparsity'' analog to
the ``weak sparsity'' assumption that underlies TV smoothness. In the univariate
case, trend filtering (or TV denoising) has been analyzed under such ``strong
sparsity'' assumptions by \citet{dalalyan2017prediction, lin2017sharp,
guntuboyina2020adaptive, ortelli2021prediction}, and others, yielding faster
rates of convergence under a certain min-length condition on the polynomial
pieces in the true signal. The multivariate case is much more subtle; for
example, for a 2d piecewise constant signal model, \citet{chatterjee2021new}
prove that the error rate of TV denoising depends on the orientation of the
boundary of the true pieces (whether they are axis-aligned or not). Currently,
it appears the most comprehensive theory for the multivariate piecewise
polynomial model is given by \citet{chatterjee2021adaptive}, who propose and
analyze CART-style estimators, inspired by earlier work of
\citet{donoho1997cart}. While we believe that the KTF estimator will exhibit
some degree of adaptivity to multivariate piecewise polynomials, we also believe
it will have some nontrivial failure cases (suboptimality), given what the
existing 1d and 2d analyses have shown.

\paragraph{Locally adaptive kernel smoothing.}

Lepski's method, which originated in the seminal paper by
\citet{lepskii1991problem}, is a procedure for selecting a local bandwidth in  
kernel smoothing; roughly speaking, at each point $x$ in the domain, it 
chooses the largest bandwidth (from a discrete set of possible values) such that
the kernel estimate at $x$ is within a carefully-defined error tolerance to
estimates at smaller bandwidths. Since its introduction, many papers have 
studied and generalized Lepski's method, see, for example, 
\citet{lepskii1992asymptotically,   
  lepskii1993asymptotically, lepski1997optimal1, lepski1997optimal2, 
  kerkyacharian2001nonlinear, kerkyacharian2008nonlinear,
  goldenshluger2008universal, goldenshluger2009structural,
  goldenshluger2011bandwidth, goldenshluger2013general,
  lepski2015adaptive}. Most related to the current paper is
\citet{kerkyacharian2001nonlinear, kerkyacharian2008nonlinear}, who consider
estimation in anisotropic Besov classes using Lepski's method, under the white
noise model. The ``dense'' and ``sparse'' zones in their work roughly correspond
to the  cases $s > 1/2$ and $s \leq 1/2$ in our theory, respectively (see Figure
\ref{fig:theory}). This will be revisited in Remark \ref{rem:besov_minimax}.    

The work referenced above---and the broader literature on locally adaptive 
kernel methods---is focused on establishing sharp theoretical guarantees.
Practical and computational considerations are not a focus. These are some
nontrivial barriers to practical implementation, even as basic as the fact that
there are effectively many tuning parameters (leading to somewhat arbitrary
practical design choices that could be made) in the proposed methods.         

\paragraph{Tensor product and hyperbolic wavelets.} 

Wavelets have a rich history in signal processing, approximation theory, and
other disciplines; classic references include \citet{daubechies1992ten,
  chui1992introduction, meyer1993progress, mallat2009wavelet}. Seminal work by
\citet{donoho1998minimax}, on minimax estimation over univariate Besov spaces
using wavelet-based estimators, contributed greatly to the popularity of
wavelets in statistics. It appears that the earliest work on multivariate
wavelet approximation is \citet{meyer1985principe, meyer1990ondelettes} and 
\citet{mallat1989theory, mallat1989multiresolution}, which focuses on
separable wavelet bases, formed from tensor products of univariate wavelet 
basis functions within each resolution level. A second approach, which is
well-studied in approximation theory but less so in statistics, instead
constructs ``truly multivariate'' wavelet bases by generating multiresolution
subspaces in the ambient domain, which gives rise to nonseparable bases; 
see \citet{meyer1990ondelettes, riemenschneider1992wavelets, chui1992compactly,
  lorentz1992wavelets, devore1992wavelets}. In both cases described above, the
support of the wavelet function has the same scale in each coordinate
direction. The desire to effectively represent functions with different degrees
of smoothness in different directions thus led to the development of hyperbolic
wavelets, whose basis functions are formed by taking tensor products of
univariate wavelet basis functions \emph{across} resolution levels
\citep{neumann1997wavelet, devore1998hyperbolic, neumann2000multi}. Of
particular relevance to our work is the latter paper, and connections will be 
drawn in Remark \ref{rem:besov_minimax}.

Practically speaking, we generally find multivariate wavelet denoising to be
more sensitive to the level of noise than an estimator like KTF, and for wavelet
denoising to suffer worse performance when the signal-to-noise ratio is at low
or moderate levels. Evidence for this is given later in
Figure \ref{fig:error_analysis} in Section \ref{sec:error_analysis}. This is
also consistent with what is observed in the univariate setting in
\citet{tibshirani2014adaptive}.   

\subsection{Summary and outline}

A summary of results in this paper and outline for this paper is given below. 

\begin{itemize}
\item In Section \ref{sec:properties}, we derive some basic properties of KTF,
  including an equivalent continuous-time formulation for problem
  \eqref{eq:ktf}, which provides insights into the local structure of KTF
  estimates.  

\item In Section \ref{sec:tv_lines}, we derive an expression for higher-order
  multivariate TV (in the standard measure-theoretic sense) in terms of an 
  integrating univariate TV on line segments running parallel to the coordinate
  axes. We use this to motivate the definition of KTV smoothness, the central
  notion of smoothness used in this paper---it is simply a discretization of the 
  aforementioned integral.        

\item In Section \ref{sec:classes}, we introduce the smoothness classes of
  interest for our study of minimax theory, and examine the relationships
  between them.   

\item In Section \ref{sec:theory}, we derive a complete set of results on the
  minimax estimation risk, as measured in the squared $\ell_2$ norm, over the
  set \smash{$\kronset^k(C_n)$} of vectors $\theta$ with $k\th$ order KTV 
  smoothness satisfying \smash{$\|\kronmatk \theta\|_1 \leq C_n$}, for a 
  given sequence $C_n>0$. We prove that KTF is minimax rate optimal (up to 
  log factors) for any $k,d$, and derive lower bounds on the minimax linear risk
  (that is, the best worst-case risk over all linear smoothers) which show that 
  linear estimators are suboptimal for any $k,d$. Interestingly, the minimax
  rates reveal a phase transition at $2(k+1)=d$, and in the low
  smoothness-to-dimension regime, linear smoothers fail to be consistent 
  altogether. See Figure \ref{fig:theory} for a more detailed summary.    

\item In Section \ref{sec:optimization}, we describe and compare specialized
  convex optimization algorithms that can be used to compute the KTF solution in
  \eqref{eq:ktf}.  

\item In Section \ref{sec:interpolation}, we present an extremely efficient and 
  simple algorithm for interpolating the discrete KTF estimate \smash{$\htheta$} 
  (the solution in \eqref{eq:ktf}), defined over the lattice, into a function
  \smash{$\hf$} (the solution in \eqref{eq:ktf_continuous}), defined over the  
  underlying continuum domain $[0,1]^d$. Remarkably, this interpolation
  method, which builds from discrete spline interpolation results from
  \citet{tibshirani2020divided}, runs in \emph{constant-time} (independent of
  $n$).    

\item In Section \ref{sec:experiments}, we carry out empirical experiments that 
  compare KTF and various other nonparametric regression estimators, and examine
  whether the empirical error rates match the minimax theory derived in Section
  \ref{sec:theory}.  

\item In Section \ref{sec:discussion}, we conclude with a discussion, and cover  
  some extensions and directions for future work. 
\end{itemize}

\begin{figure}[t!]
\centering
\includegraphics[width=0.75\textwidth]{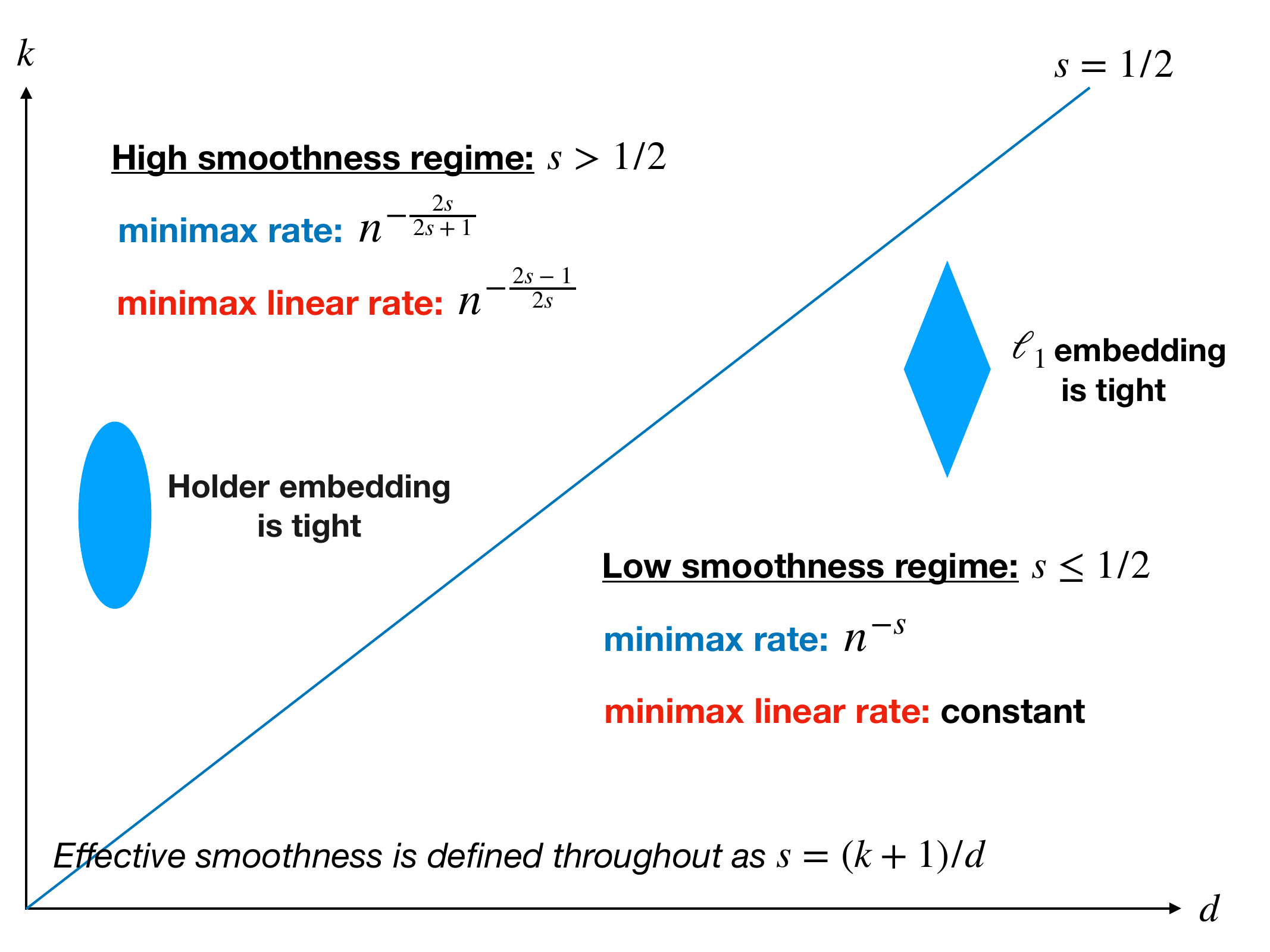}
\captionsetup{singlelinecheck=off}
\caption[]{\it\small Summary of the minimax results developed in this paper. The
  central object of our study is the set \smash{$\kronset^k(C_n^*)$} of vectors 
  $\theta$ defined over the $d$-dimensional lattice $Z_{n,d}$, with $k\th$ order
  KTV smoothness satisfying \smash{$\|\kronmatk \theta\|_1 \leq C_n^*$}, for a 
  sequence \smash{$C_n^*>0$} obeying what we call the canonical scaling, to be
  made precise later. The following two statements hold, generally (regardless
  of $k,d$):   
  \begin{enumerate}
    \item KTF achieves the minimax rate (up to log factors) over
      \smash{$\kronset^k(C_n^*)$}; and 
    \item no linear smoother is able to achieve the minimax rate over this
      class. 
  \end{enumerate}
  However, the story is more interesting, due to a phase transition 
  occurring at $2(k+1)=d$. Defining a notion of \emph{effective smoothness} by
  $s=(k+1)/d$, this can be explained as follows. When $s>1/2$, the minimax rate
  has the more classical form \smash{$n^{-2s/(2s+1)}$}, matching the minimax
  rate for a $k\th$ Holder class in dimension $d$ (or an $s\th$ order Holder
  class in the univariate case). Indeed, the lower bound on the minimax rate
  that we derive is given by embedding a Holder class into
  \smash{$\kronset^k(C_n^*)$}. Meanwhile, the minimax linear risk (the best 
  worst-case risk among linear smoothers) scales as \smash{$n^{-(2s-1)/(2s)}$},
  which can be interpreted as the rate of KTF (or any other minimax optimal
  method) for a problem with a half less degree of effective smoothness. When $s
  \leq 1/2$, the minimax rate takes on the less classical form $n^{-s}$, and the
  lower bound is obtained by embedding a suitable $\ell_1$ ball into
  \smash{$\kronset^k(C_n^*)$}. Further, the gap between the minimax linear and 
  nonlinear rates is even more dramatic: the minimax linear rate is constant,
  which means no linear smoother is even consistent over
  \smash{$\kronset^k(C_n^*)$} (in the sense of worst-case risk). Finally, though 
  not reflected in the figure, we note that when $s < 1/2$ the KTV class and its
  embedded Holder class exhibit different minimax rates, $n^{-s}$ versus
  \smash{$n^{-2s/(2s+1)}$}, respectively. Whether KTF can adapt to the latter
  (faster) Holder rate in the low smoothness-to-dimension regime, $s < 1/2$, is
  an open question.}  
\label{fig:theory}
\end{figure}
\section{Basic properties}
\label{sec:properties}

In this section, we cover a number of basis properties that reflect the
structure and complexity of KTF estimates. 

\subsection{Unpenalized component}

We start by examining the null space of the KTF penalty matrix in
\eqref{eq:ktf_pen_mat}. A word on notation here, and in general: when
convenient, we will use $x_j$ to denote the $j\th$ component of a vector $x$,
which should not be confused with our use of $x_i$ to denote the $i\th$ design
point (itself a $d$-dimensional vector). While this is an unfortunate clash of
notation, the meaning should always be clear form the context (and further, we 
will keep the use of indices $i,j$ for the two cases consistent throughout to
aid the interpretation---hence $x_j$ will always be univariate, the $j\th$
component of a vector $x$, and $x_i$ will always be $d$-dimensional, the $i\th$
design point).

\begin{proposition}
\label{prop:ktf_null_space}
The null space of the KTF penalty matrix in \eqref{eq:ktf_pen_mat} has dimension
\smash{$(k+1)^d$}. Furthermore, it is spanned by a polynomial basis made up of 
elements   
$$
p(x) = x_1^{a_1} x_2^{a_2} \cdots x_d^{a_d}, \quad x \in Z_{n,d},
$$
for all $a_1,\ldots,a_d \in \{0,\ldots,k\}$.
\end{proposition}

The proof is elementary and is deferred to Appendix \ref{app:proofs} (all proofs
in this paper are deferred to the appendix). This proposition reveals that the
KTF penalty matrix has quite a rich null space, thus KTF lets a significant 
component of the response vector $y$ ``pass through'' unpenalized. In contrast
to univariate trend filtering, which preserves univariate polynomials of degree
$k$ (precisely, it preserves the projection of $y$ onto this subspace), KTF
preserves ``much more'' than multivariate polynomials of degree $k$: it
preserves multivariate polynomials of \emph{max degree} $k$.\footnote{By a
  multivariate polynomial of degree $k$, we mean (adhering to the standard 
  classification) a sum of terms of the form \smash{$b \prod_{j=1}^d
    x_j^{a_j}$}, where the sum of degrees satisfies \smash{$\sum_{j=1}^d a_j
    \leq k$}. By a multivariate polynomial of max degree $k$, we mean the same,
  but where the degrees satisfy $a_j \leq k$, $j=1,\ldots,d$.}   
To see an example, when $k=1$ and $d=2$, the KTF estimator---which we might
be tempted to call ``linear-order'' KTF (to use an analogous term as we do in
univariate trend filtering)---preserves any polynomial of the form $p(x) = ax_1
+ bx_2 + cx_1x_2$. This is of course \emph{not} a linear function, but a
bilinear one (due to the cross-product term $x_1x_2$).

\subsection{Review: trend filtering in continuous-time} 

For univariate trend filtering \eqref{eq:tf}, an equivalent continuous-time
formulation was derived in \citet{tibshirani2014adaptive}: 
\begin{equation}
\label{eq:tf_continuous}
\minimize_{f \in \cH^k_n} \; \frac{1}{2} \sum_{i=1}^n \big(y_i - f(x_i)\big)^2 +   
\lambda \, \TV(f^{(k)}).
\end{equation}
Problems \eqref{eq:tf}, \eqref{eq:tf_continuous} are equivalent in
the sense that their solutions \smash{$\htheta,\hf$}, respectively, satisfy
\smash{$\htheta_i = \hf(x_i)$}, $i=1,\ldots,n$. In \eqref{eq:tf_continuous},
we use $f^{(k)}$ to denote the $k\th$ weak derivative of $f$, and $\TV(\cdot) =  
\TV(\cdot; [a,b])$ to denote the total variation operator defined with respect
to any interval $[a,b]$ containing the design points (henceforth, when
convenient, we will drop the underlying domain from our notation for univariate 
TV). The optimization in \eqref{eq:tf_continuous} is performed over all
functions $f$ that lie in a space \smash{$\cH^k_n$}, which is the span of   
\smash{$h^k_{n,j}$}, $j=1,\ldots,n$, the following $k\th$ degree piecewise  
polynomials:     
\begin{equation}
\label{eq:ffb}
\begin{gathered}
h^k_{n,i}(x) = \frac{1}{(i-1)!} \prod_{j=1}^{i-1}(x-j/n),  
\quad i=1,\ldots,k+1, \\
h^k_{n,i}(x) = \frac{1}{k!} \prod_{j=i-k}^{i-1} (x-j/n) \cdot   
1\{x > (i-1)/n\}, \quad i=k+2,\ldots,n. 
\end{gathered}
\end{equation}
(Here, for convenience, we interpret the empty product to be equal to 1.)  The
functions in \eqref{eq:ffb} are called the $k\th$ degree \emph{falling factorial   
  basis}. Observe that they depend on the $n$ underlying design points 
$1/n,2/n,\ldots,1$, which we express through the first subscript $n$ in
\smash{$h^k_{n,i}$}. Note also the similarity between the above basis and the
standard truncated power basis for splines; in fact, when $k=0$ or $k=1$, the
two bases are equal and \smash{$\cH^k_n$} is just a space of splines with
knots at the design points. However, when $k \geq 2$, this is no longer
true---the falling factorial functions are $k\th$ degree piecewise polynomials 
with (mildly) discontinuous derivatives of all orders $1,\ldots,k-1$, and
therefore they span a different space than that of $k\th$ degree splines---they 
space the space of $k\th$ degree \emph{discrete splines}, which are piecewise
polynomials that have continuous discrete derivatives (rather than derivatives)
at their knot points. See \citet{tibshirani2020divided}.

To be clear, the original formulation \eqref{eq:tf} is more computationally
convenient (it is a structured convex problem for which several fast algorithms
exist, discussed in Section \ref{sec:related_work}). But the variational
formulation \eqref{eq:tf_continuous} is important because it provides rigorous
backing to the intuition that a $k\th$ order trend filtering estimate exhibits 
the structure of a $k\th$ degree piecewise polynomial, with adaptively-chosen
knots (a feature of the $\ell_1$ penalty in \eqref{eq:tf} or TV penalty in 
\eqref{eq:tf_continuous}); moreover, it shows how to extend the trend filtering
estimate from a discrete sequence, defined over the design points, to a function
on the continuum interval $[0,1]$. 

\subsection{Continuous-time}
\label{sec:ktf_continuous}

We now develop a similar continuous-time representation for KTF.

\begin{proposition}
\label{prop:ktf_continuous}
Let $h^k_{N,i} : [0,1] \to \R$, $i=1,\ldots,N$ denote the $k\th$ degree falling   
factorial functions \eqref{eq:ffb} with respect to design points
$1/N, 2/N, \ldots, 1$, and \smash{$\cH^k_{n,d}$} denote the space spanned by all   
$d$-wise tensor products of these functions. That is, abbreviating \smash{$h_i =
  h^k_{N,i}$}, $i=1,\ldots,N$, this is the space of all functions $f: [0,1]^d
\to \R$ of the form    
\begin{equation}
\label{eq:ffb_tensor}
f(x) = \sum_{i_1,\ldots,i_d=1}^N \alpha_{i_1,\ldots,i_d} h_{i_1}(x_1)
h_{i_2}(x_2) \cdots h_{i_d}(x_d), \quad x \in [0,1]^d,
\end{equation}
for coefficients $\alpha \in \R^n$ (whose components we denote as
\smash{$\alpha_{i_1,\ldots,i_d}$}, $i_1,\ldots,i_d\in\{1,\ldots,N\}$). Then
the KTF estimator defined in \eqref{eq:ktf} is equivalent to the optimization
problem: 
\begin{equation}
\label{eq:ktf_continuous}
\minimize_{f \in \cH^k_{n,d}} \; \frac{1}{2} \sum_{i=1}^n \big(y_i -
f(x_i)\big)^2 + \lambda \sum_{j=1}^d \sum_{x_{-j}}    
\TV\bigg(\frac{\partial^k f(\cdot,x_{-j})} {\partial x_j^k}\bigg), 
\end{equation}
where $f(\cdot,x_{-j})$ denotes $f$ as function of the $j\th$ coordinate 
with all other dimensions fixed at $x_{-j}$, \smash{$(\partial^k/\partial
x_j^k)(\cdot)$} denotes the $k\th$ partial weak derivative operator with respect  
to $x_j$, and inner sum in the second term in \eqref{eq:ktf_continuous} is
interpreted as a sum over $Z_{m,d-1}$, where $m=N^{d-1}$ (that is, a sum over
the $(d-1)$-dimensional uniformly-spaced lattice with $N^{d-1}$ total
points). The discrete \eqref{eq:ktf} and continuous-time
\eqref{eq:ktf_continuous} problems are equivalent in the sense that at their
solutions \smash{$\htheta,\hf$}, respectively, we have: \smash{$\htheta_i =
  \hf(x_i)$}, $i=1,\ldots,n$.   
\end{proposition}

\begin{remark}
Similar to the discrete-time terminology, we will refer to \smash{$\sum_{j=1}^d
  \sum_{x_{-j}} \TV(\partial^k f(\cdot,x_{-j})/\partial x_j^k)$}, the penalty
functional in the continuous-time representation \eqref{eq:ktf_continuous} of
the KTF problem, as the $k\th$ order \emph{Kronecker total variation} (KTV) of
$f$. A key implication of Proposition \ref{prop:ktf_continuous} is that the
solution \smash{$\hf$} in \eqref{eq:ktf_continuous} not only interpolates the
solution \smash{$\htheta$} in \eqref{eq:ktf}, but it is exactly as smooth in
continuous-time as \smash{$\htheta$} is in discrete-time, as measured by KTV:  
$$
\|\kronmatk \htheta\|_1 = \sum_{j=1}^d \sum_{x_{-j}}    
\TV\bigg(\frac{\partial^k \hf(\cdot,x_{-j})} {\partial x_j^k}\bigg).
$$
How to form the interpolant \smash{$\hf$} is discussed briefly in the remark
after the next, and covered in detail in Section \ref{sec:interpolation}.
\end{remark}

\begin{remark}
From \eqref{eq:ffb_tensor}, the basis underlying the continuous-time
representation \eqref{eq:ktf_continuous} of the KTF optimization problem, we can
see that a $k\th$ order KTF estimate exhibits the structure of a tensor product
of $k\th$ degree discrete splines, with adaptively knots, chosen to promote 
higher-order TV smoothness along the coordinate axes. In other words, locally,
it exhibits the structure of a multivariate polynomial of max degree $k$. When
$k=1$ and $d=2$, for example, the structure is locally of the form
\smash{$\hf(x) = ax_1 + bx_2 + cx_1x_2$}, which is a bilinear function and has
local curvature. Such curvature is somewhat visible in Figure \ref{fig:intro}
(bottom row, middle panel), and it will be even more apparent later when we
discuss interpolation, see Figure \ref{fig:interp_example}.
\end{remark}

\begin{remark}
\label{rem:ktf_basis}
The proof of Proposition \ref{prop:ktf_continuous} reveals that
\eqref{eq:ktf_continuous} has an equivalent form, transcribed here for 
convenience: 
\begin{equation}
\label{eq:ktf_basis}
\minimize_{\alpha \in \R^n} \; \frac{1}{2} \Big\| y - \Big(\hmatk_N \otimes 
\cdots \otimes \hmatk_N\Big) \alpha \Big\|_2^2 +  
\lambda k! \left\| \left[\begin{array}{c}
I^0_N \otimes \hmatk_N \otimes \cdots \otimes \hmatk_N \\
\hmatk_N \otimes I^0_N \otimes \cdots \otimes \hmatk_N \\
\vdots \\
\hmatk_N \otimes \hmatk_N \otimes \cdots \otimes I^0_N
\end{array}\right] \alpha \right\|_1,
\end{equation}
where \smash{$\hmatk_N \in \R^{N \times N}$} is the falling factorial basis
matrix (with columns given by evaluations of the falling factorial functions at
the design points) and \smash{$I^0_N \in \R^{(N-k-1) \times N}$} denotes the
last $N-k-1$ rows of the identity $I_N$. Interestingly, the penalty in
\eqref{eq:ktf_basis} is not a pure $\ell_1$ penalty on the coefficients $\alpha$
(as it would be in basis form in the univariate case) but an $\ell_1$ penalty on
aggregated (positive linear combinations of) coefficients. 

The basis formulation in \eqref{eq:ktf_basis} gives us a natural recipe for how
to extend a KTF estimate from a discrete sequence, defined over the lattice
points, to a function on the hypercube $[0,1]^d$. This is simply: 
\begin{equation}
\label{eq:ktf_extend}
\hf(x) = \sum_{i_1,\ldots,i_d=1}^N \halpha_{i_1,\ldots,i_d} h_{i_1}(x_1) 
h_{i_2}(x_2) \cdots h_{i_d}(x_d), \quad x \in [0,1]^d,
\end{equation}
where \smash{$\halpha$} is the solution in \eqref{eq:ktf_basis}. Though it is
not obvious from the basis expansion in \eqref{eq:ktf_extend}, it turns out that
at any point $x \in [0,1]^d$, we can form the prediction \smash{$\hf(x)$} in
constant-time, starting from the fitted values \smash{$\htheta_i = \hf(x_i)$}, 
$i=1,\ldots,n$. This leverages recent advances in discrete spline interpolation
from \citet{tibshirani2020divided}, and is covered in Section
\ref{sec:interpolation}. 
\end{remark}

\subsection{Degrees of freedom}

Given data from a model \eqref{eq:data_model1}, where the errors $\epsilon_i$, 
$i=1,\ldots,n$ are i.i.d.\ with mean zero and variance $\sigma^2$, recall that
the \emph{degrees of freedom} of an estimator \smash{$\htheta_i = \hf(x_i)$},
$i=1,\ldots,n$ of the means $\theta_{0i} = f_0(x_i)$, $i=1,\ldots,n$ is a
quantitative reflection of its complexity, defined as \citep{efron1986biased,
  hastie1990generalized}:   
$$
\df(\htheta) = \frac{1}{\sigma^2} \sum_{i=1}^n \Cov(y_i, \htheta_i).
$$
When the errors are Gaussian, \citet{tibshirani2011solution,
  tibshirani2012degrees} derived an expression for the degrees of freedom of any
generalized lasso estimator, based on Stein's formula
\citep{stein1981estimation}. This covers KTF in \eqref{eq:ktf} as a special
case, and thus translates into the following result for our setting: if
$\epsilon_i \sim N(0,\sigma^2$), $i=1,\ldots,n$ in \eqref{eq:data_model1}, and 
\smash{$\htheta$} denotes the solution in \eqref{eq:ktf} with active set 
$$
A = \supp\big(\kronmatk \htheta\big) = 
\big\{ i : \big[\kronmatk \theta\big]_i \not= 0\big\},
$$
then 
\begin{equation}
\label{eq:ktf_df}
\df(\htheta) = \E\Big[\nuli\Big( \big[\kronmatk]_{-A} \Big)\Big],
\end{equation}
where $\nuli(M)$ denotes the nullity (dimension of the null space) of a matrix
$M$, and $M_{-S}$ denotes the submatrix of $M$ given by removing all rows
indexed by a set $S$. From the above, we of course have the natural estimator of
degrees of freedom 
\begin{equation}
\label{eq:ktf_df_unbiased}
\widehat\df(\htheta) = \nuli\Big( \big[\kronmatk]_{-A} \Big),
\end{equation}
which is unbiased for \eqref{eq:ktf_df}. 

The expression in \eqref{eq:ktf_df_unbiased} is easy to interpret when $k=0$: in
this case, it reduces to the number of connected constant pieces in the KTF
solution \smash{$\htheta$}, where connectivity is interpreted with respect to
the underlying $d$-dimensional grid graph. This follows from the fact that the
penalty matrix \smash{$\dmat^{(1)}_{n,d}$} in this case is the edge incidence 
operator on the grid graph, and any submatrix of this penalty matrix (defined
by removing a subset of rows) is itself the edge incidence operator with respect 
to a subgraph of the original grid (induced by removing a subset of the 
edges). This result and its interpretation was already given in
\citet{tibshirani2011solution, tibshirani2012degrees} in the context of TV
denoising on a graph.

When $k \geq 1$, the unbiased estimator of degrees of freedom in
\eqref{eq:ktf_df_unbiased} is not as easy to interpret. This is because, at a
high level, there is no longer a clear link between the local structure
exhibited by \smash{$\htheta$} and whether or not a particular entry of
\smash{$\kronmatk\htheta$} is nonzero. However, we show in Appendix \ref{app:df}
that it is possible to compute the right-hand side in \eqref{eq:ktf_df_unbiased}
with a simple, direct algorithm that runs in linear time (more precisely, the
algorithm requires $O(ndk)$ operations).

\section{Interlude: total variation on lines}
\label{sec:tv_lines}

In this section, we take a continuum perspective, looking at total variation
defined over functions in $\R^d$, and connect it to the discrete notion of total
variation used in the previous sections used to define the KTF estimator.

\subsection{Measure-theoretic total variation}

Let $U$ be an open, bounded subset of $\R^d$ and $L^p(U)$ denote the space of
real-valued functions on $U$ with finite $L^p$ norm, $\int_U |f(x)|^p \, dx <
\infty$. A function $f \in L^1(U)$ is said to be of \emph{bounded variation}
(BV) provided $\TV(f; U) < \infty$, where 
\begin{equation}
\label{eq:tv_aniso}
\TV(f; U) = \sup \bigg\{ \int_U f(x) \, \mathrm{div} \phi(x) \, dx : \phi \in  
C^\infty_c(U; \R^d), \; \|\phi(x)\|_\infty \leq 1 \; \text{for all $x \in U$}
\bigg\}.
\end{equation}
Above, \smash{$C^\infty_c(U; \R^d)$} denotes the space of infinitely
continuously differentiable functions from $U$ to $\R^d$ with compact support,
and $\mathrm{div}(\cdot)$ denotes the divergence operator, \smash{$\mathrm{div}
  \phi   = \sum_{j=1}^d \partial \phi_j/\partial x_j$}. We call $\TV(f; U)$ the  
\emph{total variation} of $f$; this is the standard measure-theoretic definition
used in modern analysis; see, for example, Chapter 5 of
\citet{evans2015measure}. To be clear, it would be more precise to call our
definition in \eqref{eq:tv_aniso} the \emph{aniostropic} total variation of $f$
(due to the use of the $\ell_\infty$ norm in the constraint in
\eqref{eq:tv_aniso} on the test function $\phi$), but we often drop the
reference to the anisotropic qualifier for simplicity.

To build intuition, we note that if $f \in W^{1,1}(U)$, that is, $f$ is in
$L^1(U)$ and it is weakly differentiable and its weak derivative $\nabla f$ is
also in $L^1(U)$, then
\begin{equation}
\label{eq:tv_smooth}
\TV(f; U) = \int_U \|\nabla f(x)\|_1 \, dx.
\end{equation}
Writing $\BV(U)$ for the space of bounded variation functions, the above shows
that $W^{1,1}(U) \subseteq \BV(U)$. Importantly, this is a strict inclusion,
because, for example, the indicator function of a set that has smooth boundary
is of bounded variation, but it is not in $W^{1,1}(U)$ (it is not weakly
differentiable).  

\subsection{Univariate total variation revisited}

In order to connect the discrete notions of TV that we use in this paper to the 
standard measure-theoretic definition of TV defined in \eqref{eq:tv_aniso}, we
must first refine our definition of univariate TV. For a function $g : [a,b] \to
\R$, we define its total variation as: 
\begin{equation}
\label{eq:tv_uni}
\TV(g; [a,b]) = \sup_{\substack{a < z_1 < \cdots < z_{m+1} < b \\ 
  z_1,\ldots,z_{m+1} \in \mathrm{AC}(g)}} \; \sum_{i=1}^m |g(z_i) - g(z_{i+1})|,  
\end{equation}
where the supremum is only taken over the set $\mathrm{AC}(g)$ of points of 
approximate continuity of $g$. Approximate continuity is a weak notion of
continuity that excludes, for example, point discontinuities; see Section 1.7.2
of \citet{evans2015measure}. Observe that the definition in \eqref{eq:tv_uni}
differs from that given in the introduction in that the latter does not require
that the supremum be taken over points of approximate continuity. Some authors,
including \citet{evans2015measure}, differentiate these definitions by calling
the latter the \emph{variation} of $g$ and \eqref{eq:tv_uni} the \emph{essential
  variation} of $g$. An intuitive way of interpreting their connection is as
follows: the essential variation of $g$ is the infimum of the variation
achievable by any function \smash{$\tilde{g}$} that agrees with $g$ Lebesgue 
almost everywhere. 

The reason the refinement in \eqref{eq:tv_uni} is important, when using the
measure-theoretic definition in \eqref{eq:tv_aniso} as a basis for defining the
BV space, is that BV functions (as with $L^p$ functions and Sobolev functions)
are only well-defined up to a set of Lebesgue measure zero. That is, if $f$ and
\smash{$\tilde{f}$} agree Lebesgue almost everywhere, then their TV as defined
in \eqref{eq:tv_aniso} (as with $L^p$ norms or Sobolev norms) must also
agree. Therefore, it should be clear that \eqref{eq:tv_uni} is the proper
univariate notion here, as otherwise redefining $g$ at a point would change its
univariate TV (without restricting the supremum to points of approximate
continuity). Lastly, and reassuringly, the multivariate measure-theoretic
definition in \eqref{eq:tv_aniso} reduces to the univariate definition in
\eqref{eq:tv_uni} once we take $U=(a,b)$ (see Theorem 5.21 in
\citet{evans2015measure}).

\subsection{Total variation on lines}

We now proceed in the opposite direction to the end of the last subsection:
instead of reducing the multivariate definition to the univariate case, we will
use the univariate definition of TV to approach the multivariate
one. Interestingly, as we will see next, it turns out that \eqref{eq:tv_uni} can
be used as a building block for \eqref{eq:tv_aniso} for an open, bounded set $U
\subseteq \R^d$. In words, the next result says that the multivariate notion of
TV on $U$ is given by aggregating the univariate notion on all line segments
parallel to the coordinate axes, anchored at boundary points of $U$.    

\begin{theorem}
\label{thm:tv_lines}
Let $U \subseteq \R^d$ be an open, bounded, convex set. Then for any $f \in
\BV(U)$,  
\begin{equation}
\label{eq:tv_lines}
\TV(f; U) = \sum_{j=1}^d \int_{U_{-j}} \TV \big( f(\cdot,x_{-j}); I_{x_{-j}}
\big) \, dx_{-j},   
\end{equation}
where for each $j=1,\ldots,d$, we define $U_{-j} = \{x_{-j} : (x_j, x_{-j}) \in
U \;\, \text{for some $x_j$}\}$, and \smash{$I_{x_{-j}} = [a_{x_{-j}},
  b_{x_{-j}}]$}, with  
\begin{align*}
a_{x_{-j}} &= \inf\{ x_j : (x_j, x_{-j}) \in U \}, \\
b_{x_{-j}} &= \sup\{ x_j : (x_j, x_{-j}) \in U \}.
\end{align*}
Recall $f(\cdot,x_{-j})$ denotes $f$ as function of the $j\th$ coordinate with 
all other dimensions fixed at $x_{-j}$. Lastly, the univariate TV operator in
the integrand in \eqref{eq:tv_lines} is to be interpreted in the essential
variation sense, as in \eqref{eq:tv_uni}.  
\end{theorem}

\begin{remark}
The assumption of convexity of $U$ in Theorem \ref{thm:tv_lines} is used for 
simplicity, to ensure that each coordinatewise slice of $U$---intersecting it
with a line segment parallel to the coordinate axis---is an interval. The proof
trivially extends to the case in which each slice is a finite union of
intervals; more complex structures could likely be handled via more complex
arguments.
\end{remark}

\begin{remark}
The above result is inspired by Theorem 5.22 of \citet{evans2015measure}. In the 
proof, we mollify $f$ in order to invoke the representation in
\eqref{eq:tv_smooth} for the TV of a smooth function, and then we leverage the 
separability of the $\ell_1$ norm (that is, we leverage the fact that the
integrand in \eqref{eq:tv_smooth} decomposes into a sum of absolute partial 
derivatives) in order to derive \eqref{eq:tv_lines}. Curiously, an analogous
result does not seem straightforward to derive for the case of isotropic TV;
this being defined by using an $\ell_2$ norm constraint on the test function 
$\phi$ in \eqref{eq:tv_aniso} (and for functions in $W^{1,2}(U)$, it would
reduce to the integral of $\ell_2$ norm of the weak derivative, instead of the
$\ell_1$ norm as in \eqref{eq:tv_smooth}).
\end{remark}

\subsection{Connection to KTV smoothness}

We connect the representation in \eqref{eq:tv_lines} to the KTF penalty
functional, which we call KTV smoothness. First note that we can rewrite the
definition of the anisotropic TV of a function $f$ in \eqref{eq:tv_aniso} as   
$$
\TV(f; U) = \sum_{j=1}^d \underbrace{
\sup \bigg\{ \int_U f(x) \; \frac{\partial \phi(x)}{\partial x_j} \, dx :   
\phi \in C^\infty_c(U), \; |\phi(x)| \leq 1 \; \text{for all $x \in U$}  
\bigg\}}_{V_j(f; U)},
$$
where \smash{$C^\infty_c(U)$} is the space of infinitely continuously
differentiable real-valued functions on $U$. Note that $V_j(f; U)$, as defined
above, measures the variation of $f$ along the $j\th$ coordinate
direction. Now consider the following definition of $k\th$ order multivariate 
TV, for an integer $k \geq 0$: 
\begin{equation}
\label{eq:tv_aniso_k}
\TV^k(f; U) = \sum_{j=1}^d V_j \bigg( \frac{\partial^k f}{\partial x_j^k}; U
\bigg), 
\end{equation}
where \smash{$\partial^k f/\partial x_j^k$} denotes the $k\th$ partial 
weak derivative of $f$ with respect to $x_j$. The formulation in
\eqref{eq:tv_aniso_k} is, in a sense, among the many possible options for 
higher-order TV in the multivariate setting, the ``most'' anisotropic. It only
looks at the variation in the partial derivatives along the coordinate
directions with respect to which they are defined (that is, the variation in the
$j\th$ partial derivative along the $j\th$ coordinate axis). 

Thanks to the representation in Theorem \ref{thm:tv_lines}, we can rewrite
the definition of $k\th$ order TV in \eqref{eq:tv_aniso_k} as: 
\begin{equation}
\label{eq:tv_lines_k}
\TV^k(f; U) = \sum_{j=1}^d \int_{U_{-j}} \TV \bigg( \frac{\partial^k
f(\cdot,x_{-j})}{\partial x_j^k} \bigg) \, dx_{-j},
\end{equation}
where we have made the dependence on the domain \smash{$I_{x_{-j}}$} in the TV
operator in the integrand implicit. Finally, we are ready to draw the connection
to KTF. Observe that the penalty functional underlying KTF, the second term in
the criterion of its continuous-time formulation \eqref{eq:ktf_continuous},
is given by taking the notion of $k\th$ order TV in \eqref{eq:tv_lines_k} and
approximating the integral via discretization; that is, the integral over
$U_{-j}$ is simply replaced by a sum over an embedded lattice.
\section{Smoothness classes}
\label{sec:classes}

We present various discrete smoothness classes, then connect them to each
other and to traditional Holder smoothness classes defined in continuous-time,
to derive what we refer to as \emph{canonical scalings} for the radii of the
discrete classes. This paves the way for the minimax analysis in the next
section.       

\subsection{Discrete TV and Sobolev classes} 

First we define the $k\th$ order \emph{Kronecker total variation} (KTV) class,
for a radius $\rho>0$, by 
\begin{equation}
\label{eq:ktv_class}
\kronset^k(\rho) = \big\{ \theta \in \R^n : \|\kronmatk \theta\|_1 \leq \rho  
\big\}. 
\end{equation}
It is a priori unclear what scaling for the radius $\rho$ in
\eqref{eq:ktv_class} makes for an ``interesting''  smoothness class for 
theoretical analysis. This is discussed at length in \citet{sadhanala2016total},
and was one of the original motivations for that paper. There, it is shown that 
taking $\rho$ to be a constant (as $n \to \infty$) leads to seemingly very fast 
minimax rates, however, in this regime trivial estimators turn out to be rate
optimal (such as the sample mean estimator \smash{$\htheta_i = \bar{y}$}, 
$i=1,\ldots,n$).     

In order to begin reasoning about scalings for $\rho$ in \eqref{eq:ktv_class},
it helps to define the order $k+1$ discrete $\ell_2$-Sobolev class:
\begin{equation}
\label{eq:sobolev_class}
\sobolset^{k+1}(\rho) = \big\{ \theta \in \R^n : \|\kronmatk \theta\|_2 \leq  
\rho \big\}. 
\end{equation}
Observe that \smash{$\sobolset^{k+1}(\rho)$} only considers partial derivatives
of order $k+1$ aligned with one of the coordinate axes, rather than considering
all mixed derivatives of total order $k+1$, as we would in a traditional Sobolev
class. For simplicity, we drop reference to the $\ell_2$ prefix when referring
to \eqref{eq:sobolev_class} henceforth.  

By the inequality \smash{$\|v\|_2 \leq \sqrt{p} \|\theta\|_1$} for vectors $v
\in \R^p$, and the fact that the number of rows of \smash{$\kronmatk$} can be 
upper bounded by $dn$, we have the following embedding:     
$$
\sobolset^{k+1}(\rho) \subseteq \kronset^k\big( \sqrt{dn} \rho \big), \quad  
\text{for any $\rho>0$}. 
$$
This shows that any reasonable regime for analysis must have $\rho$ varying with
$n$ in \eqref{eq:ktv_class}, or in \eqref{eq:sobolev_class} (or both), because a 
constant radius in one class would translate into a growing or diminishing
radius in the other, by the above display. However, it still leaves unspecified
what precise scalings for the radii in \eqref{eq:ktv_class} and 
\eqref{eq:sobolev_class}, would correspond to ``interesting'' classes,
comparable in some sense to choices of radii in analogous continuous-time TV or
Sobolev smoothness classes. We answer this question in the next subsection, by
introducing discrete and continuum Holder classes, and pursuing further
embeddings. 

\subsection{Discrete and continuum Holder classes}

Now we recall the traditional definition for the $k\th$ order Holder class of
functions from $[0,1]^d$ to $\R$, of radius $L>0$:
\begin{multline*}
C^k(L; [0,1]^d) = \Bigg\{ f : [0,1]^d \to \R: \text{$f$ is $k$ times
  differentiable and for all integers $\alpha_1,\ldots,\alpha_d \geq
  0$}, \\ \text{with $\alpha_1+\cdots+\alpha_d=k$}, \;\, 
\bigg| \frac{\partial^k f(x)}{\partial x_1^{\alpha_1} \cdots \partial
  x_d^{\alpha_d}} - \frac{\partial^k f(z)}{\partial x_1^{\alpha_1} \cdots
  \partial x_d^{\alpha_d}} \bigg| \leq L \|x-z\|_2, \;\, 
\text{for all $x,z \in [0,1]^d$} \Bigg\}.  
\end{multline*}
We define a discretized version of this class by simply evaluating the functions 
in \smash{$C^k(L; [0,1]^d)$} on the lattice $Z_{n,d}$:
\begin{equation}
\label{eq:holder_class}
\holderset^k(L) = \big\{ \theta \in \R^n : \text{there exists some $f \in  
  C^k(L; [0,1]^d)$ such that $\theta(x) = f(x)$, $x \in Z_{n,d}$} \big\}. 
\end{equation}
The next proposition derives embeddings for the discrete Holder class
\eqref{eq:holder_class} into the discrete Sobolev \eqref{eq:sobolev_class} and
KTV \eqref{eq:ktv_class} classes. It is a direct consequence of Lemma A.6 in
\citet{sadhanala2017higher} (which is a classical result of sorts that
quantifies the error of the forward difference approximation of the derivative
of a Holder function).     

\begin{proposition}[\citealt{sadhanala2017higher}]
\label{prop:ktv_sobolev_holder_embed}
The discrete classes in \eqref{eq:ktv_class}--\eqref{eq:holder_class} satisfy,
for any $L > 0$, 
\begin{equation}
\label{eq:ktv_sobolev_holder_embed}
\holderset^k(L) 
\subseteq \sobolset^{k+1} \big( c_1 L n^{\frac{1}{2}-\frac{k+1}{d}} \big)  
\subseteq \kronset^k \big( c_2 L n^{1-\frac{k+1}{d}} \big),
\end{equation}
where $c_1,c_2>0$ are constants depending only on $k,d$. 
\end{proposition}

Motivated by the last result, as in \citet{sadhanala2017higher}, we define the
{\it canonical scalings} for the discrete Sobolev and KTV classes as 
\begin{align}
\label{eq:sobolev_canonical}
B_n^* &= n^{\frac{1}{2}-\frac{k+1}{d}}, \\
\label{eq:ktv_canonical}
C_n^* &= n^{1-\frac{k+1}{d}},
\end{align}
so that \smash{$\holderset^k(1) \subseteq \sobolset^{k+1}(c_1B_n^*) \subseteq
  \kronset^k(c_2C_n^*)$}, for constants $c_1,c_2>0$ that depend only on
$k,d$. Thus, by analogy to classical results on nonparametric estimation over
Holder spaces, we should expect the minimax rate over
\smash{$\sobolset^{k+1}(B_n^*)$} (in the squared $\ell_2$ norm) to be
\smash{$n^{-2(k+1)/(2(k+1)+d)}$}. This is indeed the case, as we will show at
the end of Section \ref{sec:theory}. The minimax rate over
\smash{$\kronset^k(C_n^*)$}, on the other hand, will turn out to be more exotic,
and is the focus of the majority of the next section.
\section{Estimation theory}
\label{sec:theory}

We derive a number of results on estimation theory over KTV classes. We begin
by deriving upper bounds on the error of the KTF estimator, and then study lower 
bounds. Throughout, we assume the data model in \eqref{eq:data_model1} with
$\theta_{0i} = f_0(x_i)$, $i=1,\ldots,n$ and i.i.d.\ normal errors, to be
precise: 
\begin{equation}
\label{eq:data_model2}
y_i \sim N(\theta_{0,i},\sigma^2), \quad \text{independently, for
  $i=1,\ldots,n$}.  
\end{equation}
To set some basic notation, based on estimators \smash{$\htheta$} of the mean
$\theta_0$ in \eqref{eq:data_model2}, we define for a subset $\cK \subseteq
\R^n$,  
$$
R(\cK) = \inf_{\htheta} \; \sup_{\theta_0 \in \cK} \; \frac{1}{n} \E \|\htheta -
\theta_0\|_2^2,  
$$
which is called the \emph{minimax risk} over $\cK$. Also of interest will be  
$$
R_L(\cK) = \inf_{\htheta \; \text{linear}} \; \sup_{\theta_0 \in \cK} \;
\frac{1}{n} \E \|\htheta - \theta_0\|_2^2,    
$$
called the \emph{minimax linear risk} over $\cK$, the infimum being restricted
to linear estimators \smash{$\htheta$} (that is, of the form \smash{$\htheta
  = S y$} for a matrix $S \in \R^{n \times n}$). To finish our discussion of
notation, for deterministic sequences $a_n,b_n$ we write $a_n = O(b_n)$ when
$a_n/b_n$ is upper bounded for large enough $n$, we write $a_n = \Omega(b_n)$
when $a_n^{-1} = O(b_n^{-1})$, and $a_n \asymp b_n$ when both $a_n=O(b_n)$ and
$a_n = \Omega(b_n)$. For random sequences $A_n,B_n$, we write $A_n = O_\P(B_n)$ 
when $A_n/B_n$ is bounded in probability. In the theory that follows, all
asymptotics are for $n \to \infty$ with $k,d$ fixed.  

As a side remark, although not the focus of the current paper, analogous theory
can be established for graph trend filtering on grids (see Section
\ref{sec:related_work} for a discussion of its relation to KTF), which we defer
to Appendix \ref{app:gtf}.  

\subsection{Upper bounds on estimation risk}

To derive upper bounds on the risk of KTF, we leverage the following simple
generalization of a key result from \citet{wang2016trend}. Here and henceforth,
for an integer $a \geq 1$, we abbreviate $[a] = \{1,\ldots,a\}$.  

\begin{theorem}[\citealt{wang2016trend}]
\label{thm:genlasso_upper_bd}
Consider the generalized lasso estimator \smash{$\htheta$} with penalty matrix 
$\genmat \in \R^{r \times n}$, defined by the solution of 
\begin{equation}
	\label{eq:genlasso_estimator}
\minimize_{\theta \in \R^n} \; \frac{1}{2} \|y-\theta\|_2^2 + \lambda \|\genmat 
\theta\|_1
\end{equation}
Suppose that $\genmat$ has rank $q$, and denote by $\xi_1 \leq \cdots \leq 
\xi_q$ its nonzero singular values. Also let $u_1, \ldots, u_q \in \R^r$ be the 
corresponding left singular vectors. Assume that these vectors, except possibly
for those in a set $I \subseteq [q]$, are \emph{incoherent}, meaning that for a 
constant $\mu \geq 1$,   
$$
\|u_i\|_\infty \leq \mu/\sqrt{n}, \quad i \in [q] \setminus I.
$$
Then under the data model \eqref{eq:data_model2}, choosing 
$$
\lambda \asymp \mu \sqrt{\frac{\log r}{n} \sum_{i\in [q] \setminus I}
  \frac{1}{\xi_i^2}},   
$$
the generalized lasso estimator satisfies 
\begin{equation}
\label{eq:genlasso_upper_bd}
\frac{1}{n} \|\htheta - \theta_0\|_2^2  =  O_\P \Bigg(  
\frac{\nuli(\genmat)}{n} + \frac{|I|}{n} + \frac{\mu}{n} \sqrt{ \frac{\log r}{n}    
\sum_{i\in [q] \setminus I} \frac{1}{\xi_i^2}} \cdot \| \genmat \theta_0 \|_1\,
\Bigg).    
\end{equation}
\end{theorem}

We will now apply this result to KTF in \eqref{eq:ktf}, and choose the set $I$
in order to balance the second and third terms on the right-hand side in
\eqref{eq:genlasso_upper_bd}. Throughout, we will make reference to the 
\emph{effective degree of smoothness} (or effective smoothness for short),
defined 
by 
$$
s = \frac{k+1}{d}.
$$

\begin{theorem}
\label{thm:ktf_upper_bd}
Let \smash{$\htheta$} denote the KTF estimator in \eqref{eq:ktf}. Under the data
model \eqref{eq:data_model2}, denote \smash{$C_n = \|\kronmatk \theta_0\|_1$},
and assume $C_n>0$. Choosing    
\begin{align*}
\lambda \asymp 
\begin{cases}
\sqrt{\log n} & \text{if $s < 1/2$}, \\
\log n \vphantom{(n / C_n)^{\frac{2s-1}{2s+1}}} & \text{if $s = 1/2$}, \\  
(\log n)^{\frac{1}{2s+1}} (n / C_n)^{\frac{2s-1}{2s+1}} & \text{if $s > 1/2$},     
\end{cases}
\end{align*}
the KTF estimator satisfies 
$$
\frac{1}{n} \|\htheta - \theta_0\|_2^2 = O_\P \bigg( \frac{1}{n} +
\frac{\lambda}{n} C_n \bigg).
$$
\end{theorem}

\begin{remark}
The result in the above theorem for the case of $k=0$ (TV denoising) and any $d$ 
was already established in \citet{hutter2016optimal}. The result for $d=2$ and
any $k$ was given in \citet{sadhanala2017higher}. For $d=1$ and any $k$, the
result was established in \citet{mammen1997locally, tibshirani2014adaptive}
(though the results in the latter papers are sharper by log factors). 

Theorem \ref{thm:ktf_upper_bd} covers all $k$ and $d$. Its proof, an application
of Theorem \ref{thm:genlasso_upper_bd} to KTF, requires checking the incoherence
of the penalty matrix \smash{$\genmat = \kronmatk$}, and bounding the partial  
sum of squared reciprocal singular values \smash{$\sum_{i\in [q] \setminus I} \, 
  \xi_i^{-2}$}, for an appropriate set $I$ that corresponds to small
eigenvalues. The former---checking the incoherence of the left singular vectors
of the KTF penalty matrix---turns out to be the harder calculation. However, the
hardest part of this calculation was already done in
\citet{sadhanala2017higher}, who established the incoherence of
\smash{$\dmat^{(k+1)}_n$} (the trend filtering penalty matrix, or equivalently,
the KTF penalty matrix when $d=1$), using complex approximation results for
the eigenvectors of Toeplitz matrices. To handle KTF in arbitrary dimension $d$,
in the current paper, we use careful arguments that relate singular vectors of
Kronecker products to singular vectors of their constituent matrices.   
\end{remark}

\begin{remark}
The choice of tuning parameter $\lambda$ in Theorem \ref{thm:ktf_upper_bd}
generally depends on the smoothness level $C_n$, as well as the order $k$ of the
underlying KTV smoothness class (which appears through $s = (k+1) / d$). In this 
sense, our KTF bounds are weaker than some of the sharpest results in the
literature on nonparametric estimation, which are adaptive to underlying
smoothness parameters ($k$ and $C_n$ in our setting). However, it is interesting
to note that the case $s \leq 1/2$, that is, $2k+2\leq d$, appears to be
special: KTF ends up being adaptive to $C_n$. In other words, the prescribed
choice of tuning parameter is \smash{$\lambda \asymp \sqrt{\log n}$} when $s <
1/2$, and \smash{$\lambda \asymp \log n$} when $s = 1/2$, neither of which 
depend on $C_n$. 
\end{remark}

Under the canonical scaling \eqref{eq:ktv_canonical}, the error bound in Theorem
\ref{thm:ktf_upper_bd} reduces to the following.   

\begin{corollary}
\label{cor:ktf_upper_bd_canonical}
Assume the conditions of Theorem \ref{thm:ktf_upper_bd}. When \smash{$C_n \asymp
  L_n C_n^* = L_n n^{1-s}$}, where $C_n^*$ is the canonical scaling in
\eqref{eq:ktv_canonical} for the KTV smoothness level, the KTF estimator
satisfies     
\begin{equation}
\begin{aligned}
\label{eq:ktf_upper_bd_canonical}
\frac{1}{n} \|\htheta - \theta_0\|_2^2 =
\begin{cases}
O_\P (L_n n^{-s} \sqrt{\log n} \, \vphantom{L_n^{\frac{2}{2s+1}}}) & \text{if $s
  < 1/2$}, \\   
O_\P (L_n n^{-s}  \log n) \vphantom{L_n^{\frac{2}{2s+1}}} & \text{if $s = 1/2$},
\\   
O_\P \big( L_n^{\frac{2}{2s+1}} n^{-\frac{2s}{2s+1}} (\log n)^{\frac{1}{2s+1}}
\big) & \text{if $s > 1/2$}.  
\end{cases}
\end{aligned}
\end{equation}
\end{corollary}

\begin{remark}
Recall the continuous-time analog of the KTF penalty in
\eqref{eq:ktf_continuous}. It turns out, as we show in Appendix
\ref{app:ktf_upper_bd_continuous}, that the assumption that the mean vector
$\theta_0$ is KTV smooth in Theorem \ref{thm:ktf_upper_bd} and Corollary
\ref{cor:ktf_upper_bd_canonical} can be broadened into an assumption that the
KTV of the mean \emph{function} $f_0$ is KTV smooth, in the sense of the penalty
functional in \eqref{eq:ktf_continuous}. 
\end{remark}

\subsection{Lower bounds on estimation risk}
\label{sec:lower_bounds}

Next we give lower bounds on the minimax risk over KTV classes. 

\begin{theorem}
\label{thm:ktv_lower_bd}
The minimax risk for KTV class defined in \eqref{eq:ktv_class} satisfies, for
any sequence $C_n \leq n$, 
\begin{equation}
\label{eq:ktv_lower_bd}
R\big(\kronset^k(C_n)\big) = \Omega \Bigg( \frac{1}{n} + \frac{C_n}{n} +
\bigg(\frac{C_n}{n} \bigg)^{\frac{2}{2s+1}} \Bigg). 
\end{equation}
\end{theorem}

\begin{remark}
The result in Theorem \ref{thm:ktv_lower_bd} for $s \geq 1/2$ (that is, $2k+2
\geq d$) was already derived in \citet{sadhanala2017higher}. More precisely,
these authors established the third term on the right-hand side in
\eqref{eq:ktv_lower_bd}, by using the Holder embedding in
\eqref{eq:ktv_sobolev_holder_embed} of Proposition
\ref{prop:ktv_sobolev_holder_embed}, and suitably adapting classical results on
minimax bounds for Holder spaces \citep{korostelev1993minimax,
tsybakov2009introduction}, which ends up being tight in the $s > 1/2$ regime.
Moreover, the lower bound in the middle term in \eqref{eq:ktv_lower_bd} was
derived by \citet{sadhanala2016total}, for $k=0$ and all $d$, which was obtained 
by embedding an appropriate $\ell_1$ ball into \smash{$\kronset^k(C_n)$}, and
appealing to minimax theory over $\ell_1$ balls from
\citet{birge2001gaussian}. This ends up being tight for $k=0$ and all $d$. 

In the current work, we establish tight lower bounds over all $k,d$, by
essentially combining these two strategies. As we will see comparing to upper
bounds in the next remark, the Holder embedding ends up being tight for $s >
1/2$, and the $\ell_1$ embedding for $s \leq 1/2$.
\end{remark}

\begin{remark}
\label{rem:ktv_lower_bd_canonical}
Plugging in \smash{$C_n \asymp L_n C_n^* = L_n n^{1-s}$}, where $C_n^*$ is the
canonical scaling in \eqref{eq:ktv_canonical}, and simplifying to keep only the 
dominant terms, we see that \eqref{eq:ktv_lower_bd} becomes  
$$
R\big(\kronset^k(C_n^*)\big) = 
\begin{cases}
\Omega(L_n n^{-s} \vphantom{L_n^{\frac{2}{2s+1}}}) & \text{if $s \leq 1/2$}, \\ 
\Omega\big( L_n^{\frac{2}{2s+1}} n^{-\frac{2s}{2s+1}} \big) & \text{if $s >
  1/2$}.   
\end{cases}
$$
By comparing this to \eqref{eq:ktf_upper_bd_canonical}, we can see that KTF is 
minimax rate optimal for estimation over KTV classes, under the canonical
scaling, up to log factors. 
\end{remark}

\begin{remark}
From the upper bound in \eqref{eq:ktf_upper_bd_canonical} and the Holder 
embedding in \eqref{eq:ktv_sobolev_holder_embed}, we see that for $s \geq
1/2$, KTF achieves the rate of \smash{$n^{-2s/(2s+1)}$} (up to log factors) over
\smash{$\cH_d^k(1)$}. This matches the optimal rate for estimation over Holder
classes (see \citet{sadhanala2017higher} for a formal statement and proof for
the discretized class \smash{$\cH_d^k(1)$}), which means that KTF automatically
adapts Holder smooth signals.      

However, when $s < 1/2$, the same results show that KTF achieves a rate of  
$n^{-s}$ (ignoring log factors) over \smash{$\cH_d^k(1)$}, which is slower than  
the optimal rate of \smash{$n^{-2s/(2s+1)}$}. Whether this upper bound is
pessimistic and KTF can actually adapt to \smash{$\cH_d^k(1)$}, or whether
this upper bound is tight and KTF fails to adapt to Holder smooth signals,
remains to be formally resolved. Later, we investigate this empirically in
Section \ref{sec:homogeneous_smoothness}.
\end{remark}

\begin{remark}
\label{rem:besov_minimax}
It is interesting to compare minimax results for anisotropic Besov spaces, under
the white noise model. Past work on this topic includes \citet{neumann2000multi}
with a focus on hyperbolic wavelets, as well as
\citet{kerkyacharian2001nonlinear, kerkyacharian2008nonlinear} with a focus on
Lepski's method applied to kernel smoothing. A nice summary of past work, and
what appears to be the most comprehensive results, can be found in
\citet{lepski2015adaptive}. It should be noted that this line of work considers
a more general setup than ours (albeit in the white noise model), in a few ways:
anisotropic Besov classes with an arbitrary smoothness index in each coordinate
direction; error measured in $L^p$ norm, for arbitrary $p \geq 1$; and so
on. That said, translating their results to match our setting as best as we can
(recalling general embeddings of BV spaces into Besov spaces; for example,
\citet{devore1993constructive}), we find that the minimax rate under the squared
$L^2$ loss, for the anisotropic Besov class with integrability index 1,
smoothness index $k+1$ in each coordinate direction, and any third index $q \geq
1$, is indeed \smash{$n^{-2s/(2s+1)}$} for $s > 1/2$ (or $2k+2 > d$). This
regime is what Lepski and others refer to as the ``dense zone''.

In what they refer to as the ``sparse zone'', $s \leq 1/2$, the minimax risk
under the $L^2$ loss for the same Besov class is a constant. This is due to the
fact that this Besov space fails to embed compactly into $L^2$ (see Section 5.5
of \citet{johnstone2015gaussian} for a general discussion of the implications of 
this phenomenon). If the underlying regression function is itself additionally
assumed to be bounded in $L^\infty$ norm, then we believe the minimax rate in
their white noise setting will be $n^{-s}$, matching that in our discrete
setting. Evidence of this claim includes the result in
\citet{delalamo2021frameconstrained} for the BV space (under the white noise
model), who derive a rate of $n^{-1/d}$, assuming such $L^\infty$ boundedness;
as well as the Besov results in \citet{goldenshluger2014adaptive}, who also
assume $L^\infty$ boundedness, but study density estimation.  
\end{remark}

\subsection{Minimax rates for linear smoothers}

We now study whether linear smoothers can achieve the minimax rate over the KTV
class in \eqref{eq:ktv_class}. Before stating our result, we define a truncated 
eigenmaps estimator based on \smash{$\kronmatk$} as follows. Denote by $\xi_i
\geq 0$, $i \in [N]^d$ its singular values (noting that $(k+1)^d$ of these are
zero), where along each dimension in the multi-index, the singular values 
are sorted in increasing order. Denote also by $v_i \in \R^n$, $i \in [N]^d$ its
corresponding right singular vectors. Then, for a subset $Q \subseteq [N]^d$, we
denote by $V_Q \in \R^{n \times |Q|}$ the matrix with columns given by $v_i$, $i 
\in Q$, and define the projection estimator   
\begin{equation}
\label{eq:eigenmaps}
\htheta = V_Q V_Q^\T y.
\end{equation}
Note that this reduces to the Laplacian eigenmaps estimator (here the Laplacian 
is that of the grid graph) when $k=0$.  

\begin{theorem}
\label{thm:ktv_minimax_linear}
The minimax linear risk over the KTV class in \eqref{eq:ktv_class} satisfies,
for any sequence \smash{$C_n \leq \sqrt{n}$},  
\begin{equation}
\label{eq:ktv_lower_bd_linear}
R_L\big( \kronset^k(C_n) \big) =
\begin{cases}
\Omega(1/n + C_n^2/n) & \text{if $s < 1/2$}, \\ 
\Omega(1/n + C_n^2/n \log(1 + n / C_n^2)) & \text{if $s = 1/2$}, \\      
\Omega\big( 1/n + (C_n^2/n)^{\frac{1}{2s}} \big) & \text{if $s > 1/2$}.      
\end{cases}
\end{equation}
This is achieved in rate by the projection estimator in \eqref{eq:eigenmaps},   
where we set $Q = [\tau]^d$ for \smash{$\tau^d \asymp (C_n n^{s-1/2})^{1/s}$},  
in the case $s > 1/2$. When $s < 1/2$, the simple polynomial projection
estimator, which projects onto all multivariate polynomials of max degree $k$ 
(equivalently, the estimator in \eqref{eq:eigenmaps} with $Q = [k+1]^d$),
achieves the rate in \eqref{eq:ktv_lower_bd_linear}. When $s=1/2$, either
estimator achieves the rate in \eqref{eq:ktv_lower_bd_linear} up to a log
factor. Lastly, if \smash{$C_n^2 = O(n^\alpha)$} for $\alpha < 1$, and still 
$s=1/2$, then either estimator achieves the rate in
\eqref{eq:ktv_lower_bd_linear} without the additional log factor. 
\end{theorem}

\begin{remark}
\label{rem:ktv_minimax_linear_canonical}
Plugging in \smash{$C_n \asymp L_n C_n^* = L_n n^{1-s}$}, where $C_n^*$ is the
canonical scaling in \eqref{eq:ktv_canonical}, and simplifying to keep only the
dominant terms (when $L_n \geq 1$), we see that \eqref{eq:ktv_lower_bd_linear}
becomes    
$$
R_L\big(\kronset^k(C_n^*)\big) =
\begin{cases}
\Omega(1) & \text{if $s \leq 1/2$}, \\ 
\Omega\big( L_n^{\frac{1}{s}} n^{-\frac{2s-1}{2s}} \big) & \text{if $s > 1/2$}.   
\end{cases}
$$
These minimax linear rates display a stark difference to the minimax rates for 
the KTV class in Remark \ref{rem:ktv_lower_bd_canonical} (achieved by KTF up to
log factors, in \eqref{eq:ktf_upper_bd_canonical}). When $s > 1/2$, the minimax
linear rate of \smash{$L_n^{\frac{1}{s}} n^{-(2s-1)/(2s)}$} can be interpreted
as the rate of the optimal nonlinear estimator for a problem whose effective
smoothness level has been decremented by a half ($s-1/2$ in place of $s$). More
dramatically, when $s \leq 1/2$, we see that \emph{no linear smoother is
  consistent} over the KTV class, in the sense of worst-case risk. This
contributes an interesting addition to the line of work on the suboptimality of
linear smoothers for nonparametric regression over heterogeneous smoothness
classes, dating back to \citet{nemirovski1985rate, donoho1998minimax}.       
\end{remark}

\subsection{Summary of rates}

Table \ref{tab:theory} presents a summary of the minimax rates derived in the
previous three subsections. (It offers a more detailed summary than Figure
\ref{fig:theory}.) The upper bound on minimax risk is from Corollary
\ref{cor:ktf_upper_bd_canonical}, the lower bound on the minimax risk from
Theorem \ref{thm:ktv_lower_bd} and Remark \ref{rem:ktv_lower_bd_canonical}, and
the minimax linear risk is from Theorem \ref{thm:ktv_minimax_linear} and Remark
\ref{rem:ktv_minimax_linear_canonical}.

\renewcommand\arraystretch{1.5}
\begin{table}[htb]
\centering
\begin{tabular}{|c|c|c|c|}
\hline
Regime & $R$, upper bound & $R$, lower bound & $R_L$, linear risk \\
\hline
$s < 1/2$ &  $n^{-s} \sqrt{\log n}$ & $n^{-s}$ & 1 \\ 
$s = 1/2$ & $n^{-\frac{1}{2}} \log n$ & $n^{-\frac{1}{2}}$ & 1 \\ 
$s > 1/2$ & $n^{-\frac{2s}{2s+1}} (\log n)^{\frac{1}{2s+1}}$ & 
$n^{-\frac{2s}{2s+1}}$ & $n^{-\frac{2s-1}{2s}}$ \\
\hline
\end{tabular}
\caption{\it\small Minimax rates over the KTV class \smash{$\kronset^k(C_n^*)$},
  where the canonical scaling is \smash{$C_n^* = n^{1-s}$}, and recall
  $s=(k+1)/d$. We use the abbreviations \smash{$R = R(\kronset^k(C_n^*))$} and
  \smash{$R_L = R_L(\kronset^k(C_n^*))$}.}  
\label{tab:theory}
\end{table}

\subsection{Minimax rates over Sobolev classes}

For completeness, we establish the minimax rate for the (discrete) Sobolev class  
defined in \eqref{eq:sobolev_class}. The lower bounds are simply those from
the Holder class (due to the embedding in Proposition
\ref{prop:ktv_sobolev_holder_embed}), and as we show next, this is achieved in
rate by the eigenmaps estimator in \eqref{eq:eigenmaps}.

\begin{theorem}
\label{thm:sobolev_minimax}
The minimax risk over the (discrete) Sobolev class in \eqref{eq:sobolev_class}
satisfies, for any sequence \smash{$B_n \leq \sqrt{n}$},   
$$
R \big( \sobolset^{k+1}(B_n) \big) \asymp \frac{1}{n} + \bigg(\frac{B_n^2}{n}
\bigg)^{\frac{1}{2s+1}}.
$$
The lower bound is due to the Holder embedding in
\eqref{eq:ktv_sobolev_holder_embed} (and the lower bound on the discretized
Holder class derived in \citet{sadhanala2017higher}), and the upper bound is
from the estimator in \eqref{eq:eigenmaps}, with $Q = [\tau]^d$ for
\smash{$\tau^d \asymp (B_n^2 n^{2s})^{1/(2s+1)}$}. Finally, when \smash{$B_n
  \asymp L_n B_n^*$} where $B_n^*$ is the canonical scaling in
\eqref{eq:sobolev_canonical}, the minimax rate is \smash{$L_n^{2/(2s+1)}
  n^{-2s/(2s+1)}$}.    
\end{theorem}
\section{Optimization algorithms}
\label{sec:optimization}

In this section, we describe a number of numerical algorithms for solving the
convex KTF problem \eqref{eq:ktf}, analyze their asymptotic time complexity, and  
benchmark their performance. In particular, we will describe a family of
specialized ADMM algorithms that adapt to the structure of the KTF problem, and
as we will show, can find moderately accurate solutions much faster than a
general purpose ``off-the-shelf'' solver. This general purpose solver can be
applied to the dual of \eqref{eq:ktf}, which is a simple box-constrained
quadratic program, and is amenable to standard interior point methods. Details 
are deferred to Appendix \ref{app:more_opt}.

\subsection{Specialized ADMM algorithms}

Motivated by the popularity of operating splitting methods in machine learning 
over the last decade, and specifically by their success in application to trend 
filtering and TV denoising problems (for example, see \citet{ramdas2016fast,
  wang2016trend, barbero2018modular}), we consider the application of similar
methods to Kronecker trend filtering. We consider a proximal Dykstra algorithm, 
Douglas-Rachford splitting, and a family of specialized ADMM algorithms. For
brevity, the details on the former two are deferred to Appendix
\ref{app:more_opt}. 

Our specialized ADMM approach is inspired by that of \citet{ramdas2016fast}, for
univariate trend filtering. To reformulate \eqref{eq:ktf} into ``ADMM form'',
where the criterion decomposes as a sum of functions of separate optimization
variables, we must introduce auxiliary variables. To do so, we rely on the
following observation that decomposes the KTF penalty operator into the product
of a block-diagonal matrix and a lower-order KTF penalty operator. 

\begin{proposition}
\label{prop:ktf_decomp}
For each $j=1,2,\ldots,k+1$, the KTF penalty operator in \eqref{eq:ktf_pen_mat}
obeys (where \smash{$\dmat^{(0)}_N = I_N$} for convenience):
$$
\kronmatk  = \underbrace{\begin{bmatrix}
\dmat^{(k+1-j)}_N \otimes I_N \otimes \cdots \otimes I_N  & & \\    
& I_N \otimes \dmat^{(k+1-j)}_N \otimes \cdots \otimes I_N & \\ 
& & \ddots & \\
& & & I_N \otimes I_N \otimes \cdots \otimes \dmat^{(k+1-j)}_N  \\ 
\end{bmatrix}}_{M^{(k+1-j)}_{n,d}} \kronmat^{(j)}.
$$
\end{proposition}

This follows directly from the the univariate recursion in \eqref{eq:diff_mat},
and the Kronecker structure in \eqref{eq:ktf_pen_mat} (particularly, the 
mixed-product property of Kronecker products), and therefore we omit its proof.
Observe that each diagonal block of the block-diagonal matrix
\smash{$M^{(k+1-j)}_{n,d}$} is itself---possibly after appropriate permutation
of the row and column order---a block-diagonal matrix with all diagonal blocks
equal to \smash{$\dmat^{(k+1-j)}_N$}. This is a key fact that we will leverage
shortly.   

Now fix any $j \in \{1,\ldots,k+1\}$. We can reformulate \eqref{eq:ktf} as: 
\begin{equation}
\label{eq:ktf_admm}
\begin{alignedat}{2}
&\minimize_{\theta,z} \quad && \frac{1}{2} \|y-\theta\|_2^2 + 
\lambda\|M^{(k+1-j)}_{n,d} z\|_1 \\
&\st && z =   \kronmat^{(j)} \theta.
\end{alignedat}
\end{equation}
The augmented Lagrangian associated with \eqref{eq:ktf_admm}, for an augmented 
Lagrangian parameter $\rho \geq 0$, is:
$$
L_\rho(\theta, z,u) = \frac{1}{2}\|y-\theta\|_2^2 + 
\lambda\|M^{(k+1-j)}_{n,d}  z\|_1 + 
\frac{\rho}{2} \|z - \kronmat^{(j)} \theta + u\|_2^2 - 
\frac{\rho}{2} \|u\|_2^2. 
$$
ADMM iteratively performs a separate minimization over the primal variables
$\theta,z$, and then updates the dual variable $u$ via gradient ascent. Namely,
given some initialization \smash{$\theta^{(0)}, z^{(0)}, u^{(0)}$}, it repeats
the following, for $t=1,2,3,\ldots$:  
\begin{align}
\label{eq:ktf_admm_theta}
\theta^{(t)} &= \Big( I_n + \rho [\kronmat^{(j)}]^\T\kronmat^{(j)} \Big)^{-1}  
\Big( y + \rho [\kronmat^{(j)}]^\T(z^{(t-1)}+ u^{(t-1)}) \Big), \\
\label{eq:ktf_admm_z}
z^{(t)} &=  \prox_{\frac{\lambda}{\rho}\|M^{(k+1-j)}_{n,d} (\cdot)\|_1}
\Big( \kronmat^{(j)} \theta^{(t)}  -  u^{(t-1)} \Big), \\
\label{eq:ktf_admm_u}
u^{(t)} &= u^{(t-1)} +z^{(t)} - \kronmat^{(j)} \theta^{(t)}.
\end{align}
Here we use the notation $\prox_h(\cdot)$ for the proximal operator associated
with a function $h$. Below we make a few remarks about the computational costs
associated with the updates \eqref{eq:ktf_admm_theta}--\eqref{eq:ktf_admm_u}. 

\begin{itemize}
\item When $j=1$, the matrix \smash{$[\kronmat^{(j)}]^\T\kronmat^{(j)}$} in  
	\eqref{eq:ktf_admm_theta} is the graph Laplacian of the $d$-dimensional grid, 
	which decomposes into the Kronecker sum of Laplacians of (univariate) chain
  graphs. This can be diagonalized by a $d$-dimensional discrete cosine
  transform (DCT) (see, for example, the proof of Corollary 8 in 
  \citet{wang2016trend}).  Computationally, this is essentially sequentially 
  applying univariate DCTs to every dimension. 
	This implies that the $\theta$-update in \eqref{eq:ktf_admm_theta} can be
  done in $O(n\log n)$ time. Further improvements (to linear-time) should be
  possible with multi-grid methods. 

\item By the key fact mentioned after the proposition, the $z$-update in
  \eqref{eq:ktf_admm_z} can be decomposed into $d N^{d-1}$ univariate trend  
  filtering problems, each of order $k+1-j$ and each with sample size
  $N$. These can be solved in parallel, in a total time that is nearly-linear in
  $n$, using either the univariate ADMM approach of \citet{ramdas2016fast} or
  the primal-dual interior point method (PDIP) of \citet{kim2009trend}. PDIP is
  likely the best option for reasonably small $N$, and using it, the update
  \eqref{eq:ktf_admm_z} can be done in \smash{$O(dN^{d-1}N^{1.5}) = O(n^{1 +
      1/(2d)})$} time.  

\item When $j=k$ or $k+1$, the $z$-update \eqref{eq:ktf_admm_z} can be performed
  even more efficiently, in $O(n)$ time. This is because it reduces to
  soft-thresholding for $j=k+1$, and reduces to separate univariate TV denoising
  problems for $j=k$. In the former case, the linear time complexity is obvious; 
  in the latter, it is due to the dynamic programming (DP) method of
  \citet{johnson2013dynamic}.  

\item The case $k=0$ is quite favorable, as we can use DCT in
  \eqref{eq:ktf_admm_theta} and soft-thresholding in \eqref{eq:ktf_admm_z}, by
  choosing $j=1$, or simple coordinatewise shrinkage in
  \eqref{eq:ktf_admm_theta} and DP in \eqref{eq:ktf_admm_z}, by choosing
  $j=0$. Both are efficient, but the latter ends up being generally the better
  approach, and can be seen as the ADMM-analog of \citet{barbero2018modular}.    

\item Among the higher-order cases $k \geq 1$, the case $k=1$ ends up being
  quite special, because $j=1$ simultaneously supports the DCT solver in
  \eqref{eq:ktf_admm_theta} and DP in \eqref{eq:ktf_admm_z}. When $k \geq 2$, we
  essentially need to decide in between these highly efficient subroutines
  (choosing either $j=1$ or $j=k$).     
\end{itemize}

In summary, each iteration (cycle of updates over $\theta,z,u$) of the proposed
ADMM algorithm is $O(n)$ for $k=0,1$. For $k \geq 2$, the time complexity is
\smash{$O(n^{1+1/(2d)})$} when we choose $j=1$, which we refer to as ADMM
\emph{Type I}. When we choose $j=k$, which we refer to as ADMM \emph{Type II},
the time complexity is dominated by the sparse linear system solve in
\eqref{eq:ktf_admm_theta}. This linear system should be well-conditioned for
reasonable ranges of $\rho$, thus the standard conjugate gradient method will be
able to solve it in approximately linear-time (proportional to the number of
nonzero elements).

\subsection{Empirical comparisons}

We now compare the ADMM algorithms developed in the last subsection to
Douglas-Rachford and proximal Dykstra algorithms applied to \eqref{eq:ktf}, as
well as the Gurobi general purpose solver (free for academic use) applied to the
dual of \eqref{eq:ktf}.\footnote{We add tiny amount of regularization to the
  dual problem to avoid numerical issues that cause Gurobi to fail. The solution
  of this regularized problem is first used to reconstruct the primal solution,
  but then evaluated on the objective function of the original problem, when 
  computing the suboptimality gaps.}  
From the family of specialized ADMM algorithms, we pay particular attention to
ADMM Types I and II, which correspond to $j=1$ and $j=k$, respectively. We also
consider $j=0$, which we ADMM \emph{Type 0}, mainly because it is closely
related and should perform similarly to the Douglas-Rachford and proximal
Dykstra methods. In all ADMM algorithms, we adopt an adaptive choice of $\rho$
that balances the primal and dual suboptimality \citep{boyd2011distributed}.     

To ensure a fair comparison between ADMM Types I and II, which we will see are
generally the best performing methods (and thus the comparison between them is
of particular interest), we use optimized C++ implementations for each of their
prox subroutines; for Type I, this is the DP algorithm for univariate TV
denoising, and for Type II, this is the PDIP algorithm for univariate trend
filtering; and in both cases, we use C++ implementations from
\citet{ramdas2016fast}. Aside from specialized subroutines, the implementation
of all iterative algorithms is in MATLAB.  

The results are presented in Figures \ref{fig:optimization_exps} and
\ref{fig:optimization_admm_vs_gurobi}. Figure \ref{fig:optimization_exps}
compares the operator splitting algorithms for denoising the standard ``Lena''
method at a resolution of $256 \times 256$. The KTF orders are taken to be
$k=1,2,3$, corresponding to the columns in the figure. For each $k$, the
solution returned by Gurobi is used to define the optimal criterion value, which
is is then used to measure the suboptimality gap of solutions returned the
iterative methods. ADMM Type I is generally the winner in all cases, whether
measured by iteration or (especially) by wall-clock time.  

Figure \ref{fig:optimization_admm_vs_gurobi} compares ADMM Type I to Gurobi for 
varying $k$, and also for varying resolutions of the underlying Lena image. In
all cases, ADMM Type I obtains a moderate-quality solution in less one
second---which is sometimes two orders of magnitude faster than the
off-the-shelf solver provided by Gurobi. It should be further noted that Gurobi
is highly-optimized, whereas our ADMM Type I implementation is not---recall,
only the prox subroutine is optimized, and the outer looping is performed in
MATLAB. Transporting the entire algorithm to C++ would clearly yield further
improvements in efficiency. Of course, if a truly high-accuracy solution is
required, then Gurobi may be the best option. However, its strong performance in
this subsection suggests that ADMM Type I is an efficient and useful approach
for many applications in statistics and machine learning, where
moderate-accuracy solutions suffice.

\begin{figure}[p]
\centering
\includegraphics[width=0.325\textwidth]{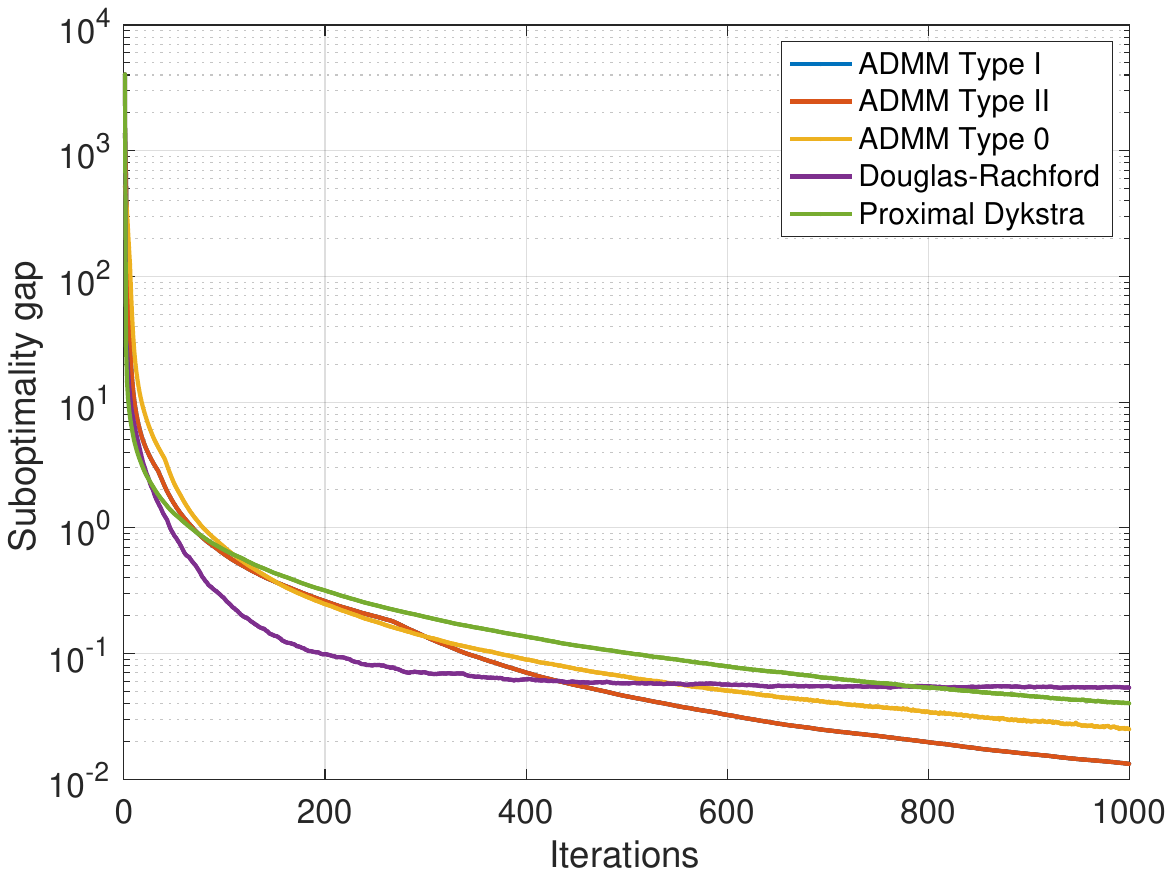}
\includegraphics[width=0.325\textwidth]{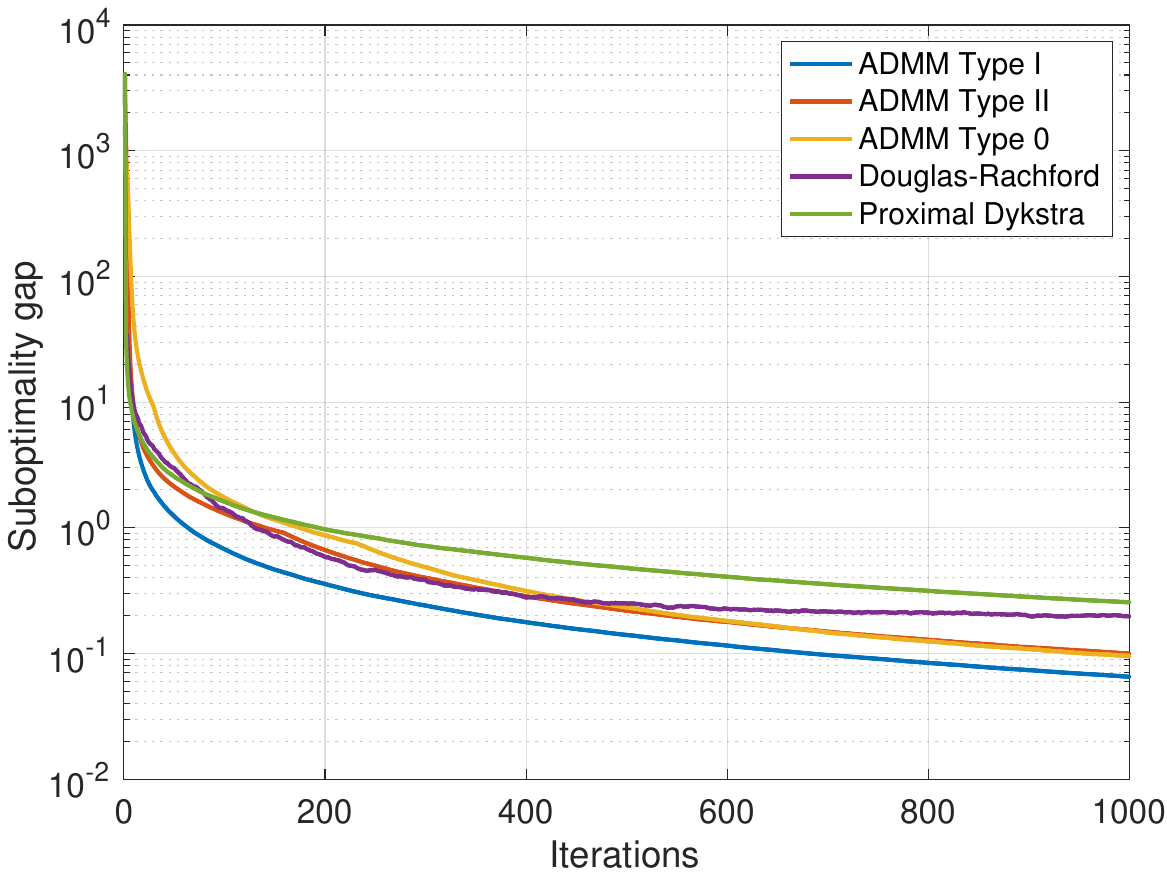}
\includegraphics[width=0.325\textwidth]{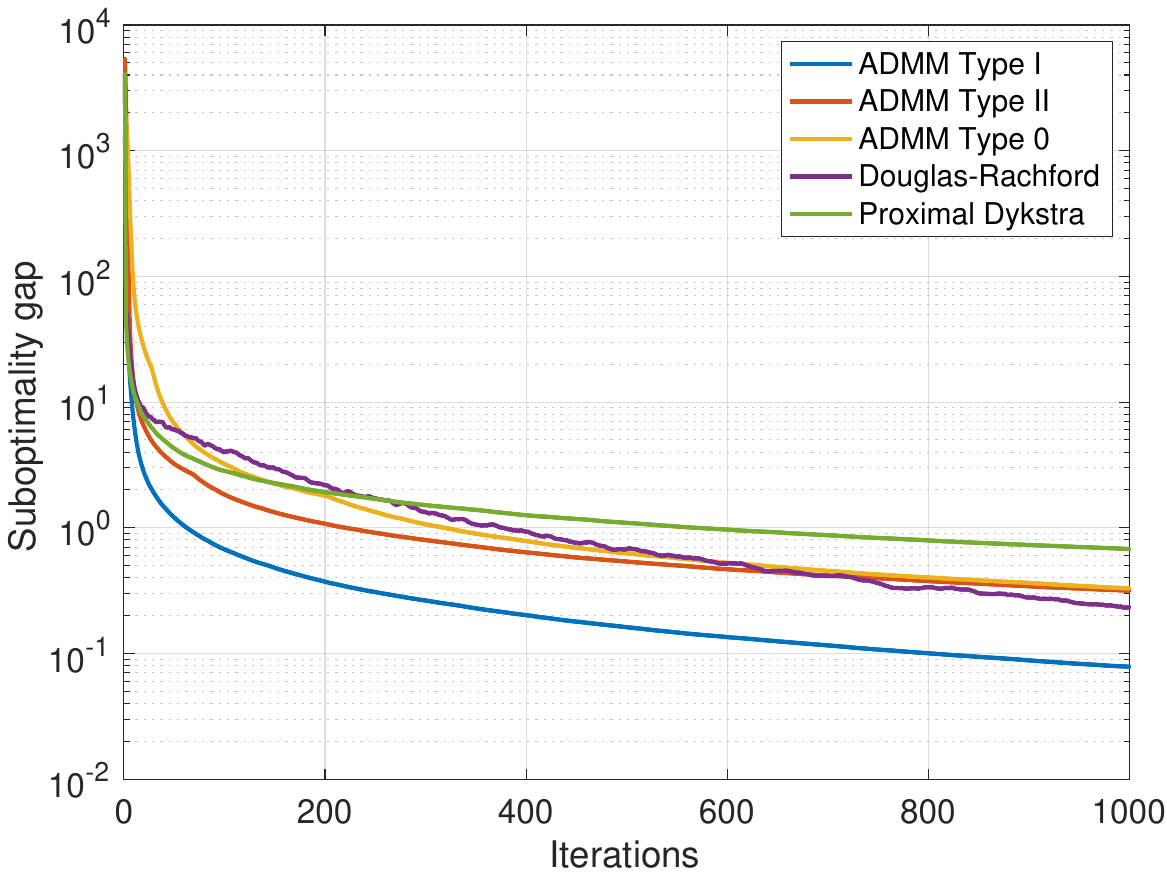} \\
\includegraphics[width=0.325\textwidth]{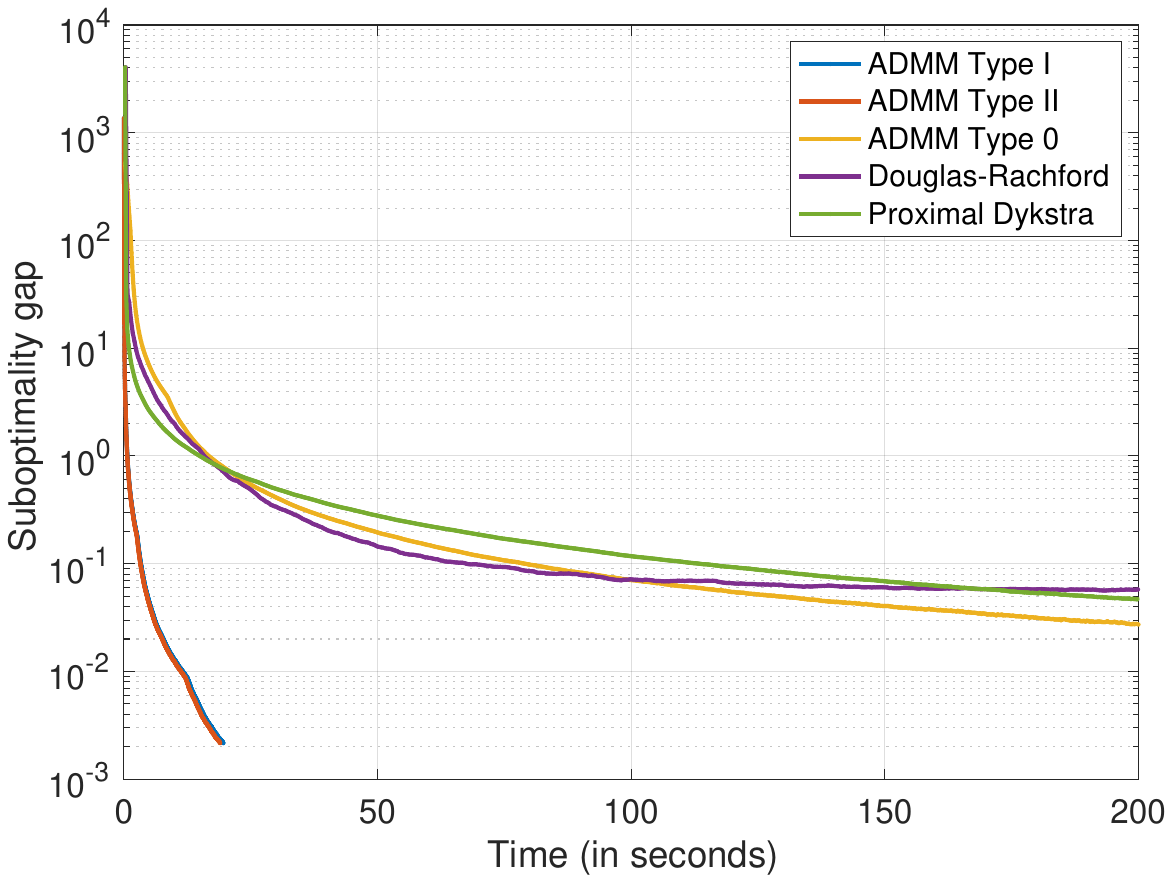}
\includegraphics[width=0.325\textwidth]{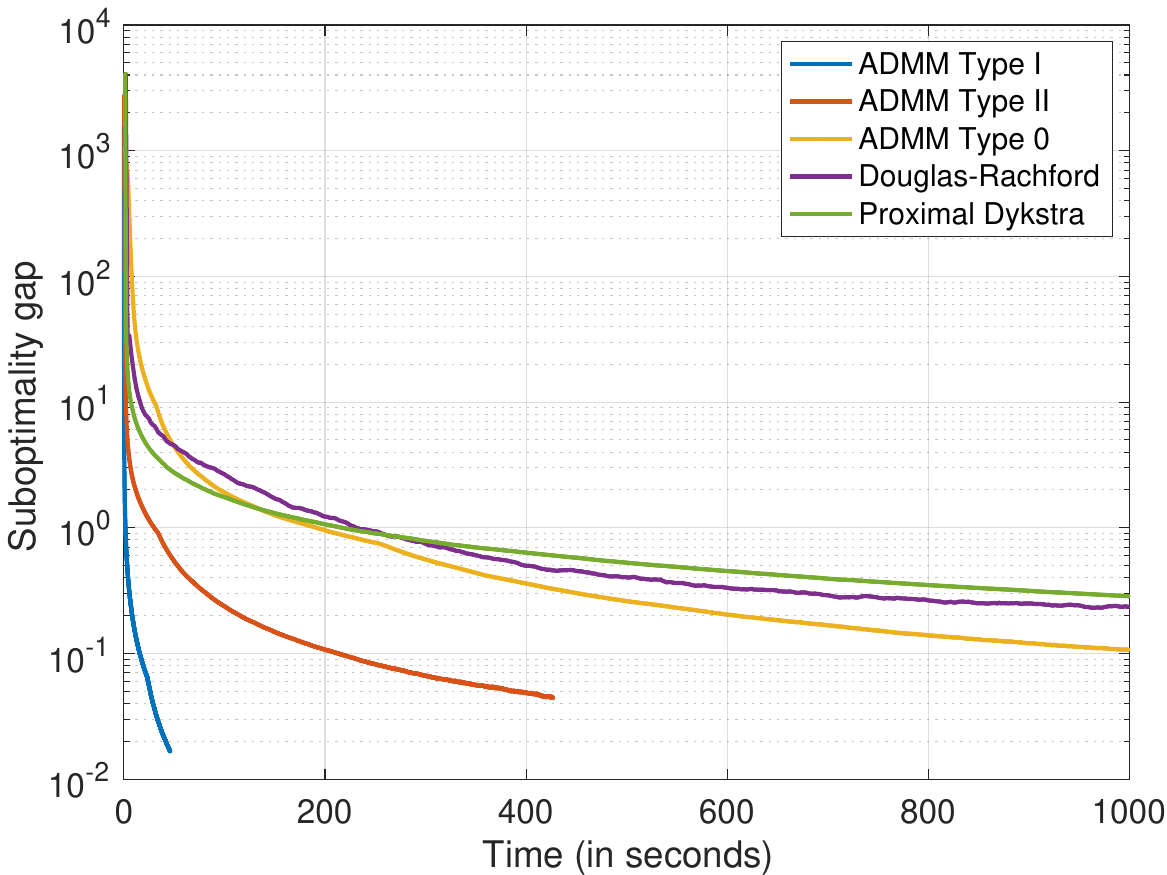}
\includegraphics[width=0.325\textwidth]{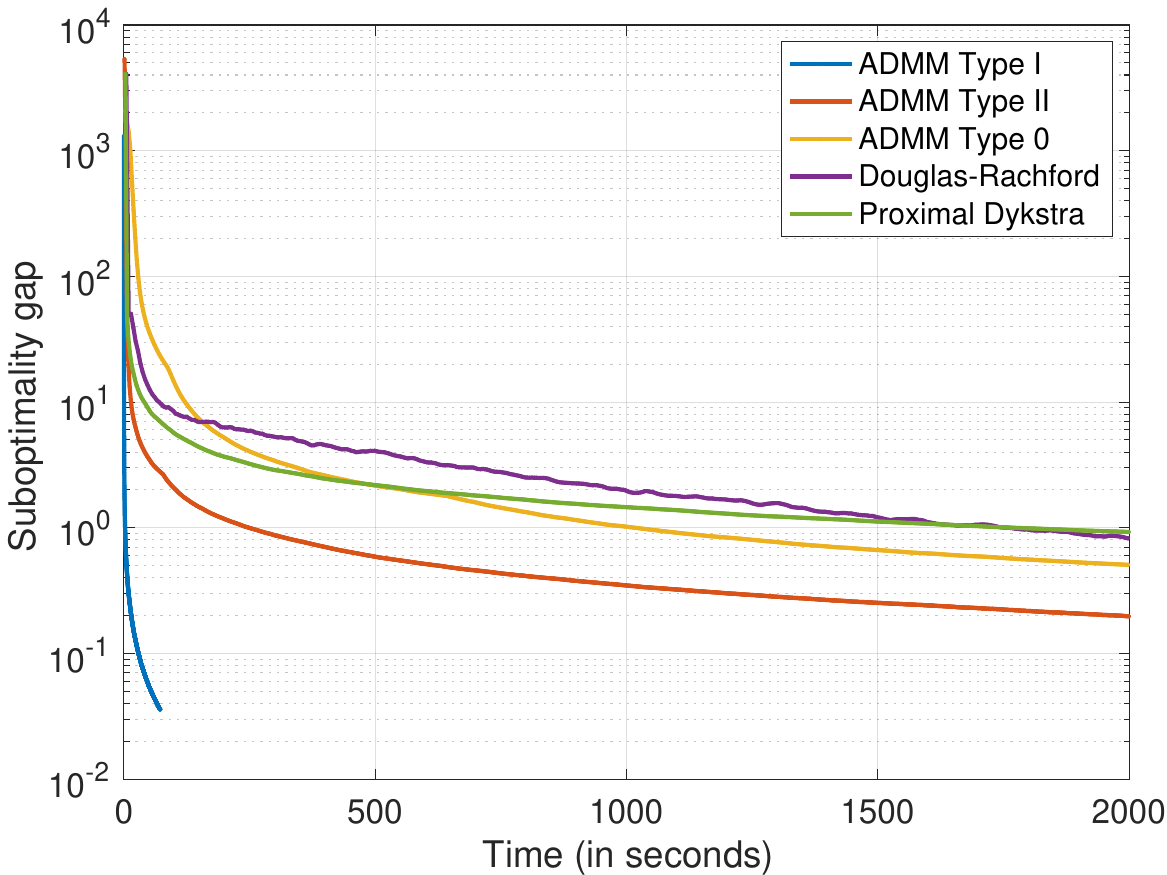} \\
\caption{\it\small Comparison of iterative algorithms for KTF on the standard
  Lena image of resolution $256 \times 256$ (that is, $n=65536$), when
  $k=1,2,3$, corresponding to the three columns, from left to right. The top row   
  compares the convergence of the suboptimality gap as a function of the number
  of iterations. The bottom row shows the same but parametrized by wall-clock
  time in seconds. While these methods have similar sublinear convergence rates
  (top row), ADMM Types I and II are clearly the fastest (bottom row) to reach a
  small suboptimality gap, due to their low per-iteration cost. Type I is the
  overall winner. (Recall that when $k=1$, Types I and II coincide, so the blue   
  curve is hidden behind the red curve).}  
\label{fig:optimization_exps}

\bigskip\bigskip

\centering
\includegraphics[width=0.325\textwidth]{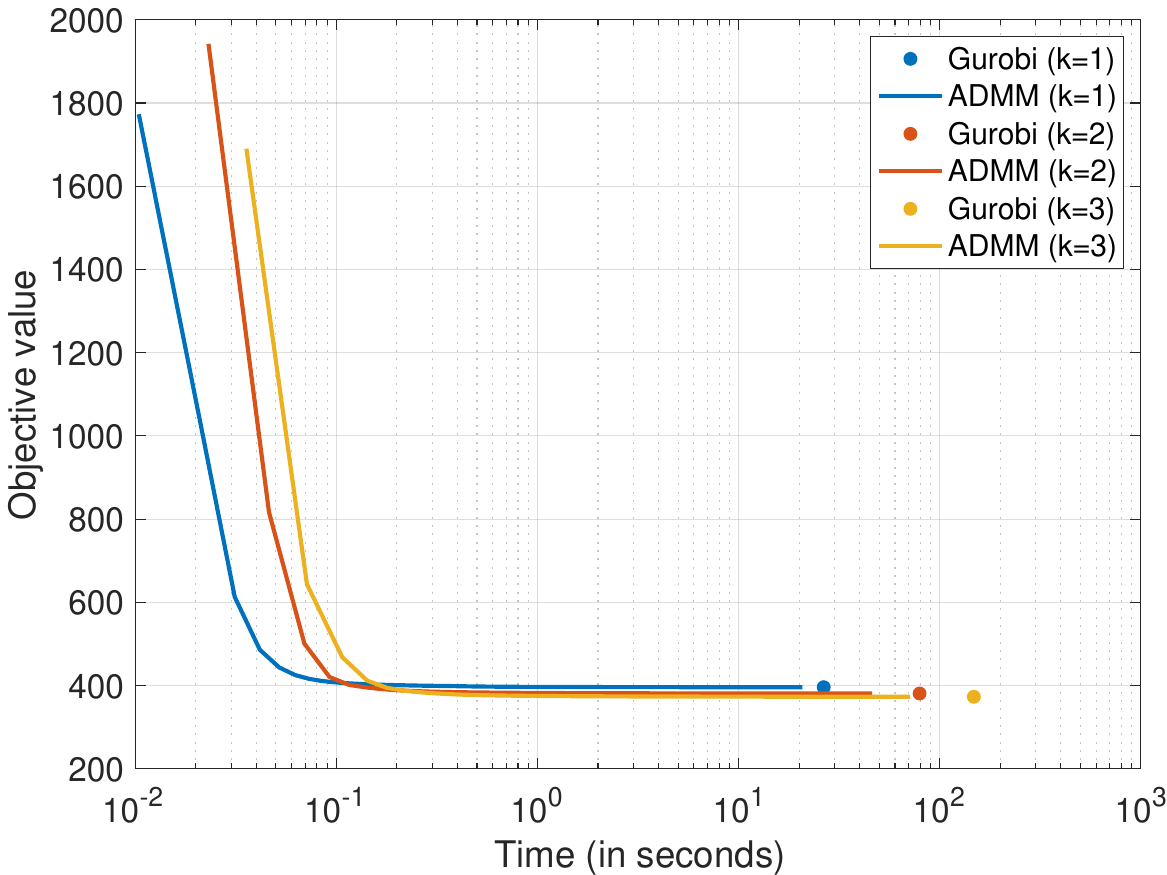}
\includegraphics[width=0.325\textwidth]{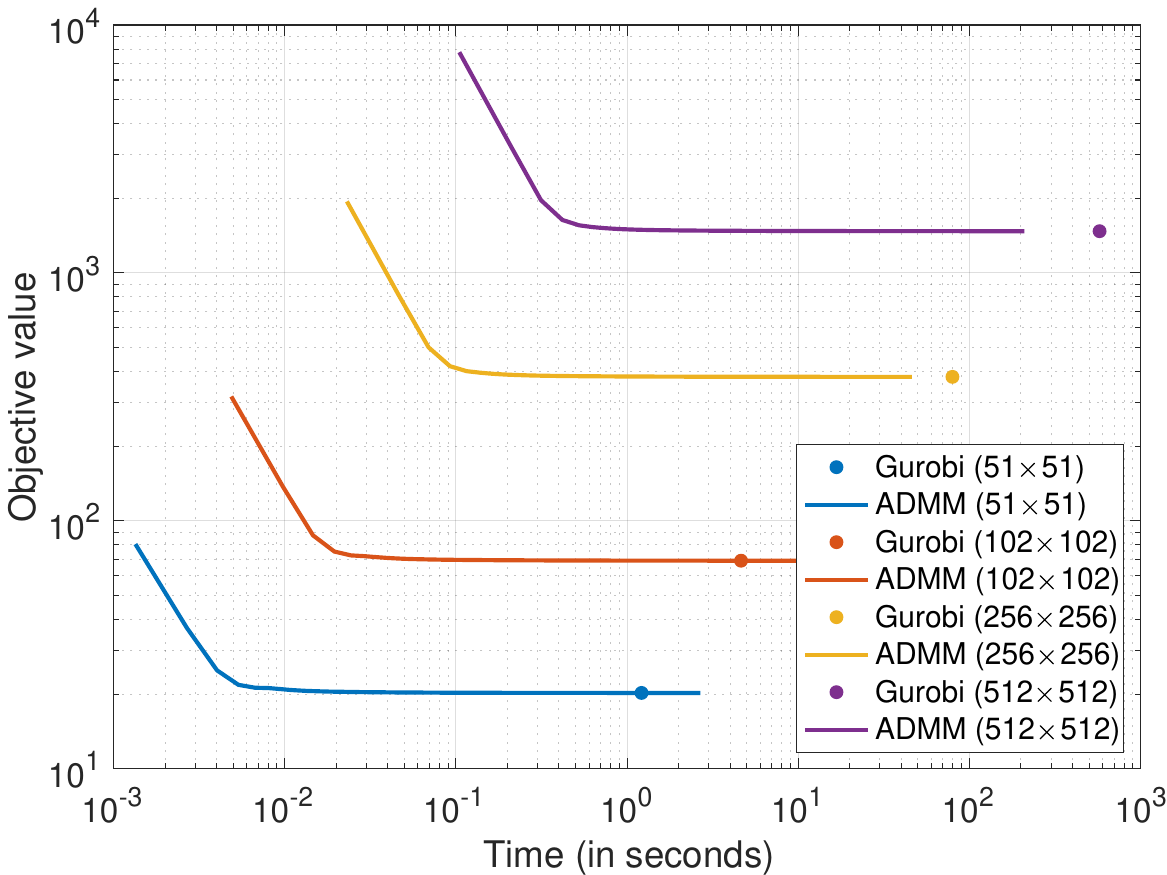}
\includegraphics[width=0.325\textwidth]{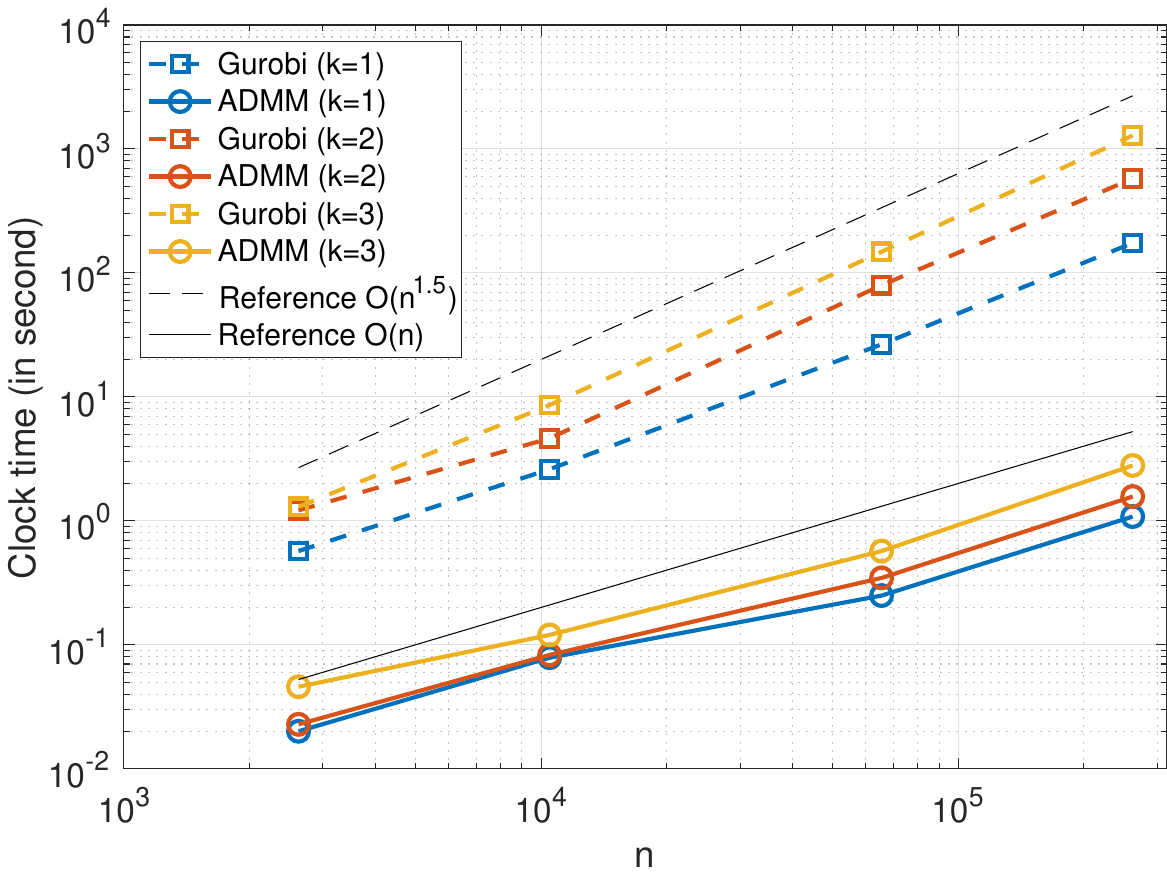}
\caption{\it\small Comparison of ADMM Type I to Gurobi for the same Lena
  problem. The left panel compares the two methods for varying $k$ at a fixed
  resolution of $256 \times 256$. The middle panel fixes $k=2$ and compares them
  across varying resolutions. The right panel plots the time needed to achieve a
  certain relative error, defined by the ratio of the suboptimality gap to the
  objective value being less than $10^{-2}$. Altogether, we see that ADMM Type I
  achieves a moderate-quality  solution several orders of magnitude faster than
  Gurobi. Furthermore, the right panel shows that Gurobi appears to scale as 
  $O(n^{1.5}$) (to be expected, if it is based on interior point methods
  internally) whereas ADMM Type I appears to scale closer to $O(n)$.}   
\label{fig:optimization_admm_vs_gurobi}  
\end{figure}

\section{Interpolation algorithm}
\label{sec:interpolation}
 
In this section, we derive an algorithm to extend the KTF solution in
\eqref{eq:ktf}, defined only at points in the lattice $Z_{n,d}$, to a function
defined on all of $[0,1]^d$. As discussed in Remark \ref{rem:ktf_basis}, this is
made possible by the continuous-time formulation for KTF in 
\eqref{eq:ktf_continuous} of Proposition \ref{prop:ktf_continuous}, which,
recall, relates to the original discrete-time problem \eqref{eq:ktf} in that at
their solutions \smash{$\hf,\htheta$}, respectively, we have \smash{$\hf(x_i) =
  \theta_i$}, for $i=1,\ldots,n$. Given the coefficients that define the
function \smash{$\hf \in \cH_{n,d}^k$} in its expansion the tensor product basis
of univariate falling factorial functions, that is, given the solution
\smash{$\halpha$} in \eqref{eq:ktf_basis}, we can form the interpolated
prediction \smash{$\hf(x)$} at an arbitrary point $x \in [0,1]^d$ by simply
evaluating this basis expansion at $x$, as shown in \eqref{eq:ktf_extend}.   

While this is conceptually the easiest way to interpolate the fitted values
\smash{$\hf(x_i)$}, $i=1,\ldots,n$ to form the prediction at an arbitrary $x \in
[0,1]^d$, it is not the most efficient. The expression in \eqref{eq:ktf_extend}
takes \smash{$O(\|\halpha\|_0 (k+1)^d)$} operations, where
\smash{$\|\halpha\|_0$} denotes the number of nonzero elements in
\smash{$\halpha$} (the number of active basis functions), because evaluating
each univariate basis function $h_{i_j}(x_j)$ takes $O(k+1)$ operations (it
being a product of $k+1$ terms, recall \eqref{eq:ffb}). In what follows, we 
present an interpolation algorithm that takes only \smash{$O((k+1)^{d+1})$}
operations, a big savings over \eqref{eq:ktf_extend} when
\smash{$\|\halpha\|_0$} is large. Moreover, our algorithm acts directly on the
solution \smash{$\htheta$} in \eqref{eq:ktf}, meaning that we never have to
solve \eqref{eq:ktf_basis} in the first place.

\subsection{Review: univariate interpolation}

Our interpolation algorithm for KTF in the multivariate case builds from the
univariate discrete spline interpolation algorithm derived in Corollary 2 of
\citet{tibshirani2020divided}. For completeness, we transcribe this in Algorithm
\ref{alg:interp_1d}. Here and henceforth, we use the abbreviation $x_{a:b} =
(x_a, \ldots, x_b)$ for integers $a \leq b$. Also, we use $f[z_1,\ldots,z_r]$
for the divided difference of a function $f$ at distinct points
$z_1,\ldots,z_r$. Recall, this is defined for $r=2$ by
$$
f[z_1,z_2] =  \frac{f(z_2)-f(z_1)}{z_2-z_1},
$$
and for any $r \geq 3$ by the recursion
$$
f[z_1,\ldots,z_r] = \frac{f[z_2,\ldots,z_r] - f[z_1,\ldots,z_{r-1}]}{z_r-z_1}.  
$$
Note that for evenly-spaced points, this simply coincides with a scaled
forward difference; in particular, recalling the notation introduced in Section
\ref{sec:tf}, we have 
$$
f[z, \ldots, z + (r-1)/n] = \frac{(r-1)!}{n^r} (\dop^r f)(z). 
$$
For more background on divided differences, discrete splines, their connection
to the falling factorial basis and to trend filtering, we refer to
\citet{tibshirani2020divided}.  

Algorithm \ref{alg:interp_1d} takes $O((k+1)^2)$ operations, as
\eqref{eq:ffb_interp1}, \eqref{eq:ffb_interp2} are each linear systems in just
one unknown, and forming the coefficients in either linear system can be done in
$O((k+1)^2)$ operations (this uses a representation of a divided difference as
an explicit linear combination of the underlying function evaluations; refer to,
for example, Section 2.1 of \citet{tibshirani2020divided}). Note that this
assumes $x_{1:n}$ are evenly-spaced design points, because in this case
identifying the smallest index $i$ such that $x_i>x$ can be done with integer
division. For general design points, the identification step takes $O(\log n)$
operations via binary search, so the total cost would be $O(\log n + (k+1)^2$). 
 
Moreover, Corollary 2 in \citet{tibshirani2020divided} establishes that the
value $f(x)$ returned by Algorithm \ref{alg:interp_1d} is equal to that produced
by the falling factorial basis representation in \eqref{eq:ffb_tensor} (for
$d=1$), where $\alpha$ is the unique coefficient vector such that $f(x_i) =
\theta_i$, $i=1,\ldots,n$.

\begin{algorithm}[tb]
\textbf{Input:} design points $x_{1:n}$ with entries in increasing order; values  
$\theta_{1:n}$ to interpolate; query point $x$; integer $k \geq 0$. \\  
\textbf{Output:} interpolated value $f(x)$, where $f$ is the unique $k\th$
order discrete spline with knots in \smash{$x_{(k+1):(n-1)}$}, such that $f(x_i) 
= \theta_i$, $i=1,\ldots,n$.  

\begin{enumerate}[topsep=0pt,itemsep=-1ex,partopsep=1ex,parsep=1ex]

\item If $x = x_i$ for some $i=1,\ldots,n$, then return $\theta_i$.

\item Else, if $x > x_{k+1}$ and $i$ is the smallest index such that $x_i>x$
  (with $i=n$ when $x>x_n$), then return $f(x)$ as the unique solution of the   
  linear system:
  \begin{equation}
    \label{eq:ffb_interp1}
    f[x_{i-k}, \ldots, x_i, x] = 0.
  \end{equation}

\item Else, if $x < x_{k+1}$, then return $f(x)$ as the unique solution of the
  linear system: 
  \begin{equation}
    \label{eq:ffb_interp2}
    f[x_1, \ldots, x_{k+1}, x] = 0.
  \end{equation}
\end{enumerate}
(Note that both \eqref{eq:ffb_interp1}, \eqref{eq:ffb_interp2} are linear
systems in just one unknown, $f(x)$, since we interpret $f(x_i) = \theta_i$, 
$i=1,\ldots,n$.) 
\caption{\textsc{Interpolate-1d}$(x_{1:n}, \theta_{1:n}, x, k)$}
\label{alg:interp_1d}
\end{algorithm}

\begin{algorithm}[tb]
\textbf{Input:} lattice {$\{z_{i1}\}_{i=1}^{N_1} \times \cdots \times 
  \{z_{id}\}_{i=1}^{N_d}$} where each set {$\{z_{ij}\}_{i=1}^{N_j}$} in
the Cartesian product is sorted in increasing order; values
\smash{$\{\theta_i\}_{i \in [N_1] \times \cdots \times [N_d]}$} over the
lattice to interpolate; query point $x$; integer $k \geq 0$. \\  
\textbf{Output:} interpolated value $f(x)$, for the unique function $f$ in the
tensor product space of $k\th$ degree discrete splines with knots in
\smash{$\{z_{i1}\}_{i=k+1}^{N_1-1} \times \cdots \times
  \{z_{id}\}_{i=k+1}^{N_d-1}$}, such that \smash{$f(z_{i_1,1},\ldots,z_{i_d,d}) 
  = \theta_{i_1,\ldots,i_d}$}, $(i_1,\ldots,i_d) \in [N_1] \times \cdots \times
[N_d]$.  

\begin{enumerate}[topsep=0pt,itemsep=-1ex,partopsep=1ex,parsep=1ex]
\item If $d=1$, then return \textsc{Interpolate-1d}$(z_{1:N_1}, \theta_{1:N_1},
  x, k)$.
\item Else, let $i_1$ denote the smallest index such that $x_{i_1,1} \geq
  z_{i_1,1}$. 
\item Let $\ell_1 = \min\{\max\{i_1-k, 1\}, N_1-k\}$. 
\item Let \smash{$\vartheta_p = \textsc{Interpolate}\big(
  \{z_{i2}\}_{i=1}^{N_2} \times \cdots \times \{z_{id}\}_{i=1}^{N_d}, \, 
  \{\theta_i\}_{i \in \{i_1+p-1\} \times [N_2] \times \cdots \times [N_d]}, \, 
  x_{2:d}, \, k \big)$}, for $p \in [k+1]$.
\item Return \smash{$\textsc{Interpolate-1d}(
  z_{\ell_1:(\ell_1+k), 1}, \vartheta_{1:(k+1)}, x_1, k)$}.
\end{enumerate}
\caption{\textsc{Interpolate}$\big(
  \{z_{i1}\}_{i=1}^{N_1} \times \cdots \times \{z_{id}\}_{i=1}^{N_d}, \,
  \{\theta_i\}_{i \in [N_1] \times \cdots \times [N_d]}, \,  x, \, k \big)$}
\label{alg:interp}
\end{algorithm}

\subsection{Multivariate interpolation}

In the multivariate case, it turns out that we can interpolate within the space
of tensor products of $k\th$ degree discrete splines in $O((k+1)^{d+1})$ time, 
assuming a uniformly-spaced lattice. The idea is to reduce the $d$-dimensional
problem calculation to $k+1$ interpolation problems, each one in dimension
$d-1$. Figure \ref{fig:interp_intuition} gives the intuition for $d=2$. The 
algorithm in described in Algorithm \ref{alg:interp}, where we recall the
notation from Section \ref{sec:theory}, abbreviating $[a] = \{1,\ldots,a\}$ for
an integer $a \geq 1$. 

Algorithm \ref{alg:interp} assumes each side length of the lattice is at least
$k+1$. Its proof of correctness, as well as the running time of $O((k+1)^{d+1})$
for a uniformly-spaced lattice, follows from a straightforward inductive
argument over $d$, whose proof we omit. We note that for an arbitrary lattice,
the running time is \smash{$O(\sum_{j=1}^d \log N_j + (k+1)^{d+1})$}. Figure
\ref{fig:interp_example} gives some examples of interpolation for $d=2$.

\begin{figure}[p]
\begin{minipage}[c]{0.375\textwidth}
\centering
\includegraphics[width=\textwidth]{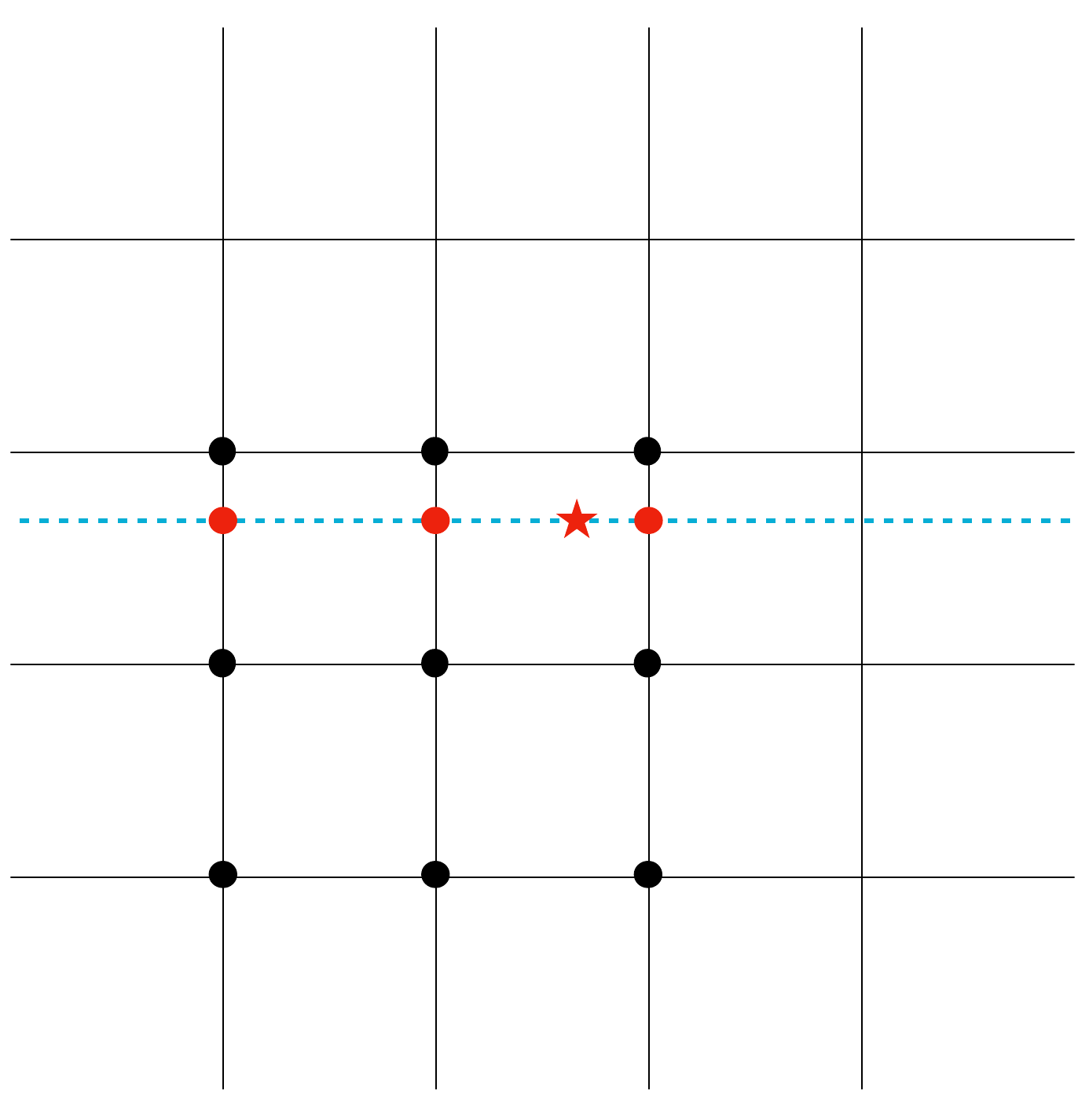}
\end{minipage}
\hspace{15pt}
\begin{minipage}[c]{0.575\textwidth}
\caption{\it\small Illustration of multivariate interpolation from Algorithm
  \ref{alg:interp}, when $d=2$ and $k=2$. The value to be interpolated is marked
  by a red star. The algorithm first interpolates the $k+1=3$ values indicated
  by red dots, each time using univariate interpolation along the y-axis, from
  Algorithm \ref{alg:interp_1d}. As the final step, these red dots are used to
  interpolate the value at the red star, using univariate interpolation along
  the x-axis, again from Algorithm \ref{alg:interp_1d}.}
\label{fig:interp_intuition}
\end{minipage}

\bigskip\bigskip

\centering
\includegraphics[width=0.425\textwidth]{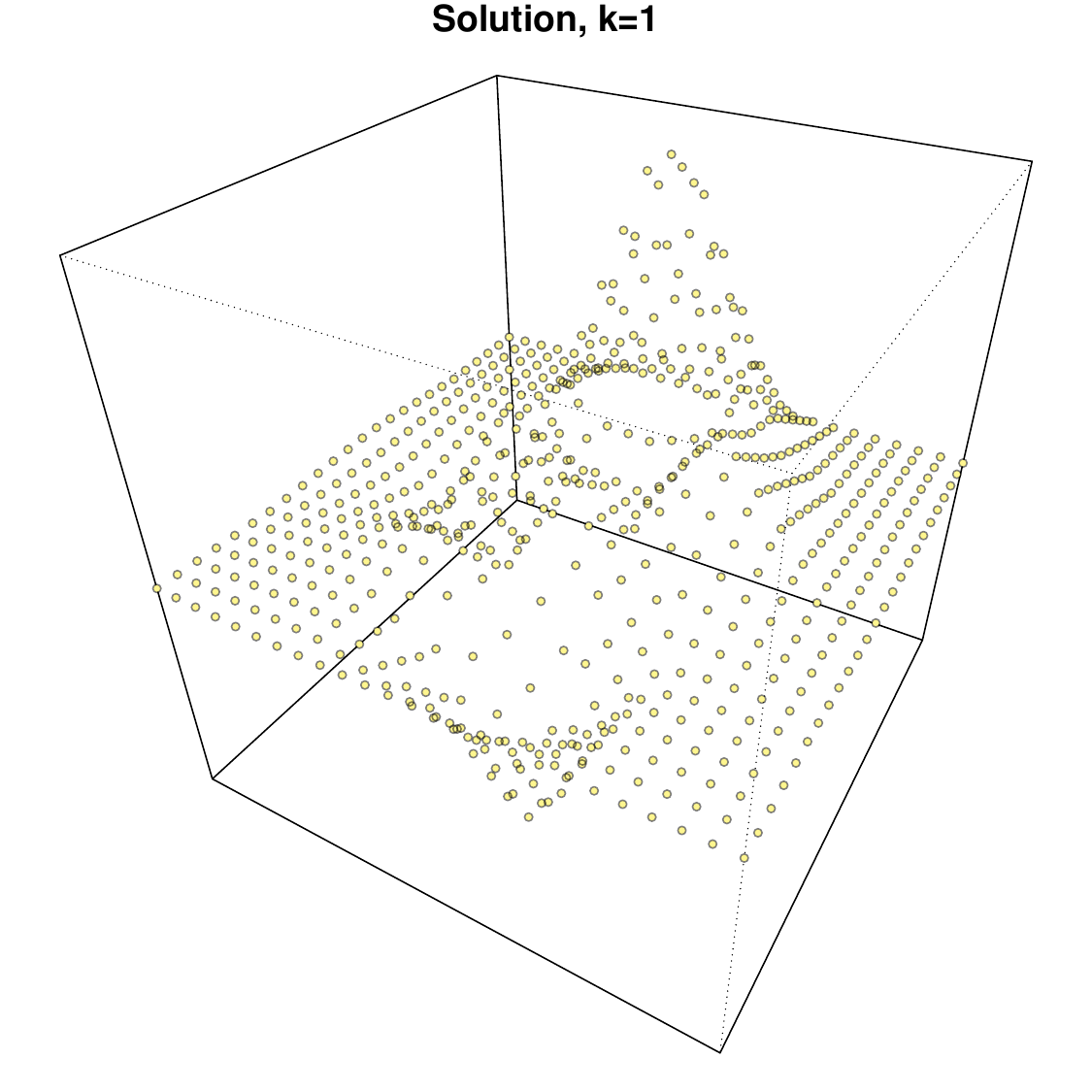}
\includegraphics[width=0.425\textwidth]{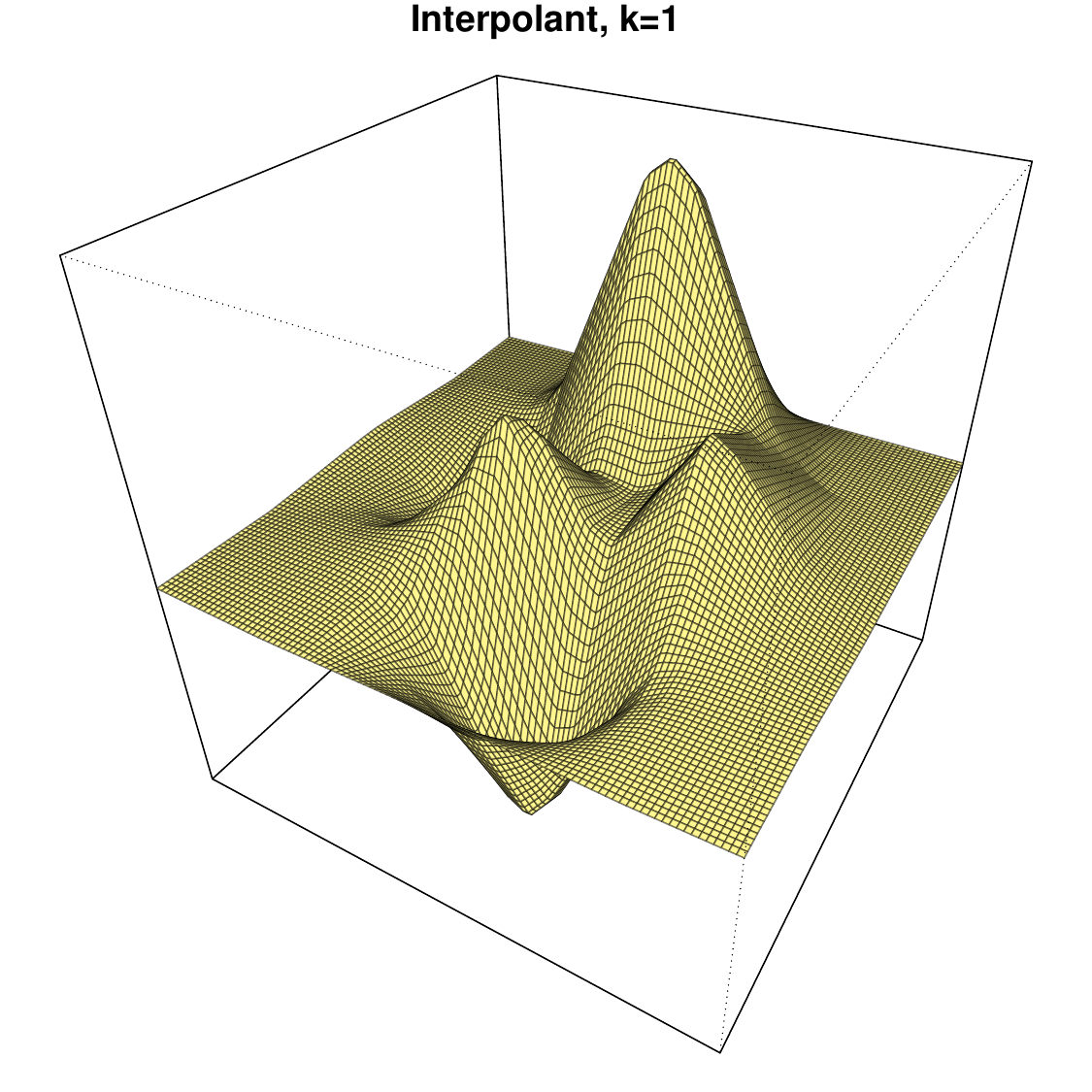} \\
\includegraphics[width=0.425\textwidth]{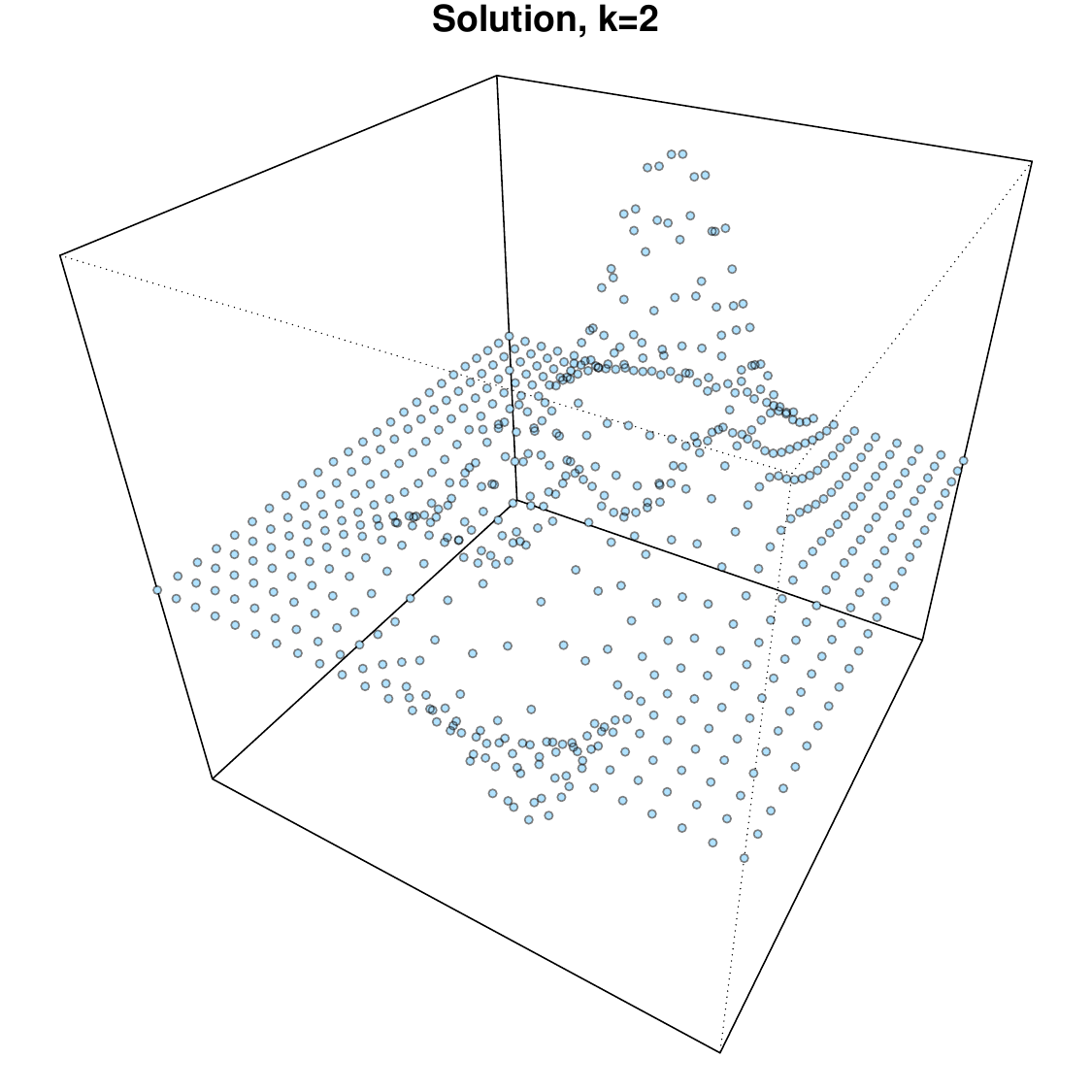}
\includegraphics[width=0.425\textwidth]{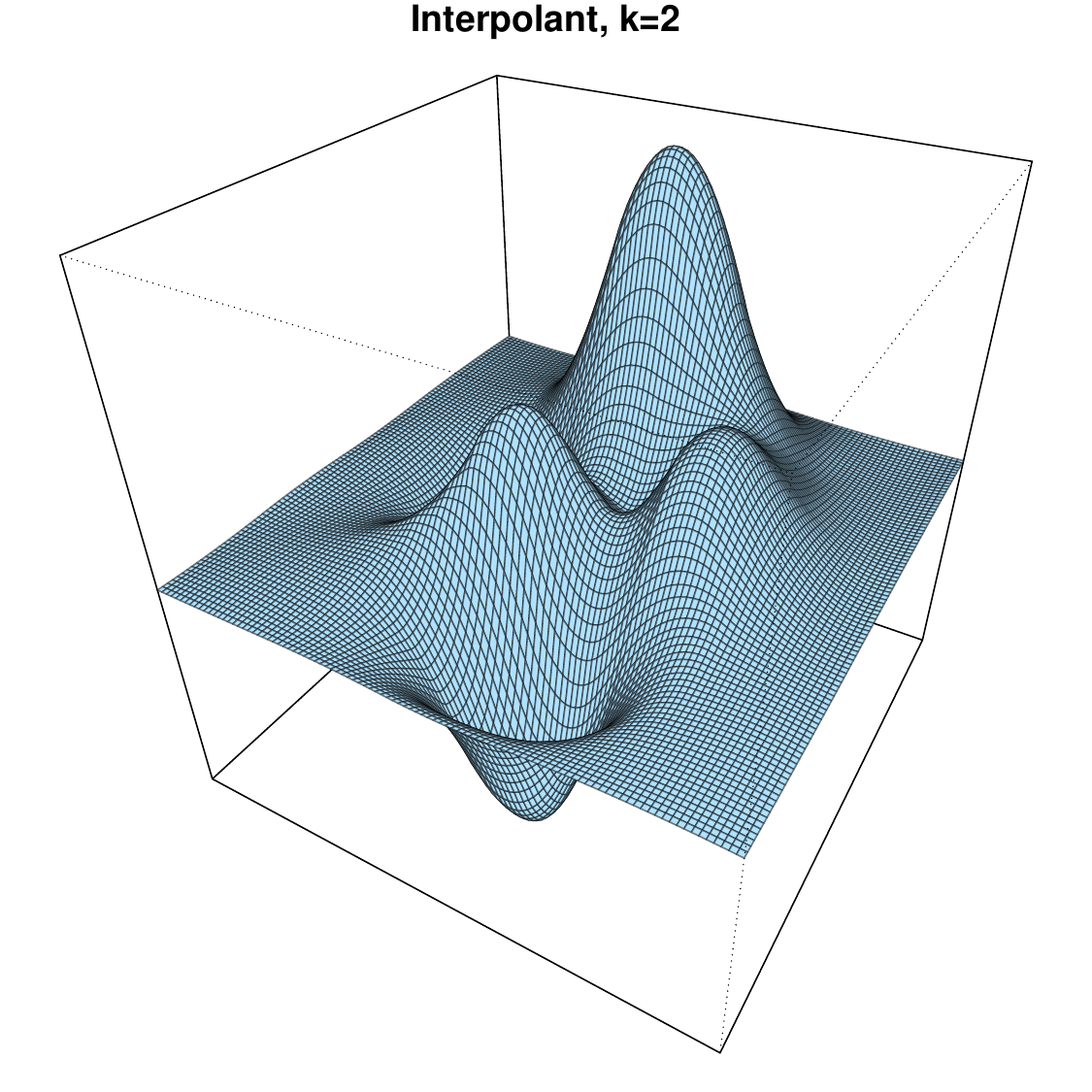}
\caption{\it\small Top left: KTF solution, with $k=1$, for a problem over a
  square lattice with side length $N=25$ (that is, $n=25^2=625$ points). Top
  right: the interpolated surface from Algorithm \ref{alg:interp} (itself
  evaluated over a grid with 4x the resolution along each dimension; that is,
  $N=100$). Bottom row: analogous but for $k=2$. In either case the interpolated 
  function is in the tensor space of $k\th$ order discrete splines.}
	\label{fig:interp_example}
\end{figure}
\section{Experiments}
\label{sec:experiments}

We explore the empirical properties of KTF through a series of experiments that
compare its performance against other nonparametric methods.  

\subsection{Comparison of methods}%
\label{sec:error_analysis}

We study the performance of KTF against other nonparametric estimators by
analyzing their empirical mean squared error (MSE) in estimating the function
$f_0$ defined in Figure \ref{fig:intro}, over a square lattice with
$n=256^2=65536$ points. In the experiments that follow, we generate data by
adding Gaussian noise to evaluations of $f_0$, of the same magnitude illustrated
in Figure \ref{fig:intro}, yielding a signal-to-noise ratio (SNR) of 0.5. Aside
from KTF and graph trend filtering (GTF) of orders $k=0, 1, 2$ (for $k=0$, they
both reduce to TV denoising), we also consider second-order Laplacian smoothing
(using the squared Laplacian $L^2$, where $L$ is the Laplacian matrix of the 2d
grid graph), the eigenmaps estimator in \eqref{eq:eigenmaps} with $k=1$, kernel
smoothing (using a Gaussian kernel), and wavelet smoothing (using Daubechies'
least asymmetric wavelets, ten vanishing moments, and hard thresholding). For
the latter two estimators, we use the implementations from the R packages
\texttt{np} and \texttt{wavethresh}, respectively.


\begin{figure}[tb]
\centering
\includegraphics[width=\textwidth]{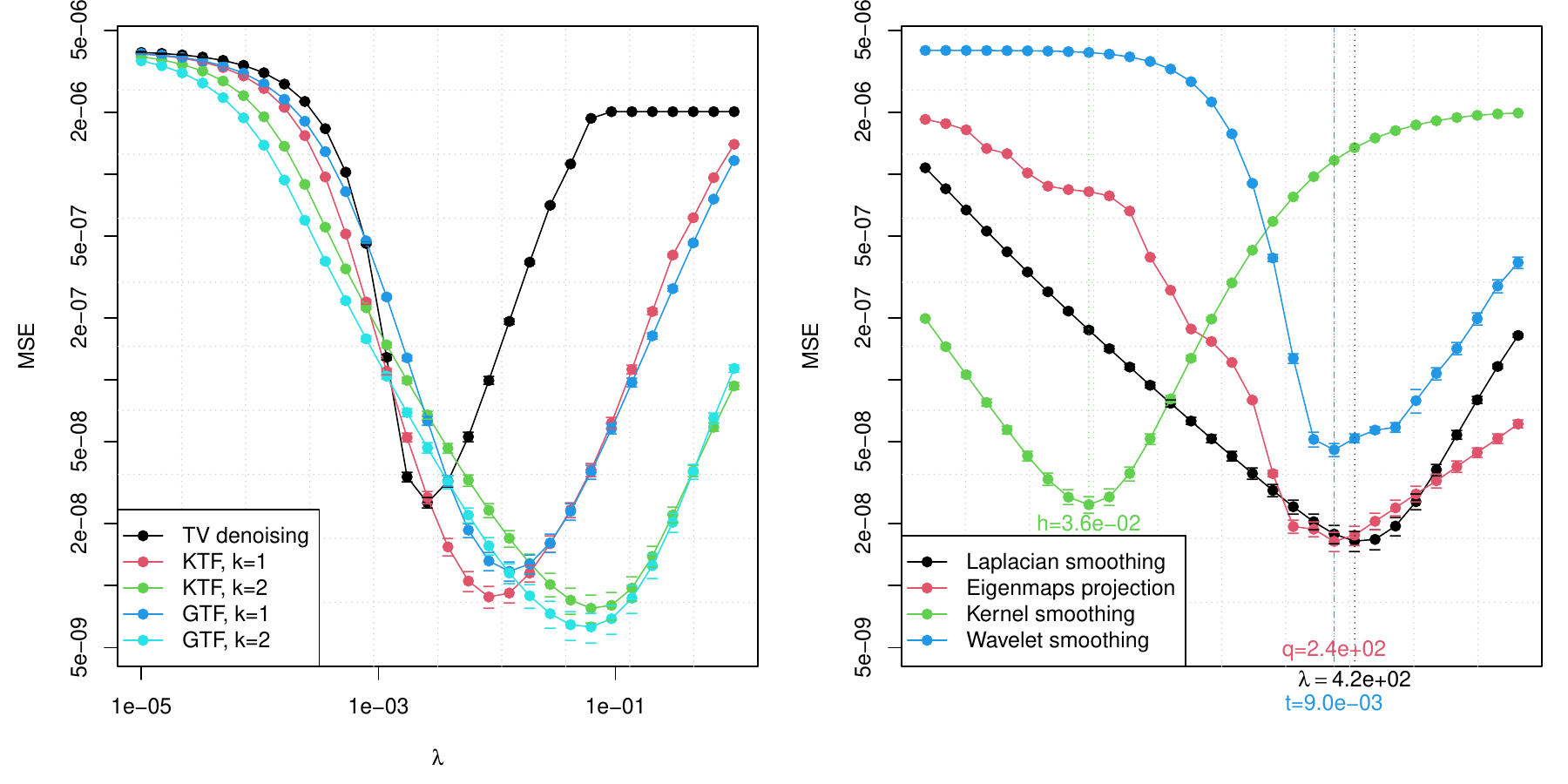}
\caption{\it\small MSE of various methods for estimating the signal function
  $f_0$ in Figure \ref{fig:intro}, over a square lattice with $n=256^2=65536$
  points, and subject to Gaussian noise that yields an SNR of 0.5. Left panel:
  the average performance (over 20 repetitions) of trend filtering estimators
  over a common range of tuning parameters $\lambda$. Right panel: the average
  performance of other nonparametric methods. As their tuning parameters lie
  on different scales, their values are scaled to fit onto a single x-axis. The
  takeaway is that KTF and GTF, in particular for $k=2$, achieve clearly the
  best performance.}  
\label{fig:error_analysis}
\end{figure}

Each estimator is fit over a range of tuning parameters, and its MSE as a
function of the tuning parameter is shown in Figure \ref{fig:error_analysis}.
The MSE curves here are averaged over 20 repetitions (each repetition forms a
data set by adding noise to $f_0$). Error bars are shown as well, denoting the
standard deviations of the MSE over these 20 repetitions. We can see that the
trend filtering estimators (with the exception of TV denoising, which returns a
piecewise constant fit that is not well-suited to an underlying signal with this
level of smoothness) perform quite a bit better than all competitors, and
factoring in the variability over repetitions, KTF and GTF achieve essentially
equivalent performance (for common values of $k$). The superiority of KTF over
the linear estimators (Laplacian smoothing, eigenmaps projection, and kernel
smoothing), should not come as a surprise---our theory prescribes that KTF
should perform better in a minimax sense than any linear smoother when the
underlying signal displays heterogeneous smoothness. These results thus serve as
a quantitative complement to the qualitative findings in Figure \ref{fig:intro},
and to the theoretical findings in Section \ref{sec:theory}.

\subsection{Rates for heterogeneous signals}%
\label{sec:heterogeneous_smoothness}

We now examine the empirical error rates of KTF and a linear smoother in
estimating signals belonging to KTV classes, for $k=0,1$. In both
cases, we choose $k,d$ so that the effective degree of smoothness remains
$s=1/2$, and we use the canonical scaling in \eqref{eq:ktv_canonical}, so that
the two classes under consideration are:  
\begin{equation}
\label{eq:ktv_classes_bdry}
\cT_{n,2}^0(\sqrt{n}) \quad\text{and}\quad \cT_{n,4}^1(\sqrt{n}).
\end{equation}
As representatives for ``hard'' signals within these two classes, we take the
true mean $\theta_0$ to be a ``one-hot'' signal when $k=0$, and a linear
``spike'' signal when $k=1$ (each scaled to the appropriate magnitude). For each
sample size $n$, we compute the MSE of KTF, and the eigenmaps projection
estimator in \eqref{eq:eigenmaps} with $k=0,1$, averaged over 20 repetitions, 
where each method is tuned to have the optimal average MSE over a range of
tuning parameter values. 

Figure \ref{fig:heterogeneous_smoothness} shows these average MSE curves as
functions of $n$, where the error bars again denote standard deviations. We can
see that in both cases, the KTF error decays faster than the minimax rate
(suggesting that the particular signals under consideration do not achieve the
worst-case rate for KTF), while the linear method exhibits a perfectly flat
error curve, suggesting that it fails to be consistent entirely.

\begin{figure}[p]
\centering
\includegraphics[width=\textwidth]{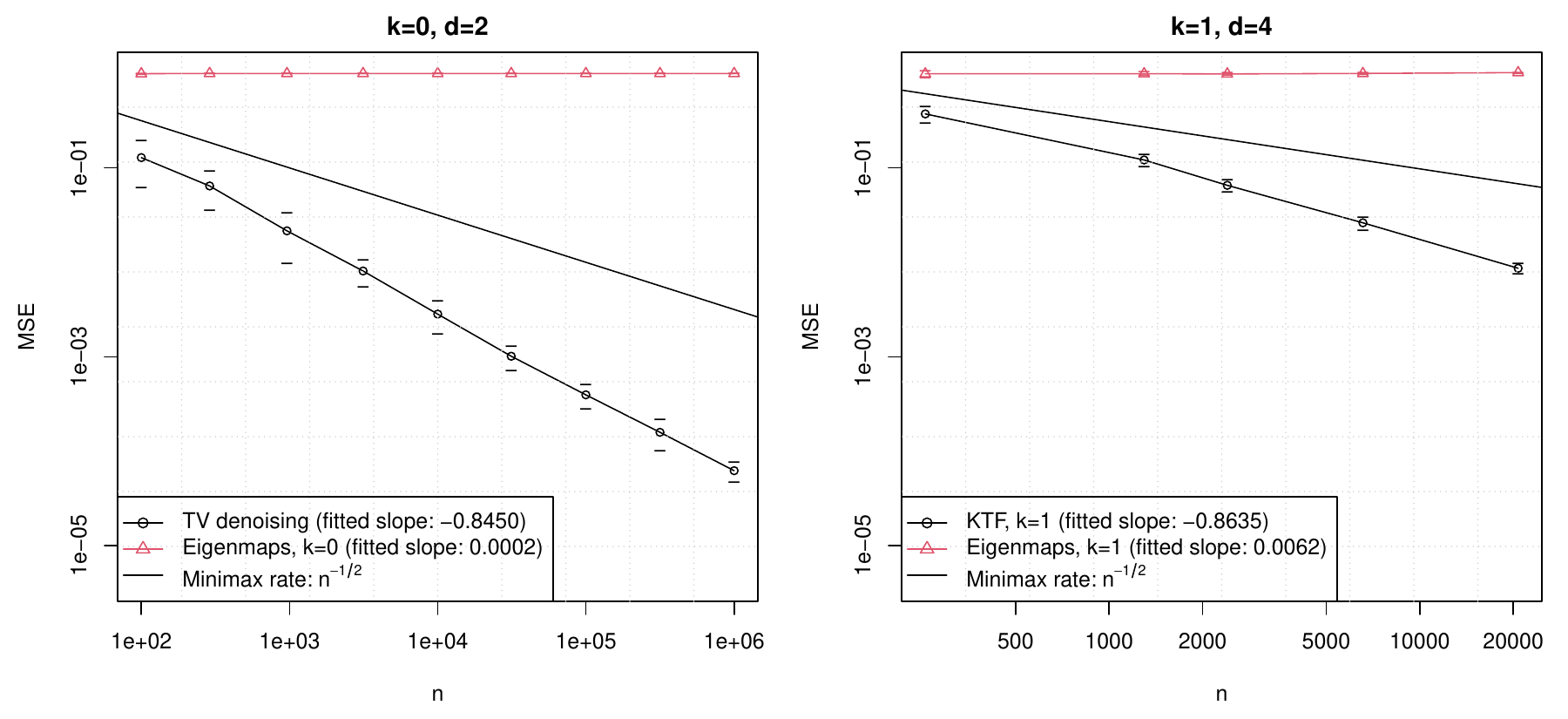}
\caption{\it\small Empirical error rates of KTF versus the eigenmaps projection 
  estimator, for two cases: $k=0, \, d=2$, and $k=1, \, d=4$. The true mean in
  each case was chosen to be representative of a ``hard'' signal in the
  appropriate KTV class in \eqref{eq:ktv_classes_bdry}. KTF converges faster
  than the minimax rate, whereas the eigenmaps estimator fails to be consistent
  entirely.} 
\label{fig:heterogeneous_smoothness}

\bigskip\bigskip

\centering
\includegraphics[width=\textwidth]{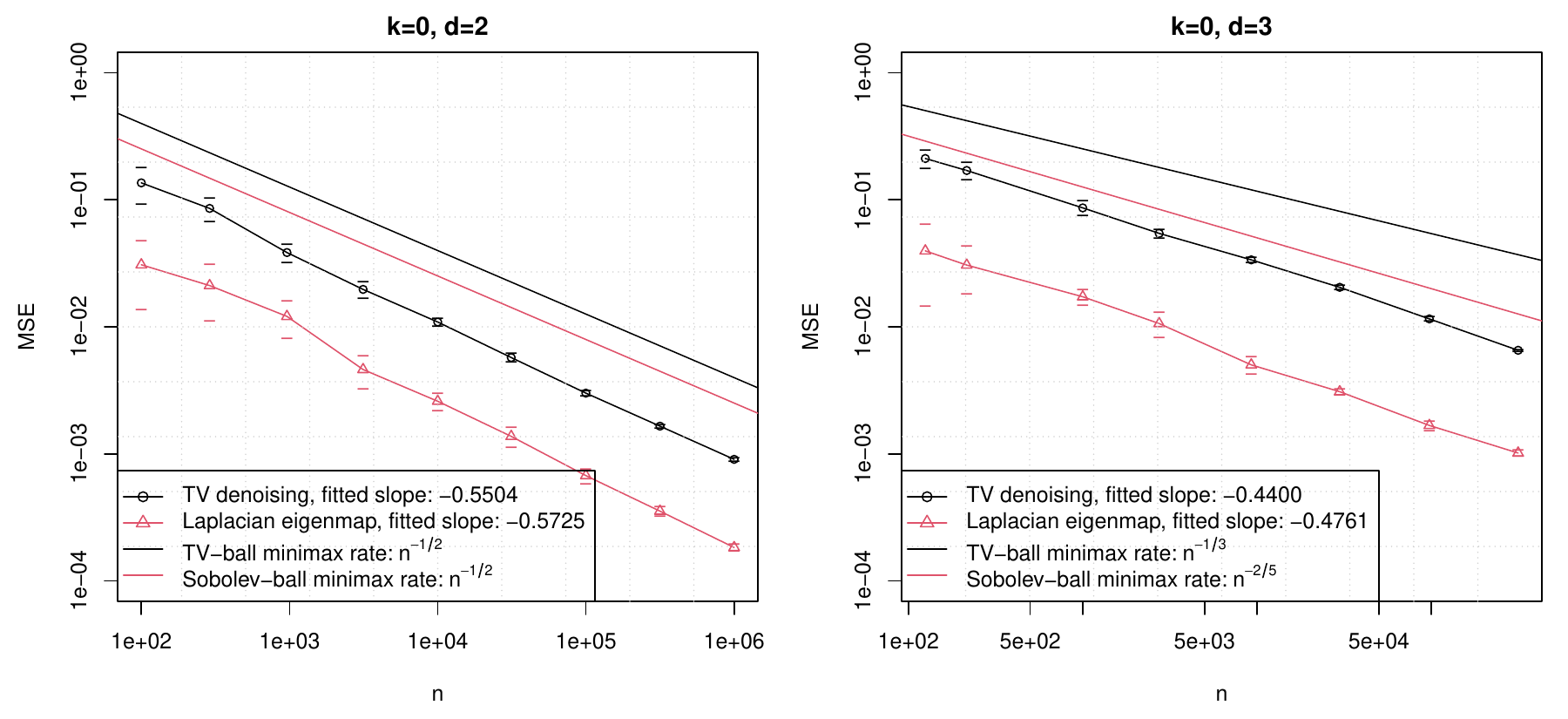}
\caption{\it\small Empirical error rates of KTF versus the eigenmaps projection
  estimator, for two cases: $k=0, \, d=2$, and $k=0, \,d=3$. The true mean was
  chosen to be a linear function in either case. KTF does not appear to be
  outperformed (in rate) by the eigenmaps estimator, leaving its adaptivity to
  smooth signals when $s < 1/2$ an open question.}   
\label{fig:linear_scaling}
\end{figure}

\subsection{Rates for homogeneous signals}%
\label{sec:homogeneous_smoothness}

Lastly, we examine the empirical error rates of KTF and the eigenmaps projector
for estimating homogeneously smooth signals. Recall we established that the
latter is minimax rate optimal over Sobolev (and Holder) classes in Theorem 
\ref{thm:sobolev_minimax}, whereas for $s < 1/2$, it is not clear (from its
minimax rate over the inscribing KTV class) that KTF achieves the faster rate
for Sobolev (or Holder) signals. We fix $k=0$ and consider two values for the
dimension, $d=2$ and $d=3$, which correspond to $s = 1/2$ and $s < 1/2$,
respectively. The true mean $\theta_0$ was taken to be the evaluations of a
linear function over the lattice, scaled appropriately (according to the
canonical scaling). Figure \ref{fig:linear_scaling} reports the MSE curves from
each method, under the same general setup as in the last subsection (averaged
over 20 repetitions, optimal tuning per $n$). Moving from $s = 1/2$ to $s <
1/2$, we do not find that the empirical performance of TV denoising becomes
markedly worse than its linear competitor, and the adaptivity of KTF to the
smooth signals remains an open question. 

\section{Discussion}
\label{sec:discussion}

We proposed and studied a method, Kronecker trend filtering, that 
extends univariate trend filtering to the multivariate setting, where the design
points lie on a uniformly-spaced lattice. We derived a continuous-time
representation for the optimization problem defining the KTF estimator, which
led to a method for rapid interpolation to off-lattice locations. We also
established a comprehensive set of minimax results which show that KTF is rate
optimal over heterogeneous classes of signals defined in terms of higher-order
TV regularity, whereas linear estimators fail to be optimal (and even
consistent, in some cases) over such classes. Lastly, we presented fast and
specialized ADMM methods for computing KTF solutions, which were shown to
perform favorably against numerous alternatives. 

There are numerous possible directions for future work. We finish by describing
five such directions. 

\subsection{General lattice structures}
\label{sec:general_lattice}

All along we have assumed a uniformly-spaced lattice with equal side lengths,
but much of this paper carries over to a more general lattice structure, namely
a Cartesian product \smash{$\{z_{i1}\}_{i=1}^{N_1} \times \{z_{i2}\}_{i=1}^{N_2}
  \times \cdots \times \{z_{id}\}_{i=1}^{N_d}$}, with \smash{$n = \prod_{j=1}^d
  N_j$}. The KTF penalty operator over this general lattice becomes:
\begin{equation}
\label{eq:ktf_general_lattice_pen_mat}
\dmatk_{z_1, \ldots, z_d} = \left[\begin{array}{c}
\dmatk_{z_1} \otimes I_{N_2} \otimes \cdots \otimes I_{N_d} \\     
I_{N_1} \otimes \dmatk_{z_2} \otimes \cdots \otimes I_{N_d} \\  
\vdots \\
I_{N_1} \otimes I_{N_2} \otimes \cdots \otimes \dmatk_{z_d}  
\end{array}\right],
\end{equation}
where for each $j=1,\ldots,d$, we abbreviate \smash{$z_j =
  \{z_{ij}\}_{i=1}^{N_j}$}, and denote by \smash{$\dmatk_{z_j}$} the $k\th$
order univariate trend filtering penalty matrix defined with respect to $z_j$, 
which (assuming the points in $z_j$ are sorted in increasing order) is given for
$k=1$ by \smash{$\dmat^{(1)}_{z_j} = \dmat^{(1)}_{N_j}$}, and for all other $k$ 
by:  
$$
\dmatk_{z_j} = \dmat^{(1)}_{N_j-k} \cdot 
\diag\bigg( \frac{k}{z_{k+1,j} - z_{1j}}, \frac{k}{z_{k+2,j} - z_{2j}}, \ldots,
\frac{k}{z_{N_j,j} - z_{N_j-k,j}} \bigg) \cdot \dmat^{(k)}_{z_j}, \quad
k=1,2,3,\ldots.
$$
Many of the properties and results derived in this paper carry over more or less   
immediately to the KTF estimator on a general lattice defined using the penalty  
matrix \eqref{eq:ktf_general_lattice_pen_mat}. All properties in Section
\ref{sec:properties} carry over with suitable adjustments, as do the specialized
ADMM algorithms in Section \ref{sec:optimization}, and the fast interpolation 
routine in Section \ref{sec:interpolation} (Algorithm \ref{alg:interp} is in
fact already written assuming a general lattice structure). Extending the theory
in Section \ref{sec:theory} is certainly less trivial, though we expect that this
should be possible under mild assumptions on the spacings between points in
$z_j$, $j=1,\ldots,d$. 


\subsection{Generalized linear models}

The KTF problem may be extended by replacing the squared loss in \eqref{eq:ktf}
with a generalized linear model (GLM) loss:
\begin{equation}
\label{eq:ktf_glm}
\minimize_{\theta \in \R^n} \; -\sum_{i=1}^n T(y_i) \theta_i + \psi(\theta) +
\lambda \|\kronmatk \theta\|_1,
\end{equation}
where $\psi : \R^n \to \R$ is convex and twice differentiable. The loss in
\eqref{eq:ktf_glm} corresponds to the negative log-likelihood of an exponential
family model for $y_i$, $i=1,\ldots,n$, with log-partition function $\psi$,
natural sufficient statistic $T$, and natural parameters $\theta_i$,
$i=1,\ldots,n$. The squared loss in \eqref{eq:ktf} corresponds to the working
model where each $y_i$ is Gaussian with mean $\theta_i$ and a common
variance. Under a working model where each $y_i$ is Poisson with mean
$\theta_i$, the GLM problem \eqref{eq:ktf_glm} becomes:
\begin{equation}
\label{eq:ktf_poisson}
\minimize_{\theta \in \R^n}\; \sum_{i=1}^n \big( -y_i \theta_i + \exp(\theta_i)
\big) + \lambda\|\kronmatk\theta\|_1.  
\end{equation}
Not far afield from this is an approach for density estimation, in which we take
data supported on a bounded domain in $\R^d$, discretize the domain using a  
lattice of our choosing, and then bin the data to form counts---so that now each
$y_i$ represents the count in a Poisson model for bin $i$ (lattice point
$x_i$). This is closely related to \citet{padilla2016nonparametric,
  bassett2019fused} who study similar ideas in different contexts. 

Generally, we note that the properties in Section \ref{sec:properties}, as well
as the interpolation algorithm in Section \ref{sec:interpolation}, all carry
over directly to \eqref{eq:ktf_glm}, since they are driven entirely by the
structure of the KTF penalty (which remains unchanged in
\eqref{eq:ktf_glm}). The optimization algorithms from Section
\ref{sec:optimization} could be applied to \eqref{eq:ktf_glm}, as the inner loop 
of a proximal Newton method (where the outler loop makes a weighted quadratic 
approximation to the loss in \eqref{eq:ktf_glm}). Meanwhile, estimation theory
will be more difficult to extend, though for upper bounds, simple truncation
arguments (as in, for example, \citet{lin2017sharp}) may be a viable approach.     

\subsection{Mixed discrete derivatives}

The KTF penalty operator computes discrete derivatives of order $k+1$ along each
coordinate direction. This can be extended by considering \emph{mixed} discrete
derivatives of total order $k+1$; in the notation of Section \ref{sec:ktf}, this
is 
$$
\sum_{|\alpha| = k+1} \sum_{x \in Z_{n,d}} \big|\big ( \dop_{x_1^{\alpha_1},
  \ldots, x_d^{\alpha_d}} \theta \big)(x) \big|,
$$
where the outer sum is over all multi-indices $\alpha \in \R^d_+$ of size
$|\alpha| = \alpha_1 + \cdots + \alpha_d = k+1$. The analog of the penalty
operator in \eqref{eq:ktf_pen_mat} is now: 
\begin{equation}
\label{eq:ktf_mixed_pen_mat}
\tilde{\dmat}_{n,d}^{(k+1)} = \left[\begin{array}{c}
D_N^{(\alpha^1_1)} \otimes D_N^{(\alpha^1_2)} \otimes \cdots 
D_N^{(\alpha^1_d)} \\ \vdots \\
D_N^{(\alpha^P_1)} \otimes D_N^{(\alpha^P_2)} \otimes \cdots 
D_N^{(\alpha^P_d)} 
\end{array}\right],
\end{equation}
where \smash{$P = {k+d \choose d-1}$}, 
and $\alpha^p \in \{0, \ldots, k+1\}^d$, $p=1,\ldots,P$ enumerate all
multi-indices of size $k+1$. (Recall that we use \smash{$D_N^{(0)} = I_N$} by 
convention.) Compared to \eqref{eq:ktf_pen_mat}, the mixed penalty operator in
\eqref{eq:ktf_mixed_pen_mat} has many more rows, which leads to computational 
difficulty even for moderate $k,d$. However, in simple comparisons, we have
found KTF with mixed derivatives to work well in certain problems, likely
explained by the fact that it is is less anisotropic (bringing it closer, in a
sense to the GTF penalty operator), and thus can work well when some degree of
isotropy is warranted. 

We note that the null space of \eqref{eq:ktf_mixed_pen_mat} has dimension
\smash{${k+d \choose d}$}  
and consists of evaluations of all polynomials of degree $k$, which provides the
analog of Proposition \ref{prop:ktf_null_space}. However, a continuous-time
representation as in Proposition \ref{prop:ktf_continuous} is likely not
possible, and extensions of results in the rest of the paper would be highly
nontrivial, and a topic for future work. 

\subsection{Coordinate-specific smoothness}

Proceeding in some sense in an opposite direction to the extension in the last
section, we can also generalize KTF to make it \emph{more} anisotropic, by
allowing the degree of modeled smoothness to differ for each coordinate. In the
notation of Section \ref{sec:ktf}, the extended penalty would be 
$$
\sum_{j=1}^d \sum_{x \in Z_{n,d}} |(\dop_{x_j^{k_j+1}} \theta)(x)|. 
$$
where $k_j \geq 0$ are (possibly) distinct integers, for $j=1,\ldots,d$, and the
analogous penalty operator would be
\begin{equation}
\label{eq:ktf_coord_pen_mat}
\tilde{\dmat}^{(k_1,\ldots,k_d)}_{n,d} = \left[\begin{array}{c}  
\dmat^{(k_1)}_N \otimes I_N \otimes \cdots \otimes I_N \\    
I_N \otimes \dmat^{(k_2)}_N \otimes \cdots \otimes I_N \\ 
\vdots \\
I_N \otimes I_N \otimes \cdots \otimes \dmat^{(k_d)}_N  
\end{array}\right].
\end{equation}
The properties and results derived Sections \ref{sec:properties},
\ref{sec:optimization}, and \ref{sec:interpolation} translate with
straightforward modification (in many instances, merely notational) to the
generalized KTD estimator with the penalty operator as in
\eqref{eq:ktf_coord_pen_mat}. With some work, we expect much of the estimation
theory in Section \ref{sec:theory} to carry over as well, where we conjecture
that the appropriate notation of effective smoothness will be \smash{$s = 
  (\sum_{j=1}^d 1/(k_j+1))^{-1}$}, by analogy to minimax theory on anisotropic
Besov spaces. 

\subsection{Scattered data}
\label{sec:scattered_data}

Saving the most ambitious extension for last, the continuous-time representation
of KTF from Section \ref{sec:ktf_continuous}, specifically the basis form in
\eqref{eq:ktf_basis}, suggests an approach for extending KTF to scattered data. 
For arbitrary design points $x_i \in \R^d$, $i=1,\ldots,n$, we form a lattice
of our choosing $z_1 \times \cdots \times z_d$, where \smash{$z_j = 
  \{z_{ij}\}_{i=1}^{N_j}$}, $j=1,\ldots,d$ and \smash{$m = \prod_{j=1}^d
  N_j$}. We then define for each $i=1,\ldots,n$ and $j=1,\ldots,d$:   
$$
h_{x_i, z_j} = \big( h^k_{z_j,1}(x_{ij}), h^k_{z_j, 2}(x_{ij}), \ldots,
h^k_{z_j, N_j}(x_{ij}) \big) \in \R^{N_j},
$$
where \smash{$h^k_{z_j,\ell}$}, $\ell=1,\ldots,N_j$ are the falling factorial
basis functions defined over the knot set $z_j$ (see
\citet{tibshirani2020divided} for the general definition). The extended basis
form of KTF for scattered data is now:
\begin{equation}
\label{eq:ktf_scattered}
\minimize_{\alpha \in \R^m} \; \frac{1}{2} \VAST\|
y - \begin{bmatrix}
h_{x_1, z_1} \otimes h_{x_1, z_2} \otimes \cdots\otimes h_{x_1, z_d} \\ 
h_{x_2, z_1} \otimes h_{x_2, z_2} \otimes \cdots\otimes h_{x_2, z_d} \\
\vdots \\
h_{x_n, z_1} \otimes h_{x_3, z_2} \otimes \cdots\otimes h_{x_n, z_d}
\end{bmatrix} \alpha \VAST\|_2^2
+ \lambda \left\| \dmat^{(k+1)}_{z_1,\ldots,z_d} \begin{bmatrix}
H_{z_1}^{(k)} \otimes H_{z_2}^{(k)} \otimes \cdots H_{z_d}^{(k)}
\end{bmatrix} \alpha \right\|_1.
\end{equation}
Here \smash{$\dmat^{(k+1)}_{z_1,\ldots,z_d}$} denotes the KTF penalty operator
over the lattice $z_1 \times \cdots \times z_d$, as defined in
\eqref{eq:ktf_general_lattice_pen_mat}, and each \smash{$H^{(k)}_{z_j}$} is 
the $k\th$ order univariate falling factorial basis matrix defined with respect
to $z_j$, for $j=1,\ldots,d$. A solution \smash{$\halpha$} in 
\eqref{eq:ktf_scattered} for the coefficients in the basis expansion can be
used to form fitted values as well as predictions at arbitrary $x \in \R^d$. 



\subsection*{Acknowledgements}

The authors thank James Sharpnack and Alden Green for numerous insightful
discussions. This material is based upon work supported by the National Science 
Foundation under grant DMS-1554123 and Graduate Research Fellowship Program 
award DGE1745016.  

{\RaggedRight
\bibliographystyle{plainnat}
\bibliography{ryantibs}}

\clearpage
\appendix
\allowdisplaybreaks
\section{Proofs}
\label{app:proofs} 

\subsection{Proof of Proposition \ref{prop:ktf_null_space}}

Abbreviating \smash{$D = \kronmatk$}, the null space of $D$ is the number of its
nonzero singular values or equivalently, the number of nonzero eigenvalues of
$D^\T D$. Following from \eqref{eq:ktf_pen_mat}, and abbreviating \smash{$Q =
  \dmatk_N$} and $I = I_N$, 
$$
D^\T D = Q^\T Q \otimes I \otimes \cdots \otimes I \;+\;
I \otimes Q^\T Q \otimes \cdots \otimes I \;+\; \ldots \;+\;   
I \otimes I \otimes \cdots \otimes Q^\T Q,
$$
the Kronecker sum of \smash{$Q^\T Q$} with itself, a total of $d$ times. Using a
standard fact about Kronecker sums, if we denote by $\rho_i$, $i=1,\ldots,N$ the
eigenvalues of $Q^\T Q$ then 
$$
\rho_{i_1 }+ \rho_{i_2} + \cdots + \rho_{i_d}, \quad i_1,\ldots,i_d \in
\{1,\ldots, N\} 
$$
are the eigenvalues of $D^\T D$. By counting the multiplicity of the zero
eigenvalue, we arrive at a nullity for $D$ of $(k+1)^d$. It is straightforward
to check that the vectors specified in the proposition, given by evaluations of 
polynomials of max degree $k$, are in the null space, and that these are
linearly independent, which completes the proof.
\hfill\qedsymbol

\subsection{Proof of Proposition \ref{prop:ktf_continuous}}

Let us define
$$
B^{(k+1)}_N =
\begin{bmatrix}
C^{(k+1)}_N \\ \dmatk_N
\end{bmatrix}
\in \R^{N \times N},
$$
where the first $k+1$ rows are given by a matrix \smash{$C^{(k+1)}_N \in  
  \R^{(k+1) \times N}$} that completes the row space, as in
Lemma 2 of \citet{wang2014falling}, or Section 6.2 of
\citet{tibshirani2020divided}. Now, again by Lemma 2 of \citet{wang2014falling},
or Section 6.3 of \citet{tibshirani2020divided},
\begin{equation}
\label{eq:h_inv}
\big(\hmatk_N \big)^{-1} = \frac{1}{k!} B^{(k+1)}_N
\end{equation}
where \smash{$\hmatk_N \in \R^{N \times N}$} is the falling factorial basis matrix
of order $k$, which has elements  
$$
\big[\hmatk_N\big]_{ij} = h^k_{N,j}(i/N), \quad i,j=1,\ldots,N,
$$
with \smash{$h^k_{N,i}$}, $i=1,\ldots,N$ denoting the falling factorial
functions in \eqref{eq:ffb} with respect to design points $1/N, 2/N, \ldots, 1$.  

We now transform variables in \eqref{eq:ktf} by defining 
$$
\theta = \Big(\hmatk_N \otimes \cdots \otimes \hmatk_N\Big) \alpha, 
$$
and using \eqref{eq:h_inv}, this turns \eqref{eq:ktf} into an equivalent basis
form,  
$$
\minimize_{\alpha \in \R^n} \; \frac{1}{2} \Big\| y - \Big(\hmatk_N \otimes 
\cdots \otimes \hmatk_N\Big) \alpha \Big\|_2^2 +  
\lambda k! \left\| \left[\begin{array}{c}
I^0_N \otimes \hmatk_N \otimes \cdots \otimes \hmatk_N \\
\hmatk_N \otimes I^0_N \otimes \cdots \otimes \hmatk_N \\
\vdots \\
\hmatk_N \otimes \hmatk_N \otimes \cdots \otimes I^0_N
\end{array}\right] \alpha \right\|_1,
$$
where \smash{$I^0_N \in \R^{(N-k-1) \times N}$} denotes the last $N-k-1$ rows of
the identity $I_N$. We can rewrite the problem once more by parametrizing the
evaluations according to $f$ in \eqref{eq:ffb_tensor}, which we claim yields
\eqref{eq:ktf_continuous}. The equivalence between loss terms in the above
problem and \eqref{eq:ktf_continuous} is immediate (by definition of
$f$); to see the equivalence between penalty terms, it can be directly checked
that  
$$
k! \Big( I^0_N \otimes \hmatk_N \otimes \cdots \otimes \hmatk_N \Big) \alpha 
$$
contains the differences of the function \smash{$\partial^k f/\partial x_1^k$}
over all pairs of grid positions that are adjacent in the $x_1$ direction, where
$f$ is as in \eqref{eq:ffb_tensor}. This, combined with the fact that
\smash{$\partial^k f/\partial x_1^k$} is constant in between lattice positions,
means that  
$$
k! \Big\|\Big( I^0_N \otimes \hmatk_N \otimes \cdots \otimes \hmatk_N \Big) 
\alpha\Big\|_1 = \sum_{x_{-1}} \TV\bigg(\frac{\partial^k f(\cdot,x_{-1})}
{\partial x_1^k}\bigg), 
$$
the total variation of \smash{$\partial^k f/\partial x_1^k$} added up over all
slices of the lattice $Z_{n,d}$ in the $x_1$ direction. Similar arguments apply
to the penalty terms corresponding to dimensions $j=2,\ldots,d$, and this
completes the proof.  
\hfill\qedsymbol

\subsection{Proof of Theorem \ref{thm:tv_lines}}

Denote by $f_\epsilon = \eta_\epsilon * f$ the mollified version of $f$, where 
\smash{$\eta_\epsilon(x) = \epsilon^{-d} \eta(x/\epsilon)$}, and $\eta : \R^d
\to \R$ is the standard mollifier, defined by  
$$
\eta(x) = \begin{cases}
\displaystyle 
c \exp\bigg(\frac{1}{\|x\|_2^2 - 1}\bigg) & \text{if $\|x\|_2 \leq 1$}, \\
0 & \text{else},
\end{cases}
$$
and $c>0$ is a normalization constant so that $\eta$ integrates to 1. 
By construction, for any $\epsilon>0$, we have \smash{$f_\epsilon \in
  C^\infty(U)$}. Now from \eqref{eq:tv_smooth}, simply exchanging the sum over
absolute partial derivatives and the integral, we have
\begin{align*}
\TV(f_\epsilon; U) 
&= \sum_{j=1}^d \int_U \bigg| \frac{\partial f_\epsilon(x)}{\partial x_j} \bigg|
  \, dx \\ 
&= \sum_{j=1}^d \int_{U_{-j}}\int_{I_{x_{-j}}} \bigg| \frac{\partial 
  f_\epsilon(x_j, x_{-j})}{\partial x_j} \bigg| \, dx_j \, dx_{-j} \\   
&= \sum_{j=1}^d \int_{U_{-j}} \TV \big( f_\epsilon(\cdot,x_{-j}); I_{x_{-j}}
\big) \, dx_{-j},
\end{align*}
where in the last line we applied the representation of TV for smooth functions
in \eqref{eq:tv_smooth}, but to the univariate function $x_j \mapsto
f_\epsilon(x_j, x_{-j})$, for fixed $x_{-j}$. Recalling standard results on
approximation of BV functions by smooth functions (see, for example, Theorem 
5.22 in \citet{evans2015measure}), by sending $\epsilon \to 0$, we have that the
left-most and right-most sides of the previous display approach those in 
\eqref{eq:tv_lines}, completing the proof. 


\subsection{Proof of Theorem \ref{thm:ktf_upper_bd}}
\label{sec:pf_ktf_upper_bd}
Abbreviate $N' = N-k-1$. Let $\beta_i,u_i,v_i$ be a 
triplet of nonzero singular
value, left singular vector, and right singular vector of 
\smash{$\dmatk_{N,1}$}, for
$i \in [N']$ and let $p_j$, $j \in [k+1]$ form an 
orthogonal basis for the null 
space of
\smash{$\dmatk_{N,1}$}.
From Lemma~\ref{lem:incoherence_kron} it suffices
to show incoherence of $u_i,v_i$, $i \in [N']$, and 
$p_i$, $i \in [k+1]$.  
Incoherence of $u_i$, $i \in [N']$ and $v_i$, $i \in 
[N']$is established in 
\cite{sadhanala2017higher} . 
Incoherence of $p_i$, $i \in [k+1]$ may be seen by 
choosing, e.g., these vectors
to be the discrete Legendre orthogonal polynomials as in
\citet{neuman1974discrete}.   Applying 
Lemma~\ref{lem:incoherence_kron}, we can 
see that 
\smash{$\kronmatk$}
satisfies the incoherence property, as defined in Theorem
\ref{thm:genlasso_upper_bd}, say with a constant $\mu$.

From the incoherence property and Theorem 
\ref{thm:genlasso_upper_bd}, the KTF 
estimator $\htheta$, satisfies 
\begin{align}
	\label{eq:genlasso_rate_i0_ktf}
\frac 1n \| \htheta - \theta_0 \|_2^2  
=  O_\P \Bigg( 
	\frac{\kappa}{n} + \frac{|I|}{n} + 
	\frac{\mu}{n} \sqrt{ \frac{\log n}{n} 
		\sum_{i \in [N]^d \setminus (I \cup [k+1]^d)} 
		\frac{1}{\xi_i^2}} \cdot \| \Delta \theta_0 \|_1\, 
		\Bigg),
\end{align}
where we abbreviate \smash{$\Delta = \kronmatk$},
\smash{$\xi_i,  i \in [N]^d$} are eigenvalues of 
\smash{$\Delta^\T \Delta$}
with \smash{$\xi_i = 0$} for \smash{$i \in [k+1]^d$}. We 
reindexed the eigenvalues so that they correspond to grid positions.

Recall the shorthand $s = (k+1) / d$.
For $s \leq  1/2$, set \smash{$I = [k+2]^d \setminus [k+1]^d$}.
From Lemma~\ref{lem:kron_tf_evals},
\begin{align*}
	\sum_{i \in [N]^d \setminus (I \cup [k+1])^d}  
	\frac{1}{\xi_i^2} 
	\leq c
	\begin{cases}
		n & s < 1/2\\
		n \log n & s = 1/2.
	\end{cases}
\end{align*}
Plugging this into \eqref{eq:genlasso_rate_i0_ktf} gives 
the desired bounds when $s \leq 1/2$:
\begin{align*}
\frac 1n \| \htheta - \theta_0 \|_2^2 = 
	\begin{cases}
		O_\P \Big( \frac{(k+2)^d}{n} + \| \Delta \theta_0 
		\|_1  
		\sqrt{\log n} \Big) & s < 1/2\\
		O_\P \Big( \frac{(k+2)^d}{n} + \| \Delta \theta_0 
		\|_1 \log n 
		\Big)& s = 1/2.
	\end{cases}
\end{align*}
\noindent
Now consider the case $ s > 1/2$.
Set $I = \{ i \in [N]^d : \| (i - k - 2)_+ \|_2 < r\} 
\setminus [k+1]^d$ for an 
$r\in [1, \sqrt{d}N]$ to be 
chosen 
later. 
\smash{$|I| \leq  (r+k+1)^d$} because \smash{$I \subseteq 
\big[ \lfloor r 
	\rfloor + k + 1 \big]^d$}.
Lemma \ref{lem:kron_tf_evals} shows that for a constant 
$c>0$ depending only on 
$k, d$,
\begin{align*}
	\sum_{i \in [N]^d \setminus (I \cup [k+1]^d) }  
	\frac{1}{\xi_i^2} 
	\leq c n^{2s} / r^{d(2s-1)}.
\end{align*}
Plug this bound in \eqref{eq:genlasso_rate_i0_ktf} and 
in order to minimize the resulting bound, choose $r$ to 
balance
$$
(r+k+1)^d  \quad\text{with}\quad \frac{C_n}{\sqrt{n}} 
\sqrt{  n^{2s} / 
	r^{d(2s-1) } \log n}.
$$
This leads us to take $$ (r+k+1)^d \asymp  (C_n\sqrt{\log 
n})^{\frac{2}{2s+1}} 
n^{\frac{2s-1}{2s+1}}$$
when \smash{$C_n \sqrt{\log n} / n = O_\P(1)$} 
and $r=1$ otherwise. With this choice 
\eqref{eq:genlasso_rate_i0_ktf} gives the 
desired bound for 
$s>1/2$. This completes the proof.
\hfill\qedsymbol

\subsection{Incoherence of Kronecker product type 
operators}
Let 
\begin{equation}
	\label{eq:ktf_pen_mat2}
	\Delta = \left[\begin{array}{c}
		D \otimes I \otimes \cdots \otimes I \\ 
		I \otimes D \otimes \cdots \otimes I \\ 
		\vdots \\
		I \otimes I \otimes \cdots \otimes D  
	\end{array}\right]
\end{equation}
where each Kronecker product has $d$ terms.
With \smash{$D = \dmatk_{N,1}, I = I_N$}, we get the 
KTF penalty operator $\Delta = \kronmatk$.

\begin{lemma}
	\label{lem:incoherence_kron}
	Let $\Delta$ be as defined in \eqref{eq:ktf_pen_mat2} 
	for a matrix $D 
	\in \R^{N'\times N}$ with $N' 
	\leq 
	N$. 
	Let $\gamma_i, u_i, v_i, i\in [N]$ denote the singular 
	values of $D$, 
	its left and right singular 
	vectors. Note that $\gamma_i=0, u_i=0, v_i \in \nul(D)$ 
	for $i\in [p]$ 
	where $p=\nuli(D)$.
	If these singular vectors are incoherent, that is 
	$\|v_i\|_\infty \leq \mu/\sqrt{N}, \|u_i\|_\infty \leq 
	\mu/\sqrt{N'}$ 
	for a constant $\mu \geq 1$, then
	the left singular vectors $\nu$ of $\Delta$ are 
	incoherent with a 
	constant $\mu^d$, that is,
	$\|\nu\|_\infty \leq \mu^d/\sqrt{N^{d-1} N'}$.
\end{lemma}
Note that 
$p=k+1$ when $\Delta$ is the KTF penalty operator with 
$D= D_{N,1}^{(k+1)}.$
\begin{proof}[Proof of Lemma~\ref{lem:incoherence_kron}]
	Abbreviate $\rho_i = \gamma_i^2$ for $i\in [N]$.
	We are looking for a total of $N^d - p^d$ eigenvectors 
	for $\Delta 
	\Delta^\T$.
	Assume for exposition that $d=3$. For any $(i,j,k)\in 
	[N]^d \setminus 
	[p]^d$ (where $\setminus$ is 
	the set difference operator), the vectors
	\begin{align}
		\label{eq:eigv_ddt_3d}
		\nu_{i,j,k} := 
		\frac{1}{\sqrt{\rho_i+\rho_j+\rho_k}}
		\left[ \begin{array}{c}
			\gamma_i \cdot u_i \otimes v_j \otimes v_k \\
			\gamma_j \cdot v_i \otimes u_j \otimes v_k \\
			\gamma_k \cdot v_i \otimes v_j \otimes u_k
		\end{array}\right]
	\end{align}
	are eigenvectors of $\Delta \Delta^\T$ as verified below.
	\begin{align}
		\Delta \Delta^\T 
		\left[ \begin{array}{c}
			\gamma_i \cdot u_i \otimes v_j \otimes v_k \\
			\gamma_j \cdot v_i \otimes u_j \otimes v_k \\
			\gamma_k \cdot v_i \otimes v_j \otimes u_k
		\end{array}\right]
		&= 
		\Delta \; \left( 
		\gamma_i^2  +
		\gamma_j^2 +
		\gamma_k^2
		\right) v_i \otimes v_j \otimes v_k 
		\\
		\nonumber
		&=
		\left( \rho_i + \rho_j + \rho_k \right) 
		\left[ \begin{array}{c}
			\gamma_i \cdot u_i \otimes v_j \otimes v_k \\
			\gamma_j \cdot v_i \otimes u_j \otimes v_k \\
			\gamma_k \cdot v_i \otimes v_j \otimes u_k
		\end{array}\right]
	\end{align}
	We see all $N^d - p^d$ eigenvectors of $\Delta 
	\Delta^\T$ here. 
	Notice that $\|z_{i,j,k}\|_2=1$ and the incoherence is 
	readily 
	available given that the left and right 
	singular vectors of $D$ are incoherent.
	
	For general $d$, these $N^d-p^d$ eigenvectors are given 
	by
	\begin{align}
		\label{eq:eigv_ddt_gend}
		\nu_{i_1,i_2,.,i_d} = 
		\frac{1}{\sqrt{\sum_{j=1}^d \rho_{i_j}}}\left[ 
		\begin{array}{c}
			\gamma_{i_1} \cdot u_{i_1} \otimes v_{i_2} \otimes 
			\hdots v_{i_d} \\
			\gamma_{i_2} \cdot v_{i_1} \otimes u_{i_2} \otimes 
			\hdots v_{i_d} \\
			\vdots\\
			\gamma_{i_d} \cdot v_{i_1} \otimes v_{i_2} \otimes 
			\hdots u_{i_d}
		\end{array}\right]
	\end{align}
	with eigenvalues $\sum_{j=1}^d \rho_{i_j}$ and are 
	easily seen to be 
	incoherent. 
\end{proof}

\subsection{Upper bound for continuous KTV class }
\label{app:ktf_upper_bd_continuous}
Recalling the continuous analog of KTF penalty from \eqref{eq:ktf_continuous}, 
define 
the class
 \begin{equation*}
	 	\KTV_{n,d}^k(C)  = \bigg\{ f : \sum_{j=1}^d \sum_{x_{-j}} \TV \bigg( 
	 	\frac{\partial^k f(\cdot,x_{-j})}{ \partial x_j^k}  
	 	\bigg)  \leq C\bigg\}
	 \end{equation*}
 for $C>0$.
 If the true signal $\theta_0$ on the grid is an evaluation of a function $f 
\in \KTV_k^d(C)$, 
 the rates in Theorem~\ref{thm:ktf_upper_bd}  hold with $C_n$ replaced by $C,$
 due to the following result. 
 \begin{lemma}
	 	\label{lem:ktf_penalty_continuous}
	 	Let $C>0$ and let $d\geq 1,k\geq 0$ be integers. 
	 	For all $f\in \KTV_{n,d}^k(C),$ if $\theta_f \in \R^n$ is the 
evaluation of $f$ on the grid points $Z_{n,d}$, 
	 	then
	 	$$
	 	\| \kronmatk \theta_f \|_1 \leq c_1 C
	 	$$
	 	for a constant $c_1$ that depends only on $k$ and $d$.
	 \end{lemma}

\begin{proof}[Proof of Lemma~\ref{lem:ktf_penalty_continuous}]
	Let $f$ be an arbitrary function from $\KTV_{n,d}^k(C)$. 
	Pick a $j \in [N]$ and an 
	$x_{-j}$ and consider the function 
	$\phi(\cdot) = f(\cdot,x_{-j})$ ($f$ with all but its $j$ argument 
fixed to elements of $x_{-j}$ 
	appropriately in order). 
	From Theorem 1 in \cite{mammen1991nonparametric} and its proof, there 
exists a spline 
	$\tilde{\phi}$ such that
	\begin{align*}
		\tilde{\phi}(i/N) &= \phi(i/N), \quad i \in [N] \\
		\TV(\tilde{\phi}^{(k)}) &\leq \TV(\phi^{(k)})
	\end{align*}
	Let $t_1,\dots,t_L$ be the knots of $\tilde{\phi}$, which are not 
necessarily in the set of input points.
	Because it is a spline, $\tilde{\phi}$ can be written as the sum of a 
polynomial and a linear 
	combination of $k$th degree truncated power basis functions $g_t : x 
\mapsto (x-t)_+^k / k!$ 
	$$
	\tilde{\phi}(u) = p(u) + \sum_{\ell=1}^L \beta_\ell g_{t_\ell}(u), 
\quad u \in [0,1]
	$$
	where $p$ is a polynomial of degree $\leq k$ and $\beta_\ell \in \R, 
\ell \in [L]$.
	Let \smash{$D_{\mathrm{1d}}^{(k+1)} = \dmatk_{N,1}$}.
	Now 
	\begin{align}
		\nonumber
		\bigg\| D_{\mathrm{1d}}^{(k+1)}  \begin{bmatrix} \phi(1/N) \\ 
\vdots \\  \phi(N/N) \end{bmatrix} 
		\bigg\|_1
		&= 
		\bigg\| D_{\mathrm{1d}}^{(k+1)}  \begin{bmatrix} 
\tilde{\phi}(1/N) \\ \vdots \\  \tilde{\phi}(N/N) 
		\end{bmatrix} \bigg\|_1 \\
		\nonumber
		&= 
		\bigg\| D_{\mathrm{1d}}^{(k+1)}  
		\begin{bmatrix} p(1/N) \\ \vdots \\  p(N/N) \end{bmatrix}
		+ \sum_{\ell=1}^L \beta_\ell \cdot D_{\mathrm{1d}}^{(k+1)} 
		\begin{bmatrix} g_{t_\ell}(1/N) \\ \vdots \\  g_{t_\ell}(N/N) 
\end{bmatrix}
		\bigg\|_1 \\
		\label{eq:Dphi_1}
		&\leq \sum_{\ell=1}^L |\beta_\ell|  \| D_{\mathrm{1d}}^{(k+1)} 
G_\ell^{(k)} \|_1
	\end{align}
	where the vector $G_\ell^{(k)}$ is the evaluation of $g_{t_\ell}$ on 
$1/N, \dots, N/N$, that is 
	$(G_\ell^{(k)})_i = g_{t_\ell}(i/N), i\in [N]$. Here we used the fact 
that $D_{\mathrm{1d}}^{(k+1)} $ 
	times the evaluations of a polynomial at the input points 
$1/N,\dots,N/N$ is $0$. 
	
	The terms in \eqref{eq:Dphi_1} can be bound as follows. 
	For $\ell\in [L]$, let $i_\ell = \max_{i\in[N]} \{ i/N \leq t_\ell\}$, 
that is, let $i_\ell/N$ be the largest 
	input point that is not greater the knot $t_\ell$.
	For any vector $v\in \R^N$, and $i\in [N-k-1]$, 
$(D_{\mathrm{1d}}^{(k+1)} v)_i = 0 $  if $(v_i, 
	\dots,v_{i+k})$ is the evaluation of a polynomial at 
$i/N,\dots,(i+k+1)/N$. 
	$g_{t_\ell}$ is a polynomial on $[0,t_\ell]$ and on $[t_\ell, 1]$ for 
$\ell \in [L]$. Therefore, 
	$ D_{\mathrm{1d}}^{(k+1)} G_\ell^{(k)} $ is nonzero in at most $k+1$ 
elements.
	Letting $A_i$ denote the $i$th row of matrix $A$, we can write
	\begin{align*}
		\| D_{\mathrm{1d}}^{(k+1)} G_\ell^{(k)} \|_1
		&=
		\sum_{i=1}^{N-k-1}  \left| \big( D_{\mathrm{1d}}^{(k+1)} 
\big)_i  G_\ell^{(k)}  \right|\\
		&=
		\sum_{i=(i_\ell-k)\vee 1}^{i_\ell}  \left| \big( 
D_{\mathrm{1d}}^{(k+1)} \big)_i  G_\ell^{(k)}  \right|\\
		&\leq 
		\sum_{i=(i_\ell-k)\vee 1}^{i_\ell}  \left\| \big( 
D_{\mathrm{1d}}^{(k+1)} \big)_i \right\|_1  
		g_{t_\ell}\left(\frac{i_\ell + k+1}{N} \right) \\
		&\leq 
		\sum_{i=(i_\ell-k)\vee 1}^{i_\ell}  \left\| \big( 
D_{\mathrm{1d}}^{(k+1)} \big)_i \right\|_1  \left( 
		\frac{k+1}{N} \right)^k \frac{1}{k!}\\
		&\leq 
		(k+1) \cdot \left( \frac{k+1}{N} \right)^k \frac{1}{k!} 
\max_{i\in[N-k-1]} \left\| \big( 
		D_{\mathrm{1d}}^{(k+1)} \big)_i \right\|_1 \\
		&=
		(k+1) \cdot \left( \frac{k+1}{N} \right)^k \frac{1}{k!}  \cdot 
2^{k+1} N^k\\
		&= b_k
	\end{align*}
	where $b_k$ is a constant depending only on $k$. Plugging this upper 
bound in 
	\eqref{eq:Dphi_1}, 
	\begin{align}
		\nonumber
		\bigg\| D_{\mathrm{1d}}^{(k+1)}  \begin{bmatrix} \phi(1/N) \\ 
\vdots \\  \phi(N/N) \end{bmatrix} 
		\bigg\|_1
		&\leq 
		b_k \sum_{\ell=1}^L | \beta_\ell | = b_k 
\TV(\tilde{\phi}^{(k)}) \leq b_k \TV(\phi^{(k)}).
	\end{align}
	This means,
	\begin{align*}
		\| \kronmatk \theta_f \|_1 &=
		\sum_{j=1}^d \sum_{x_{-j}} 
		\bigg\| D_{\mathrm{1d}}^{(k+1)}  \begin{bmatrix} f(1/N,x_{-j}) 
\\ \vdots \\  f(N/N,,x_{-j}) 
		\end{bmatrix} \bigg\|_1\\
		&\leq  \sum_{j=1}^d \sum_{x_{-j}} 
		b_k \TV \bigg( 
		\frac{\partial^k f(\cdot,x_{-j})}{ \partial x_j^k}  
		\bigg) \\
		&\leq b_k \cdot C.
	\end{align*}
	This completes the proof.
\end{proof}

\subsection{Proof of Theorem \ref{thm:ktv_lower_bd}}
\label{sec:pf_ktv_lower_bd}

Here and henceforth,
we use the notation \smash{$B_p(r) = \{ x : \|x\|_p \leq 
r \}$} for
the $\ell_p$ ball of radius $r$, where $p,r > 0$ (and the 
ambient
dimension will be determined based on the context).
\begin{lemma}[Lemma 7 in \citet{sadhanala2016total}]
	\label{lem:l1ball_embed_Dset}
	Let $\cT(r) = \{ \theta \in \R^n : \| \genmat \theta 
	\|_1 \leq r \}$ 
	for a matrix $\genmat$ and $r>0$.
	Recall that $\|\genmat\|_{1,\infty}= \max_{i\in [n]} \| 
	\genmat_i \|_1$ 
	where $\genmat_i$ is the $i$th 
	column of $\genmat$.
	Then for any
	$r>0$, it holds that \smash{$B_1(r/ 
	\|\genmat\|_{1,\infty}) \subseteq 
		\cT(r)$}. 
\end{lemma}
\noindent
From Lemma~\ref{lem:l1ball_embed_Dset} and the fact that
\smash{$\|\kronmatk\|_{1,\infty} = 	2^{k+1} d$}
	\begin{equation}
	\label{eq:ell1_ball_embedding_ktf}
	B_1(r/(2^{k+1} d) ) \subseteq \kronset^k(r).
\end{equation}
for any $r> 0$, and integers $d\geq 1, k\geq 0$.

To prove Theorem \ref{thm:ktv_lower_bd} we will use the 
following  result from
\citet{birge2001gaussian}, which gives a lower bound for 
the risk in a
normal means problem, over $\ell_p$ balls.
We state the result in our notation.
\begin{lemma}[Proposition 5 of \citet{birge2001gaussian}]
	\label{lem:birge}
	Assume i.i.d.\ observations \smash{$y_i \sim
		N(\theta_{0,i},\sigma^2)$}, $i=1,\ldots,n$, and $n 
		\geq 2$.
	Then the minimax risk over the $\ell_p$ ball 
	$B_p(r_n)$, where
	$0 < p < 2$, satisfies
	\begin{equation*}
		n \cdot R(B_p(r_n)) \geq
		c \cdot
		\begin{cases}
			\displaystyle
			\sigma^{2-p} r_n^p
			\bigg[1+\log\bigg(
			\frac{\sigma^p n}{r_n^p}\bigg)
			\bigg]^{1-p/2} &
			\text{if $\sigma\sqrt{\log n}\leq r_n \leq
				\sigma n^{1/p}/\sqrt{\rho_p}$} \\
			r_n^2 &\text{if $ r_n< \sigma \sqrt{\log n}$} \\
			\sigma^2 n / \rho_p &\text{if
				$r_n > \sigma n^{1/p} / \sqrt{\rho}$}
		\end{cases}.
	\end{equation*}
	Here $c>0$ is a universal constant, and $\rho_p > 1.76$ 
	is the unique
	solution of $\rho_p\log\rho_p = 2/p$.
\end{lemma}

\begin{proof}[Proof of Theorem \ref{thm:ktv_lower_bd}]
	It suffices to show that the minimax optimal risk 
	$R\big( \kronset^k(C_n) 
	\big)$ 
	is lower bounded by the three terms present in the 
	statement's lower bound 
	separately:
	\begin{align}
		\nonumber
		R\big( \kronset^k(C_n) \big) &= \Omega \bigg( 
		\frac{\kappa \sigma^2}{n} 
		\bigg), \\
		\label{eq:kron_three_lower_bds}
		R\big( \kronset^k(C_n) \big) &= \Omega \bigg( 
		\frac{\sigma C_n}{n} 
		\wedge \sigma^2 \bigg),\\
		\nonumber
		R\big( \kronset^k(C_n) \big) &= \Omega \bigg( 
		\bigg( \frac{C_n}{n} \bigg)^{\frac{2}{2s+1}} 
		\sigma^{\frac{4s}{2s+1}} 
		\wedge \sigma^2 \bigg),
	\end{align}
	where $\kappa=\nuli\big(\kronmatk\big) = (k+1)^d$.
	First, as the null space of $\kronmatk$ has dimension 
	$\kappa$, we get the 
	first lower bound:
	\begin{align*}
		\inf_{\htheta} \sup_{\theta_0 \in \kronset^k(C_n) } 
		\frac{1}{n} \E \| \htheta - \theta_0 \|_2^2 
		\; \geq \;
		\inf_{\htheta} \sup_{\theta_0 \in 
		\nul\big(\kronmatk\big) } 
		\frac{1}{n} \E \| \htheta - \theta_0 \|_2^2 
		\;\geq \;
		\frac{\kappa \sigma^2}{n}.
	\end{align*}
	We get the second lower bound in 
	\eqref{eq:kron_three_lower_bds} by using the 
	$\ell_1$-ball 
	embedding 
	$$
	B_1 \left(C_n/d_{\max} \right) \subset \kronset^k( C_n)
	$$
	from \eqref{eq:ell1_ball_embedding_ktf} and then 
	using 
	Lemma~\ref{lem:birge}.
	Finally, from Theorem 4 in \citet{sadhanala2017higher}, 
	it follows that
	\begin{equation}
		\label{eq:holder_lower_bd}
		R\big(\holderset^{k}(L_n)\big) = \Omega \Big( 
		\Big(\frac{\sigma^2}{n} 
		\Big)^{\frac{2s}{2s+1}}
		L_n^{\frac{2}{2s+1}} \wedge \sigma^2 \Big)
	\end{equation}
	with additional tracking for $\sigma^2.$
	Taking \smash{$L_n=C_n /n^{1-s}$} and applying the embedding in
	Proposition~\ref{prop:ktv_sobolev_holder_embed} would 
	then give the third lower 
	bound in 
	\eqref{eq:kron_three_lower_bds}. This completes the 
	proof.
\end{proof}

\subsection{Proof of Theorem \ref{thm:ktv_minimax_linear}} 
\label{sec:pf_minimax_linear}

We use the following shorthand for the  risk of an estimator \smash{$\htheta$} 
over a class \smash{$\cK$}:
\begin{equation*}
	\Risk(\htheta) = \sup_{\theta_0 \in \cK} \frac{1}{n} \E \| \htheta - \theta_0 
	\|_2^2.
\end{equation*}
For a matrix \smash{$S \in \R^{n\times n}$} let $\Risk(S)$ also denote the risk 
of the linear smoother \smash{$\htheta = S y$}.
\begin{proof}[Proof of Theorem 
\ref{thm:ktv_minimax_linear}]
	For brevity, denote $ \genmat = \kronmatk$ and let $S$ 
	stand for a linear 
	smoother in the context of this 
	proof.
	The minimax linear risk for the class $\kronset^k(C_n)$ 
	is 
	\begin{align*}
		R_L( \kronset^k(C_n) )  &= 
		\inf_{S\in \R^{n\times n}} \sup_{\theta_0 \in  
		\kronset^k(C_n) } \; 
		\frac{1}{n}\E \|Sy-\theta_0\|_2^2\\
		&=
		\inf_{S} \sup_{\theta_0 \in  \kronset^k(C_n) } \; 
		\frac{1}{n} \E 
		\|S(\theta_0+\epsilon)-\theta_0\|_2^2\\
		&=
		\frac{1}{n}\inf_{S} \sup_{\theta_0 \in  
		\kronset^k(C_n) } \; \sigma^2 \|S\|_F^2 
		+ \|(S-I)\theta_0\|_2^2
	\end{align*}
	where in the last line we used the assumption that 
	$\epsilon_i,i\in[n]$ are 
	i.i.d. with mean zero and variance $\sigma^2$ and used 
	the notation $\|A\|_F$ 
	for the Frobenius norm of a matrix $A$.
	The infimum can be restricted to the set of linear 
	smoothers 
	$$
	\S = \left\{ S : \nul(S-I) \supseteq \nul(\genmat)  
	\right\}
	$$
	because if for a linear smoother $S$, if there exists 
	$\eta \in 
	\mathrm{null}(\genmat)$ such that $(S-I)\eta 
	\neq 0$, then the inner supremum above will be 
	$\infty$, that is, its risk will 
	be $\infty$. If the outer 
	infimum is over $\S$, then the supremum can be 
	restricted to $\{ \theta_0 \in 
	\mathrm{row}(\genmat) : 
	\theta \in \kronset^k(C_n) \}$. We continue to lower 
	bound minimax linear risk 
	as follows:
	\begin{align}
		\nonumber
		R_L( \kronset^k(C_n) )  &=
		\frac{1}{n}\inf_{S\in \S} \sigma^2 \|S\|_F^2 + 
		\sup_{\theta_0 \in 
			\mathrm{row}(\genmat) : \| \genmat \theta_0 \|_1 
			\leq C_n } \;   \|(S-I)\theta_0\|_2^2\\
		\nonumber
		&=\frac{1}{n}\inf_{S\in \S} \sigma^2 \|S\|_F^2 + 
		\sup_{z : \| z \|_1 \leq C_n } 
		\;   \|(S-I)\genmat^+ z\|_2^2\\
		\label{eq:R_S}
		&=
		\frac{1}{n}\inf_{S\in \S} \sigma^2 \|S\|_F^2 +   
		C_n^2  \max_{i\in[m]} \left\| 
		\left( (S-I)\genmat^+ \right)_i 
		\right\|_2^2\\
		\nonumber
		&\geq 
		\frac{1}{n}\inf_{S\in \S} \sigma^2 \|S\|_F^2 +    
		\frac{C_n^2 }{m} \sum_{i=1}^m 
		\left\| \left((S-I)\genmat^+ 
		\right)_i \right\|_2^2\\
		\label{eq:r_defn}
		&=
		\inf_{S\in \S} \underbrace{ 
			\frac{\sigma^2}{n} \|S\|_F^2 +    \frac{C_n^2 }{mn} 
			\left\| (S-I)\genmat^+  
			\right\|_F^2
		}_
		{ =:r(S)
		}
	\end{align}
	In the third line, $(A)_i$ denotes the $i$th column of 
	matrix $A$ and $m$ 
	denotes the number of 
	rows in $\genmat$. In the fourth line, we used the fact 
	that the maximum of a 
	set is at least as much as 
	their average. In the last line --- within the context 
	of this 
	proof --- we  define the quantity $r(S)$ which is a 
	lower bound on the risk of 
	a linear smoother $S\in 
	\S$.
	
	Notice that $r(\cdot)$ is a quadratic in the entries of 
	$S$ and the constraint 
	$S\in\S$ translates to 
	linear constraints on the entries of $S$. Writing the 
	KKT conditions, after 
	some work, we see that  
	$r(\cdot)$ is minimized at
	\begin{align}
		\label{eq:S_0}
		S_0 =  a_n \left( \sigma^2 L^{(k+1)} + a_n I 
		\right)^{-1} 
	\end{align}
	where we denote $a_n = \frac{C_n^2}{m}$ and $L^{(k+1)} 
	= \genmat^\T \genmat$ (the inverse is well defined because 
	\smash{$a_n > 0$}).
	Further, $S_0 \in \S$.
	Therefore, 
	\begin{align}
		\label{eq:RL_kron_lb_r}
		R_L( \kronset^k(C_n) )  &\geq r(S_0).
	\end{align}
	We simplify the expression for $r(S_0)$ now.
	Let $\lambda_i, i\in[n]$ be the eigenvalues of 
	$L^{(k+1)}$. Then the 
	eigenvalues of $S_0$ are
	$$
	\frac{a_n}{\sigma^2 \xi_i + a_n}, i\in [n]
	$$
	and the non-zero squared singular values of $(S_0 - 
	I)\genmat^+$ are given by
	$$
	\frac{\sigma^4 \xi_i}{(\sigma^2 \xi_i + 
	a_n)^2},  \quad \kappa < i \leq 
	n.
	$$
	Using the fact that the squared Frobenius norm of a 
	matrix is the sum of 
	squares of its singular values, substituting the above 
	eigenvalues and singular 
	values in \eqref{eq:r_defn}, we have
	\begin{align}
		\nonumber
		r(S_0) &= \frac{\sigma^2}{n} \sum_{i=1}^n   \left( 
		\frac{a_n}{\sigma^2 
			\xi_i + a_n} \right)^2 
		+ \frac{a_n}{n} \sum_{i=1}^n \frac{\sigma^4 
		\xi_i}{(\sigma^2 \xi_i + 
			a_n)^2}\\
		\label{eq:r_S0}
		&= 
		\frac{1}{n}\sum_{i=1}^n \frac{\sigma^2 a_n}{\sigma^2 
		\xi_i + a_n}.
	\end{align}
	
	\noindent
	Now we upper bound the risk $\Risk(S_0)$ of the linear 
	smoother defined by 
	$S_0$. From
	\eqref{eq:R_S}, we can write 
	\begin{align*}
		\Risk(S_0) 
		&= \frac{\sigma^2}{n}  \|S_0\|_F^2 +   
		\frac{C_n^2}{n}  \max_{i\in[m]} \left\| 
		\left( (S_0-I)\genmat^+ \right)_i \right\|_2^2.
	\end{align*}
	Let $\genmat = U \Sigma V^\T$ be the singular value 
	decomposition of $\genmat$. 
	Also let the eigen-decomposition of $S_0 - I = V 
	\Lambda V^\T$. Then using 
	incoherence of columns of $U$, that is, the fact that 
	there exists a constant 
	$c>1$ that depends only on $k,d$ such that $U_{ij}^2 
	\leq \frac{c}{m}$ for all 
	$i\in[m],j\in[n]$, we can write
	\begin{align*}
		\max_{i\in[m]} \left\| \left( (S_0-I)\genmat^+ 
		\right)_i \right\|_2^2
		&=
		\max_{i\in[m]} \left\| V\Lambda V^\T V \Sigma^+ 
		(U^\T)_i \right\|_2^2 \\
		&=
		\max_{i\in[m]}\; (U^\T)_i^\T  (\Lambda \Sigma^+)^2 
		(U^\T)_i \\
		&\leq \frac{c}{m} \mathrm{tr}\left( (\Lambda 
		\Sigma^+)^2 \right) \\
		&= \frac{c}{m} \sum_{i=1}^n \frac{\sigma^4 
		\xi_i}{(\sigma^2 \xi_i + 
			a_n)^2}.
	\end{align*}
	Plugging this back in the previous display and also 
	using the fact that the 
	squared Frobenius norm of a matrix is equal to the sum 
	of the squares of its 
	eigenvalues,
	\begin{align*}
		\Risk(S_0) 
		&= 
		\frac{\sigma^2}{n} \sum_{i=1}^n   \left( 
		\frac{a_n}{\sigma^2 \xi_i + a_n} 
		\right)^2 
		+ \frac{c \cdot a_n}{n} \sum_{i=1}^n \frac{\sigma^4 
		\xi_i}{(\sigma^2 
			\xi_i + a_n)^2}\\
		&\leq c \cdot r(S_0)
	\end{align*}
	Combining this with the lower bound in 
	\eqref{eq:RL_kron_lb_r}, we have
	\begin{align}
		r(S_0) \leq R_L(\kronset^k(C_n) ) \leq \min \left\{ 
		\sigma^2, \Risk(S_0) 
		\right\}  \leq \min \left\{ 
		\sigma^2, \;\;c\cdot r(S_0) \right\}.
	\end{align}
	In other words, the minimax linear rate is essentially 
	$r(S_0)$ up to a 
	constant factor. Further, one of the estimators 
	$\hat{y} = S_0 y$, $\hat{y} = 
	y$ achieves the minimax linear rate up to a constant 
	factor.
	
	Now we bound $r(S_0).$ Let $\kappa=(k+1)^d$ denote the 
	nullity of $\genmat$. 
	Recall from \eqref{eq:r_S0}
	\begin{align}
		\label{eq:rS0_lb0}
		r(S_0) = 
		\frac{1}{n}\sum_{i=1}^n \frac{\sigma^2 a_n}{\sigma^2 
		\xi_i + a_n} = 
		\frac{\kappa \sigma^2}{n}  + 
		\frac{1}{n}\sum_{i=\kappa+1}^n \frac{\sigma^2 
		a_n}{\sigma^2 \xi_i + a_n}.
	\end{align}
	
	\paragraph{Lower bounding $r(S_0)$.} 
	We give three lower bounds on $r(S_0)$. 
	By using the fact that arithmetic mean of positive 
	numbers is at least as large 
	as their harmonic mean, we have
	\begin{align}
		\nonumber
		r(S_0) 
		&=  \frac{1}{n}\sum_{i=1}^n \frac{\sigma^2 
		a_n}{\sigma^2 \xi_i + a_n}\\
		\nonumber
		&\geq 
		\frac{n \sigma^2 a_n}{\sum_{i=1}^n (\sigma^2 
		\xi_i + a_n) } \\
		\nonumber
		&= 
		\frac{n \sigma^2 a_n}{na_n + \sigma^2 \|\genmat\|_F^2 
		} \\
		\nonumber
		&= 
		\frac{n\sigma^2 a_n}{na_n + \sigma^2 d n^{1-1/d} 
			\|D_{\mathrm{1d}}^{(k+1)}\|_F^2  } \\
		\nonumber
		&= 
		\frac{\sigma^2 a_n}{a_n + \sigma^2 d n^{-1/d} 
		(n^{1/d}-k-1) { 
				{2k+2}\choose{k+1} }  }\\
		\label{eq:rS0_lb1}
		&\geq 
		\frac{\sigma^2 a_n}{a_n + \sigma^2 d 4^{k+1} }
	\end{align}

\noindent
	Now we bound in $r(S_0)$ in a second way. Let $n_1$ be the 
	cardinality of $\{ i \in [n]: 
	\sigma^2 \xi_i \leq a_n \}$. Then 
	\begin{align*}
		r(S_0) 
		=  \frac{1}{n}\sum_{i=1}^n \frac{\sigma^2 
		a_n}{\sigma^2 \xi_i + a_n} 
		\geq 
		\frac{1}{n}\sum_{i=1}^{n_1} \frac{\sigma^2 a_n}{a_n + 
		a_n} 
		= \frac{n_1 \sigma^2}{2n}.
	\end{align*}
	Note that $n_1 = \lfloor n F(a_n/\sigma^2) \rfloor$ 
	where $F$ is the spectral distribution of 
	$(\kronmatk)^\T\kronmatk$ defined in 
	Lemma~\ref{lem:spectral_dist_kron}. Applying 
	Lemma~\ref{lem:spectral_dist_kron}, we get 
	\begin{align}
		\nonumber
		r(S_0) &\geq \frac{\sigma^2}{2} \left( 
		F\left(\frac{a_n}{\sigma^2} \right) - 
		\frac{1}{n} \right) \\
		\nonumber
		&\geq c  \sigma^2 
		\min \big\{1, (a_n/\sigma^2)^{\frac{1}{2s}} \big\} - \sigma^2/2n\\
		\label{eq:rS0_lb2}
		&= \min \big\{ c \sigma^2,  c \sigma^{2 - \frac{1}{s}} a_n^{\frac{1}{2s}} - 
		\sigma^2/2n \big\}
	\end{align}

\noindent
In the special case $s = 1/2$, from Lemma~\ref{lem:ktf_inv_eigval_sum_lower_bd} 
we get a third bound:
	\begin{align}
\label{eq:rS0_lb3}
	r(S_0) 
	&=  \frac{1}{n}\sum_{i=1}^n \frac{\sigma^2 
		a_n}{\sigma^2 \xi_i + a_n} \geq c_1 a_n \log \big( 1 + c_2 / a_n)
\end{align}
where $c_1, c_2$ constants that depend only on $k, d$.

	From \eqref{eq:rS0_lb0},\eqref{eq:rS0_lb1},  	\eqref{eq:rS0_lb2} and 
	\eqref{eq:rS0_lb3} we have the 
	lower bound 
	\begin{align}
		\label{eq:rS0_lb}
		r(S_0) \geq \max
		\left\{ 
		\frac{\kappa \sigma^2}{n},  
		\frac{\sigma^2 a_n}{a_n + \sigma^2 d 2^{2k+2} },  
		\sigma^2 \wedge c \sigma^{2 - \frac{1}{s}} a_n^{\frac{1}{2s}}  - 
		\frac{\sigma^2}{2n}
		\right\}.
	\end{align}
and an additional lower bound of 
\smash{$c_1 a_n \log \big( 1 + c_2 / a_n)$} 
in the case $s = 1/2$.
	Substituting $a_n = C_n^2 / m$, using the assumption that
	\smash{$C_n^2/n \leq 1$} and treating $k, d, \sigma$ as constants , we 
	get the stated lower bound.
	
	\paragraph{Upper bounding $r(S_0).$} 
	If $s < 1/2$, then 
	\begin{align}
		\nonumber
		r(S_0) &= \frac{1}{n} \sum_{i=1}^n \frac{\sigma^2 
		a_n}{\sigma^2 \xi_i + 
			a_n} \\
		\nonumber
		&\leq \frac{\kappa \sigma^2}{n} + \frac{1}{n} 
		\sum_{i=\kappa+1}^n 
		\frac{\sigma^2 a_n}{\sigma^2 \xi_i } \\
		\nonumber
		&= \frac{\kappa \sigma^2}{n} + \frac{a_n}{n} 
		\sum_{i=1}^{\kappa+1} 
		\frac{1}{\xi_i} \\
		\nonumber
		&\leq 
		\frac{\kappa \sigma^2}{n} + \frac{a_n}{n} (c_3 n)\\
		\label{eq:rS0_ub1}
		&= 
		\frac{\kappa \sigma^2}{n} + c_3 a_n 
	\end{align}
	We used Lemma~\ref{lem:kron_tf_evals} to control the 
	second term in the third 
	line.
Similarly, if $s = 1/2$, 
	\smash{$r(S_0) \leq \kappa \sigma^2/n + 
	c_3 a_n \log n$}. For the case $s > 1/2$, we can write
	\begin{align}
		\nonumber
		r(S_0) &= \frac{1}{n} \sum_{i=1}^n \frac{\sigma^2 
		a_n}{\sigma^2 \xi_i + 
			a_n} \\
		\nonumber
		&\leq  \frac{1}{n}\sum_{i=1}^{n_1} \frac{\sigma^2 
		a_n}{a_n} + 
		\frac{1}{n}\sum_{i=n_1+1}^{n} \frac{\sigma^2 
		a_n}{2\sigma^2 \xi_i } \\
		\nonumber
		&= \frac{n_1 \sigma^2}{n} + \frac{a_n}{2n} 
		\sum_{i=n_1+1}^{n} 
		\frac{1}{\xi_i}\\
		\nonumber
		&\leq
		c \frac{ \sigma^2}{n}  + c \sigma^2 
		\left(\frac{a_n}{\sigma^2}\right)^{\frac{1}{2s}}
		+ c \frac{a_n}{2n}  n^{2s}   \left(n 
		(a_n/\sigma^2)^{\frac{1}{2s}} \right)
		^{1-2s}\\
		\label{eq:rS0_ub2}
		&\leq 
		c \frac{ \sigma^2}{n}  + c \sigma^{2 - \frac{1}{s}} 
		a_n^{\frac{1}{2s}} 
	\end{align}
	To get the fourth line, we used 
	Lemma~\ref{lem:spectral_dist_kron} to bound 
	$n_1$ and  
	Lemma~\ref{lem:kron_tf_evals} to bound the summation.
	
	\paragraph{Upper bound with the polynomial projection estimator 
	$\htheta^{\mathrm{poly}}$.}
	For brevity, let $\Pi$ denote the matrix that projects on to the null space 
	of $\genmat$. Note 
	that $(I-\Pi) \genmat^+ = 
	\genmat^+$. From bias variance decomposition similar to 
	that in \eqref{eq:R_S}, 
	\begin{align*}
		\frac 1n \sup_{\theta_0 \in \kronset^k(C_n)} 
		\E \big[ \| \htheta^{\mathrm{poly}} - \theta_0 \|_2^2\big]
		&= \frac{\sigma^2}{n} \| \Pi \|_F^2 + \max_{i\in[m]} 
		\| \big( (\Pi - I) 
		\genmat^+\big)_i \|_2^2 \\
		&= \frac{\kappa \sigma^2}{n}  + \max_{i\in[m]} \| 
		\genmat^+_i \|_2^2
	\end{align*}
	Then using incoherence of columns of $U$, that is, the 
	fact that there exists a 
	constant $c>1$ that depends only on $k,d$ such that 
	$U_{ij}^2 \leq \frac{c}{m}$ 
	for all $i\in[m],j\in[n]$, we can write
	\begin{align*}
		\max_{i\in[m]} \| \genmat^+_i \|_2^2
		&=
		\max_{i\in[m]} \left\| V \Sigma^+ (U^\T)_i 
		\right\|_2^2 \\
		&=
		\max_{i\in[m]}\; (U^\T)_i^\T  ( \Sigma^+)^2 (U^\T)_i \\
		&\leq \frac{c}{m} \mathrm{tr}\left( ( \Sigma^+)^2 
		\right) \\
		&= \frac{c}{m} \sum_{i=\kappa+1}^n \frac{1}{\xi_i}
	\end{align*}
	Plugging this back in the above display and using the 
	bound on 
	$\sum_{i=\kappa+1}^n \frac{1}{\xi_i}$ from 
	Lemma~\ref{lem:kron_tf_evals}, 
	we get the desired result.

\paragraph{Upper bound with the projection estimator \eqref{eq:eigenmaps} when 
$s > 1/2$.}
From \eqref{eq:R_S}, 
for the projection estimator \smash{$\htheta = S_Q y$} in \eqref{eq:eigenmaps},
\begin{equation}
	\label{eq:risk_S_Q}
	\Risk(\htheta) =
	\frac{\sigma^2}{n} |Q| +
	\frac{C_n^2}{n}  \max_{i\in[m]} \left\| \left( (S_Q -I)\genmat^+ \right)_i 
	\right\|_2^2.
\end{equation}
Set \smash{$Q = [\tau]^d$} for a $\tau$ to be 
chosen later from $(k+2, N]$.
Also write \smash{$S_Q - I = V \Lambda_Q V^\T$}.
Again using incoherence of columns of $U$, we can write
\begin{align*}
	\max_{i\in[m]} \left\| \left( (S_Q-I)\genmat^+ 
	\right)_i \right\|_2^2
	&=
	\max_{i\in[m]} \left\| V\Lambda_Q V^\T V \Sigma^+ 
	(U^\T)_i \right\|_2^2 \\
	&=
	\max_{i\in[m]}\; (U^\T)_i^\T  (\Lambda_Q \Sigma^+)^2 
	(U^\T)_i \\
	&\leq \frac{c}{m} \mathrm{tr}\left( (\Lambda_Q 
	\Sigma^+)^2 \right) \\
	&= \frac{c}{m} \sum_{i \in [N]^d \setminus Q} \frac{1}{\xi_i}
\end{align*}
The summation in the last line can be bound using Lemma~\ref{lem:kron_tf_evals} 
(recall $s>1/2$ here):
\begin{align*}
	\sum_{i \in [N]^d \setminus Q} \frac{1}{\xi_i} \leq 
	\sum_{\| (i-k-2)_+ \|_2 \geq \tau-k-2} \frac{1}{\xi_i} 
	\leq 
	c n (n/ (\tau - k - 2)^d )^{2s-1}
\end{align*}
Tracing this back to \eqref{eq:risk_S_Q},
\begin{equation*}
	\Risk(\htheta) \leq 
	\frac{\sigma^2}{n} \tau^d + \frac{c C_n^2}{m} \cdot (n/ (\tau - k - 2)^d 
	)^{2s-1}
\end{equation*}
Minimize this bound by setting $\tau$ such that \smash{$\tau^d \asymp (C_n / 
\sigma)^{\frac{1}{s}} 
	n^{1-\frac{1}{2s}}$} to get the desired bound.
\end{proof}
\begin{remark}
In Theorem~\ref{thm:ktv_minimax_linear}, in the case  
\smash{$s \leq 1/2$}, the lower 
bound may also be obtained by embedding 
the \smash{$\ell_1$}-ball \smash{$B_1(C_n/(2^{k+1}d))$} 
into 
\smash{$\kronset^k(C_n)$}.
\end{remark}

\subsection{Proof of Theorem \ref{thm:sobolev_minimax}}
\label{app:pf_thm_le_ls}

\paragraph{Proof of upper bound.}
Like in the proof of minimax linear rates for KTV class in 
Theorem~\ref{thm:ktv_minimax_linear}, 
for the projection estimator \smash{$\htheta = S_Q y$} where 
\smash{$S_Q = V_Q V_Q^\T$}, we can derive
\begin{equation*}
	\frac 1n \sup_{\theta_0 \in \sobolset^{k+1}(B_n)} \E\big[ \|\htheta - 
	\theta_0 
	\|_2^2 \big]
	=
	\frac{\sigma^2}{n} |Q| +
	\frac{1}{n} \sup_{\theta_0 \in \sobolset^{k+1}(B_n)} \| (I - S_Q) \theta_0 
	\|_2^2.
\end{equation*}
Denote \smash{$\genmat = \kronmatk$} for brevity.
Set \smash{$Q=[\tau]^d$}, where \smash{$\tau \in (k+2, N]$} is an 
integer (recall \smash{$N=n^{1/d}$}) and analyze the 
maximum of the
second term:
\begin{align*}
	\sup_{\theta_0 : \|\genmat\theta_0\|_2 \leq B_n}
	\frac{1}{n} \| (I - S_q) \theta_0 \|_2^2
	&= \sup_{z : \| z \|_2 \leq C_n}
	\frac{1}{n} \| (I - S_q) \genmat^\dagger z \|_2^2 \\
	&= \frac{B_n^2}{n}
	\sigma^2_{\max} \big((I-S_q) \genmat^\dagger \big) \\
	&\leq \frac{B_n^2}{n}
	\frac{1}{4^{k+1} \sin^{2k+2}( \pi (\tau-k-2)/ (2N))} \\
	&\leq \frac{B_n^2}{n} 
	\frac{N^{2k+2}}{(\pi(\tau-k-2))^{2k+2}}.
\end{align*}
Here we denote by \smash{$\sigma_{\max}(A)$} the maximum
singular value of a matrix $A$.
The last inequality above used the inequality
$\sin(x) \geq x/2$ for $x \in [0,\pi/2]$.
The earlier inequality used that
\smash{$\sigma^2_{\max}((I - S_Q) 
\genmat^\dagger)$} is the
reciprocal of the smallest eigenvalue \smash{$\rho_Q$} 
of \smash{$M = D^\T D$}
with index in \smash{$[N]^d \setminus Q$}. That is,
\begin{align*}
	\rho_{Q} 
	= \rho_{\tau+1,1,\dots,1}
	\geq \big( 4\sin^2(\pi (\tau-k-2) / (2N)) \big)^{k+1},
\end{align*}
where the last inequality is due to the relation in 
\eqref{eq:rho_interlacing}.
Hence, we have established
\begin{equation*}
	\sup_{\theta_0 : \|\genmat\theta_0\|_2 \leq B_n}
	\frac 1n \E\big[ \|\htheta - \theta_0 \|_2^2 \big]
	\leq
	\frac{\sigma^2}{n} \tau^d +
	\frac{B_n^2}{n} \frac{N^{2k+2}}{ (\pi 
	(\tau-k-2))^{2k+2}}.
\end{equation*}
Choosing $\tau$ to balance the two terms on the 
right-hand side
above results in
\smash{$\tau^d \asymp 
	(k+2)^d+ \big(B_n^2 n^2s / \sigma^2 \big)^{\frac{1}{2s+1}}$}.
Also, in the edge case where $Q=[N]^d$, 
the risk is $\sigma^2$.
Plugging this choice of $\tau$ gives the upper bound result.

\paragraph{Proof of lower bound.}
Similar to argument in the proof of Theorem 
\ref{thm:ktv_lower_bd}, the nullity 
of 
\smash{$\kronmatk$} 
implies the lower bound
\begin{align}
	\label{eq:sobolev_lower_bd_nuli}
	R( \sobolset^{k+1}(C_n))  = \Omega\big( \frac{\kappa 
	\sigma^2}{n} \big).
\end{align}
The Holder ball embedding
$$
\sobolset^{k+1}(C_n) \supseteq \holderset^{k}(cC_n 
n^{s-\half} )
$$
implies that
$$
R(\sobolset^{k+1}(C_n) ) \geq R \big(\holderset^{k}(cC_n 
n^{s-\half} ) \big) =
\Omega \bigg( \frac{C_n^2}{n} \bigg)^{\frac{1}{2s+1}} 
\sigma^{\frac{4s}{2s+1}}\wedge \sigma^2,
$$
where the second step follows from 
\eqref{eq:holder_lower_bd}.
Putting these two bounds together, we get the desired 
lower bound.

\section{Estimation theory for graph trend filtering on grids}
\label{app:gtf}

We recall the GTF operator from \cite{wang2016trend} for convenience.
Let $G(V,E)$ be a graph with $n$ vertices and $m$ edges $(u_1, v_1), \dots, 
(u_m, v_m) \in [n] \times [n]$.  Assume that $u_i < v_i$ for $i\in[m]$
 in the edges here for notational convenience.
Let $D \in 
\R^{m\times n}$ be the 
incidence matrix of $G$ satisfying
$$
(Dx)_j = x_{u_j} - x_{v_j} \quad \text{ for all } x \in \R^n
$$
for all edges $(u_j, v_j)$ for $j\in [m]$. The graph Laplacian is \smash{$L = 
D^\T D$} 
The GTF operators of all orders are defined by
\begin{align}
	\graphmat^{(1)} = D, 	\quad &\graphmat^{(2)} = L, \nonumber\\
	\label{eq:gtf}
	\graphmat^{(2k+1)} = DL^{k}, \quad &\graphmat^{(2k)} = L^{k} \text{ for 
} k\ge 0, k\in \mathbb{Z}.
\end{align}

\subsection{Upper bounds on estimation risk}

\cite{wang2016trend} used Theorem
\ref{thm:genlasso_upper_bd} (their Theorem 6) in order to derive error 
rates for GTF on 2d grids already; see their Corollary 8. 
\cite{sadhanala2017higher} refine this result 
using a  tighter upper bound for the partial sum of inverse eigenvalues.
Here, we give a more general  result that applies
to not just 2d grids, but all $d\geq 2$ and $k\geq 0$.
We further show that these rates are optimal by deriving a matching lower bound.
Recall the abbreviation $s = (k+1)/d$.
\begin{theorem}
	\label{thm:graph_tf_upper_bd}
	Assume that $d\geq 1$ and $k\ge 0$.
	Denote
	\smash{$C_n=\|\graphmatk \theta_0\|_1$}.
	Then GTF defined by the estimator in \eqref{eq:genlasso_estimator} with 
$D=\graphmatk$ in 
	\eqref{eq:gtf} 
	satisfies 
	$$
	\frac 1n \| \htheta - \theta_0 \|_2^2  =  O_\P \bigg(
	\frac{1}{n} + \frac{\lambda}{n} C_n \bigg)
	$$
	with 
	\begin{align*}
		\lambda \asymp 
		\begin{cases}
			\sqrt{\log n} & s < 1/2\\
			\log n & s = 1/2\\
			(\log n)^{\frac{1}{2s+1}}  \big( \frac{n}{C_n} \big)^{ 
\frac{2s-1}{2s+1}} & s > 1/2.
		\end{cases}
	\end{align*}
\end{theorem}

With canonical scaling of $C_n$, we see the following error bound.
\begin{corollary}
	\label{cor:graph_tf_upper_bd_canonical}
	With canonical scaling $C_n = C_n^* = n^{1-s}$, the GTF estimator with  
$\lambda$ scaling as in 
	Theorem~\ref{thm:graph_tf_upper_bd} satisfies
	\begin{align*}
		\sup_{\theta_0 \in \graphset^k(C_n)} 
			\frac 1n \| \htheta - \theta_0 \|_2^2 =
		\begin{cases}
			O_\P \left( n^{-s} \sqrt{\log n}\right) & s < 1/2\\
			O_\P \left( n^{-s} \log n\right) & s = 1/2\\
			O_\P \left( n^{-\frac{2s}{2s+1}} (\log 
n)^{\frac{1}{2s+1}}\right) & s > 1/2.
		\end{cases}
	\end{align*}
\end{corollary} 
Remarks following Theorem~\ref{thm:ktf_upper_bd} for KTF apply for GTF as well.
The proof is in Appendix~\ref{sec:pf_graph_tf_upper_bd}.

\subsection{Lower bounds on estimation risk}

Similar to the lower bound in Theorem~\ref{thm:ktv_lower_bd} for KTV class, we 
give a bound for the graph total variation (GTV) class
\begin{align}
	\label{eq:gtv_class}
	\graphset^k(C_n) &= \{\theta \in \R^n : \|\graphmatk\theta\|_1\leq 
C_n\}.
\end{align}
Due to the lower order discrete derivatives on the boundary of
the grid $Z_{n,d}$, the GTV class \smash{$\graphset^k(C_n)$} cannot contain 
the discrete Holder 
class with 
appropriate scaling
\smash{$\holderset^{k}(C_n n^{s - 1})$}; 
see Lemma 4 in \citet{sadhanala2017higher}.
However, by an alternative route \citet{sadhanala2017higher} show a lower bound 
for \smash{$\graphset^k(C_n)$}  that matches with the lower bound for the Holder
class \smash{$\holderset^{k}(C_n n^{s- 1})$}. 
We further tighten their result by embedding an $\ell_1$ ball of appropriate 
size.
\begin{theorem}
	\label{thm:graph_tf_lower_bd}
	For any integers $k \geq 0$, $d \geq 1$, the minimax estimation error 
  for the GTV class	defined in \eqref{eq:gtv_class} satisfies
	\begin{align*}
		R\big(\graphset^k(C_n)\big) &= \Omega \bigg(  
\frac{\sigma^2}{n}  + \frac{\sigma C_n}{n} + 
		\left(\frac{C_n}{n} \right)^{\frac{2}{2s+1}} 
\sigma^{\frac{4s}{2s+1}} \wedge \sigma^2 \bigg).
	\end{align*}
\end{theorem}
\begin{proof}[Proof of Theorem~\ref{thm:graph_tf_lower_bd}]
Similar to the proof of Theorem~\ref{thm:ktv_lower_bd}, it is sufficient 
to show three  lower 
bounds separately. We get the first two lower bounds just as in the proof of  
Theorem~\ref{thm:ktv_lower_bd} by using the fact that $\nuli \big( 
\graphmatk \big)=1$  and 
the $\ell_1$-ball embedding
$$
B_1(C_n/(2^{k+1} d)) \subseteq \graphset^k(C_n)
$$
from Lemma~\ref{lem:l1ball_embed_Dset} and the fact that 
\smash{$\| \graphmatk \|_{1,\infty} \leq 2^{k+1} d$}.
The third term is from Theorem 5 in 
\cite{sadhanala2017higher}.
\end{proof}

\subsection{Minimax rates for linear smoothers}

The minimax linear rate analysis for GTV class is very similar to that for 
KTV class. So we simply 
state the result and skip the proof.

\begin{theorem}
	\label{thm:gtf_minimax_linear}
	The minimax linear risk over the GTV class in \eqref{eq:gtv_class} 
	satisfies,
	for any sequence \smash{$C_n \leq \sqrt{n}$},  
	\begin{equation}
		\label{eq:gtv_lower_bd_linear}
		R_L\big( \graphset^k(C_n) \big) = 
		\begin{cases}
			\Omega (1/n + C_n^2/n) & \text{if $s < 1/2$}, \\ 
			\Omega \big( 1/n + C_n^2/n \log(1 + n / C_n^2) \big) & 
			\text{if $s = 1/2$}, \\    
			\Omega \big( 1/n + (C_n^2/n)^{\frac{1}{2s}} \big) & 
			\text{if $s > 1/2$}.    
		\end{cases}
	\end{equation}
	This is achieved in rate by the projection estimator in 
	\eqref{eq:eigenmaps},   
	by setting $Q = [\tau]^d$ for \smash{$\tau^d \asymp (C_n n^{s-1/2})^{1/s}$}, 
	in the case $s > 1/2$. When $s < 1/2$, the simple mean
	estimator, \smash{$\hat{\theta}^{\mathrm{mean}} = \bar y \one$},
	achieves the rate in \eqref{eq:gtv_lower_bd_linear}. When $s=1/2$, this
	estimator achieves the rate in \eqref{eq:gtv_lower_bd_linear} up to a log
	factor. Lastly, if \smash{$C_n^2 = O(n^\alpha)$} for $\alpha < 1$, and still 
	$s=1/2$, then the mean estimator achieves the rate in
	\eqref{eq:gtv_lower_bd_linear} without the additional log factor. 
\end{theorem}

\subsection{Proof of Theorem \ref{thm:graph_tf_upper_bd}}
\label{sec:pf_graph_tf_upper_bd}

For $d=2$, it is shown in the proof of Corollary 8 in \citet{wang2016trend} 
that 
the GTF operator \smash{$\graphmatk$} satisfies the incoherence property, as
defined in Theorem \ref{thm:genlasso_upper_bd}, with a constant $\mu=4$ when 
$k$ is even and 
$\mu=2$ when $k$ is
odd.  
Here we extend this incoherence property for $d> 2$ using
Lemma~\ref{lem:incoherence_kron}. We treat the cases where $k$ is odd and 
even separately.

If $k$ is odd we can extend the argument from Corollary 8 in 
\citet{wang2016trend} in a 
straightforward manner.  The GTF operator is \smash{$\graphmatk = L^{(k+1)/2}$}
where $L$ is the Laplacian of the
$d$-dimensional grid graph. Denoting the Laplacian of the chain graph of length 
$N$ by $L_{\mathrm{1d}}$, we note that $L$ is given by
$$
L = L_{\mathrm{1d}} \otimes I \otimes I +
I \otimes  L_{\mathrm{1d}}  \otimes I +
I  \otimes I \otimes  L_{\mathrm{1d}}
$$
for $d=3$ and 
$$
L = L_{\mathrm{1d}} \otimes  I \dots \otimes I + 
I \otimes  L_{\mathrm{1d}}   \dots \otimes I + 
\dots +
I \otimes \dots I \otimes  L_{\mathrm{1d}}
$$
for general $d$ where each term in the summation is a Kronecker product of $d$
matrices. Let $\alpha_i, u_i, i\in [N]$ be the eigenvalues and eigenvectors of  
\smash{$L_{\mathrm{1d}}$}. As shown in 
\citet{wang2016trend}, in 1d, we have the incoherence property $ \| u_i 
\|_\infty \leq \sqrt{2/N}$ for 
all $i\in[N]$.  
The eigenvalues of $L$ are \smash{$\sum_{j=1}^d \alpha_{i_j}$} and the
corresponding eigenvectors are $u_{i_1} \otimes \dots \otimes u_{i_d}$ for $i_1,
\dots,i_d\in [N]$. Clearly, incoherence holds for 
the eigenvectors of $L$ with constant $\mu = 2^{d/2}$.

If $k$ is even, then the left singular vectors of \smash{$\graphmatk$} are the 
same as those of \smash{$\graphmat^{(1)}$}. We know that both the left and right
singular vectors of \smash{$D_{\mathrm{1d}}^{(1)}$} satisfy the incoherence
property with constant \smash{$\mu=\sqrt{2}$} (see the proof of Corollary 7 in
\citet{wang2016trend}). Setting \smash{$D=D_{\mathrm{1d}}^{(1)}$} in
Lemma~\ref{lem:incoherence_kron}, we see that the left singular vectors of
$\graphmat^{(1)}$ and hence those of \smash{$\graphmat^{(k+1)}$} satisfy 
incoherence property with constant $2^{d/2}$. Therefore, for all integers 
$k\geq 0$, the left singular vectors of $\graphmatk$ are incoherent with
constant \smash{$2^{d/2}$}. 

From the incoherence property and Theorem \ref{thm:genlasso_upper_bd}, the GTF 
estimator \smash{$\htheta$}, satisfies 
\begin{align}
	\label{eq:genlasso_rate_i0}
	\frac 1n \| \htheta - \theta_0 \|_2^2  =  O_\P \Bigg( 
	\frac{1}{n} + \frac{|I|}{n} + 
	\frac{\mu}{n} \sqrt{ \frac{\log n}{n} 
		\sum_{i \in [N]^d \setminus (I \cup \{1\}^d) } 
\frac{1}{\rho_i^2}} \cdot \| \Delta \theta_0 \|_1\, \Bigg),
\end{align}
where \smash{$\rho_i, i\in[N]^d$} are the eigenvalues \smash{${\graphmatk}^\T 
\graphmatk$} and
\smash{$\mu = 2^{d/2}$}.

Consider the set \smash{$I = \{ i \in [N]^d : \| i - 1 \|_2 < r\} \setminus 
\{1\}^d $ } for an $r 
\in [1, \sqrt{d}N]$ chosen later.
Lemma \ref{lem:lap_eigs_ddim} gives the key calculation
where it is
shown that for large enough $n$,
\begin{align*}
	\sum_{\|i - 1\|\geq r} \frac{1}{\rho_i^2} 
	= \sum_{\|i - 1\|\geq r} \frac{1}{\lambda_i^{k+1}} \leq c
	\begin{cases}
		n & s<1/2\\
		n \log (2\sqrt{d} N / r) & s=1/2\\
		n  (n/ r^d)^{2s-1 }   & s > 1/2
	\end{cases}
\end{align*}
where \smash{$\lambda_i, i\in [N]^d$} are eigenvalues of the Laplacian $L$ and 
$c>0$ is a constant that 
depends only on $k, d$ .

For $s \leq 1/2$, to minimize
the upper bound in \eqref{eq:genlasso_rate_i0}, set $r=1$ so that $I$ is empty 
and apply the above 
inequality.
This gives the desired bound.
Now consider $s>1/2$. Note that $| I | \leq r^d$ because $I \subseteq [r]^d$. 
Therefore \eqref{eq:genlasso_rate_i0} reduces to
\begin{equation}
	\label{eq:genlasso_rate_i0_smooth}
	\frac 1n \| \htheta - \theta_0 \|_2^2  =  O_\P \Big( \frac{r^d}{n} + 
\frac{\mu}{n}
	\sqrt{\log n (n/r^d)^{2s-1}} \| \Delta \theta_0 \|_1 \Big)	
\end{equation}
To minimize the upper 
bound in \eqref{eq:genlasso_rate_i0_smooth} balance 
$$
r^d \quad\text{with}\quad \frac{C_n}{\sqrt{n}} \sqrt{  n  (n/r^d)^{2s-1 } \log 
n}.
$$
This leads us to take $$r^d \asymp  (C_n\sqrt{\log n})^{\frac{2}{2s+1}} 
n^{\frac{2s-1}{2s+1}}$$
and plugging this in \eqref{eq:genlasso_rate_i0_smooth} gives the desired bound 
for $s>1/2$.
This completes the proof.
\hfill\qedsymbol

\section{Technical lemmas}

\begin{lemma}
	\label{lem:lap_eigs_ddim}
	Consider the eigenvalues $\{ \lambda_{i} : i = 
	(i_1,\cdots,i_d) \in [N]^d \}$ 
	of the $d$-dimensional 
	grid 
	graph Laplacian with $n=N^d$ nodes. Let $k$ be a 
	non-negative integer and $r_0 
	\in [1, 
	\sqrt{d}N]$. Then,
	$$
	\sum_{ i \in [N]^d : \|i - 1\|_2^2 \geq r_0^2 }  
	\frac{1}{\lambda_i^{k}} 
	\leq c
	\begin{cases}
		n & 2 k < d\\
		n \log (2\sqrt{d}N/r_0) & 2k = d\\
		N^{2k} r_0^{d-2k}  & 2k > d
	\end{cases}
	$$
	for a constant $c>0$ that depends on $k, d$ but not on 
	$N, r_0$.
\end{lemma}

\begin{proof}[Proof of Lemma \ref{lem:lap_eigs_ddim}]
	Let $I$ denote the summation on the left. Then
	\begin{align}
		\nonumber
		I = \sum_{ i \in [N]^d : \|i - 1\|_2 \geq r_0 }  
		\frac{1}{\lambda_i^{k}} 
		&= 
		\sum_{\|i - 1\|_2 \geq r_0 }  
		\Big( \sum_{j=1}^d 4 \sin^2 \frac{\pi(i_j-1)}{2N} 
		\Big)^{-k} \\
		\nonumber
		&\leq 
		\sum_{\|i - 1\|_2 \geq r_0 }  
		\Big( \sum_{j=1}^d  \frac{\pi^2(i_j-1)^2}{4N^2} 
		\Big)^{-k} \\
		\nonumber
		&=
		c N^{2k} \sum_{\|i - 1\|_2 \geq r_0 }
		\Big( \sum_{j=1}^d (i_j-1)^2 \Big)^{-k}\\
		\label{eq:eigval_bd_I}
		&\leq 
		c N^{2k} \sum_{i\in \{0,1, .., N-1\}^d : \|i \|_2 
		\geq r_0 }  
		\| i \|_2^{-2k}
	\end{align}
	In the second line, we used the fact that $\sin x \geq 
	x/2$ for $x\in 
	[0,\pi/2]$.
	
	\paragraph{Case \smash{$r_0 \ge 2\sqrt{d}$}.}
	In the last expression, upper bound 
	\smash{$\| i \|_2^{-2k}$} with the integral of 
	\smash{$f(x) =\| x \|_2^{-2k}, f:\R^d \rightarrow \R$}
	over the unit length cube whose top right corner is at 
	$i$.
	Note that, the norm of any point in this cube is at 
	least $\| i - \one\|_2 \geq 
	\|i\|_2 - \| \one 
	\|_2 = \|i\|_2 - \sqrt{d} \geq r_0 - \sqrt{d} \ge 
	r_0/2.$ 
	Therefore, we can continue to bound
	\begin{align*}
		I &\leq 
		c N^{2k} \int_{r_0/2 \leq \|x\|_2 \leq \sqrt{d} N} 
		\|x \|_2^{-2k}  \;dx 
		\\
		&\leq c N^{2k}
		\int_{r_0/2 \leq r \leq \sqrt{d} N} (r^2)^{-k} 
		r^{d-1} \; dr
	\end{align*}
	The last line is obtained by changing to polar 
	coordinates and integrating out 
	the angles.
	Recall that the constants $c$ may change from line to 
	line and they may depend 
	on $k, d$, but not 
	on $N, r_0$.
	
	\noindent
	If $d=2k,$ then the integral
	$$I \leq cN^{2k} \log (2\sqrt{d} N/r_0) = cn \log 
	(2\sqrt{d} N/r_0).$$
	If $2k < d,$ then 
	$$
	I \leq cN^{2k} \big( (N\sqrt{d})^{d-2k} 
	-(r_0/2)^{d-2k}\big)
	\leq c N^d.
	$$
	If $2k > d,$ then 
	$$
	I \leq cN^{2k} \big( (r_0/2)^{d-2k} - 
	(N\sqrt{d})^{d-2k} \big).
	$$
	Treating $d,k$ as constants, we write
	$$
	I \leq c N^{2k} r_0^{d-2k}.
	$$
	
	\paragraph{Case $r_0 < 2\sqrt{d}$.}
	Continuing from \eqref{eq:eigval_bd_I}, write
	\begin{align}
		I \leq c N^{2k} \sum_{i\in \{0,1, .., N-1\}^d : \|i 
		\|_2 \in [r_0, 
			2\sqrt{d}) }  \| i \|_2^{-2k}
		+ 	c N^{2k} \sum_{i\in \{0,1, .., N-1\}^d : \|i \|_2 
		\geq 2 \sqrt{d} }  \| 
		i \|_2^{-2k}
	\end{align}
	From the previous case, the second summation can be 
	upper bound with
	$cn$ if $2k<d$, $cn \log n$ if $2k=d$ and $cN^{2k}$ if 
	$2k > d$.
	In the first summation (in the above display), the 
	number of entries $i$  is at 
	most $(2\sqrt{d})^d$
	and each entry is at most $r_0^{-2k}$.
	Therefore the first term is at most
	$ c N^{2k} (2\sqrt{d})^d r_0^{-2k}$. Putting the two 
	sums together,
	we can verify the stated bounds.
\end{proof}

Following lemma  provides a result analogous to Lemma 
\ref{lem:lap_eigs_ddim} 
for KTF.
\begin{lemma}
	\label{lem:kron_tf_evals}
	Let \smash{$\{ \xi_{i} : i = (i_1,\cdots,i_d) \in [N]^d 
	\}$} be the eigenvalues 
	of 
	\smash{${\kronmatk}^\T \kronmatk$}
	and suppose  $r_0 \in [1, \sqrt{d}N]$. Then,
	$$
	\sum_{ i \in [N]^d  \setminus [k+2]^d}
	\frac{1}{\xi_i^{2}} 
	\leq c
	\begin{cases}
		n & 2 (k+1) < d\\
		n \log n & 2(k+1) = d.
	\end{cases}
	$$
	In the case $2k+2 > d$,
	$$
	\sum_{ i \in [N]^d : \| (i - k - 2)_+ \|_2 \geq r_0 }  
	\frac{1}{\xi_i^{2}} 
	\leq c
	N^{2k+2} r_0^{d-2k-2}
	$$
	Here $c>0$ is a constant that depends on $k, d$ but not 
	on $N, r_0$.
\end{lemma}
\begin{proof}
	Using Lemma~\ref{lem:ktf_gtf_eigval_ineq}, we can write
	\begin{equation}
		\label{eq:ktf_gtf_relation}
		\sum_{ i \in [N]^d : \| (i-k-2)_+ \|_2 \ge r_0} 
		\frac{1}{\xi_i^2}
		\leq d^{2k} \sum_{ i \in [N]^d : \| (i-k-2)_+ \|_2 
		\ge r_0} 
		\frac{1}{\lambda_{i-k-1}^{2k+2}}
		\leq d^{2k}  \sum_{ i \in [N]^d : \| i-1 \|_2 \ge 
		r_0} 
		\frac{1}{\lambda_i^{2k+2}}.
	\end{equation}
	Applying Lemma \ref{lem:lap_eigs_ddim} directly gives 
	the desired 
	result in the case $2k+2 > d$. In 
	the case $2k+2 \leq d$, we get the bound by setting 
	$r_0 = 1$ in 
	\eqref{eq:ktf_gtf_relation}
	and then applying Lemma~\ref{lem:lap_eigs_ddim}.
\end{proof}
\begin{lemma}
	\label{lem:ktf_gtf_eigval_ineq}
	Let \smash{$\{ \xi_{i} : i = (i_1,\cdots,i_d) \in [N]^d 
	\}$} be the 
	eigenvalues of 
	\smash{${\kronmatk}^\T \kronmatk$} for $k \ge 0, d\ge 1, 
	N \ge 1, n = 
	N^d$.
	Let $\alpha_i, i \in [N]$ be the eigenvalues of $L$, 
	the Laplacian of 
	chain graph of length $N$.
	Let 
	$
	\lambda_{i_1,\dots,i_d} = \sum_{j=1}^d \alpha_{i_j}, 
	\quad i \leq N 
	\text{ elementwise}
	$
	with the convention that $\alpha_\ell = 0$ for $\ell 
	\leq 0$. Then 
	\[
	\xi_i \ge d^{-k} \lambda_{i-k-1}^{k+1} \text{ for } i 
	\in [N]^d.
	\]
\end{lemma}
\begin{proof}
	Abbreviate \smash{$D=\dmat_{N,1}^{(k+1)}$}, and let $G$ 
	be  the $k$th order GTF 
	operator defined over a 1d chain of length $N$. Also 
	let $N'=N-k-1$,
	and \smash{$k' =\lfloor (k+1)/2 \rfloor$}.
	Let 
	\begin{itemize}
		\setlength{\itemsep}{0pt}
		\setlength{\parskip}{0pt}
		\item $\beta_\ell, \ell \in [N']$ be the eigenvalues 
		of $DD^\T$
		\item $\gamma_\ell$, $\ell \in [N'']$ be the 
		eigenvalues of $GG^\T$ 
		where 
		\smash{$N'' = N - 1\{ k \text{ is even} \}$}
	\end{itemize}
	$GG^\T$ and $G^\T G$ should have the same 
	nonzero eigenvalues. 
	From the definition of $G$, $G^\T G = L^{k+1}$.
	The first eigenvalue of $L$ is 0 and the rest are 
	nonzero.
	Putting these facts together, we see that
	\begin{equation}
		\label{eq:gtf_lap_eigval_rel}
		\gamma_\ell = \alpha_{\ell+N-N''}^{k+1}
		\text{ and }
		\alpha_\ell^{k+1} \leq \gamma_\ell \quad \text{ for } 
		\ell \in [N''].
	\end{equation}
	Removing the top $k'$ and bottom $k'$ rows of
	$G$ yields $D$, i.e., 
	$$
	D = PG, \quad \text{where} \;\,
	P = \left[\begin{array}{ccc}
		0_{N' \times k'} & I_{N'} & 0_{N'\times k'}
	\end{array}\right].
	$$
	As \smash{$DD^\T = PGG^\T P^\T$}
	and \smash{$PP^\T = I_{N'}$}, Cauchy interlacing theorem 
	(Lemma~\ref{lem:cauchy_interlacing}) 
	tells us that
	\begin{equation}
		\label{eq:cauchy_interlace}
		\gamma_i \leq \beta_i \leq \gamma_{i+N'' - N'}, 
		\quad \text{ for } i \in [N'].
	\end{equation}
	
	\noindent
	Thanks to the Kronecker sum structure, the eigenvalues 
	of 
	\smash{$(\kronmatk)^\T\kronmatk$} are
	$$
	\xi_{i_1,\dots,i_d} = \sum_{j=1}^d \rho_{i_j}, \quad i 
	\in [N]^d,
	$$
	where \smash{$\rho_1, \dots, \rho_N$} denote the 
	eigenvalues of \smash{$D^\T 
		D$}, i.e.,
	\smash{$\rho_1=\cdots=\rho_{k+1}=0$} and 
	\smash{$\rho_{\ell+k+1}=\beta_\ell 
		\ell \in [N']$}.
	Similarly, we can write the 
	eigenvalues of the Laplacian of the $d$-dimensional 
	grid graph as
	$$
	\lambda_{i_1,\dots,i_d} = \sum_{j=1}^d \alpha_{i_j}, 
	\quad i\in [N]^d.
	$$
	For arbitrary $i \in [N]^d$, we can write
	$$
	\xi_{i_1,\dots,i_d} = 
	\sum_{j=1}^d \beta_{i_j-k-1} \geq 
	\sum_{j=1}^d \gamma_{i_j - k - 1} \geq 
	\sum_{j=1}^d\alpha_{i_j-k-1}^{k+1} \geq 
	d^{-k} \lambda_{i_1-k-1,\dots,i_d-k-1}^{k+1},
	$$
	with the convention $\alpha_{\ell} = \beta_{\ell}= 
	\gamma_\ell= 0$ for $\ell 
	\leq 0$.
	The first inequality is due to 
	\eqref{eq:cauchy_interlace},
	the second is due to \eqref{eq:gtf_lap_eigval_rel}, 
	and the third is due to a simple application of 
	Jensen's inequality: 
	$( \frac 1d \sum_{j=1}^d a_i )^k \leq \frac 1d 
	\sum_{j=1}^d a_i^k$
	if $k\geq 1$ and $a \geq 0$ elementwise. 
\end{proof}

\begin{lemma}[Cauchy Interlacing theorem]
	\label{lem:cauchy_interlacing}
	Let $A$ be an $n\times n$ symmetric matrix, $P \in 
	\R^{m\times n}$ be 
	an orthogonal 
	projection matrix (satisfying $PP^\T = I_m$) with $m\leq 
	n$ and define 
	$B = P A P^\T$. Let 
	$\alpha_1 \leq 
	\alpha_2 \leq  \dots \leq  	\alpha_n$ be the 
	eigenvalues of $A$ and 
	$\beta_1 \leq \beta_2 \leq \dots 
	\leq \beta_m$ be the eigenvalues of 
	$B$. Then 
	$$
	\alpha_i \leq \beta_i \leq \alpha_{i+n-m}, \quad \text{ 
	for } i \in [m].
	$$
\end{lemma}

\begin{lemma}
	\label{lem:ktf_inv_eigval_sum_lower_bd}
Let \smash{$\{ \xi_{i} : i = (i_1,\cdots,i_d) \in [N]^d 
		\}$} be 
	the 
	eigenvalues of 
	\smash{${\kronmatk}^\T \kronmatk$} 
	for $k \ge 0, d\ge 1, 	N \ge 1, n = 	N^d$.
	Suppose $s = 1/2$ and $a>0$.
	Then
	\begin{equation*}
		\sum_{i \in [N]^d} \frac{1}{\xi_i + a} \ge c n \log \big( 1 + \pi^{2k+2} 
		a^{-1}\big)
	\end{equation*}
for a constant $c$ that depends only on $k, d$.
\end{lemma}
\begin{proof}[Proof of Lemma~\ref{lem:ktf_inv_eigval_sum_lower_bd}]
From \eqref{eq:rho_interlacing} and the inequality $\sin x \leq x$ for $x\geq 
0$, for 
any \smash{$i \in [N]^d$},
	\begin{align*}
		\xi_i = \sum_{j=1}^d \rho_{i_j} \leq 
		\sum_{j=1}^d 4^{k+1} \sin^{2k+2} \frac{\pi(i_j  - 1)}{2N}
		\leq 
		\pi^{2k+2} n^{-2s} \| i - 1 \|_{2k+2}^{2k+2}
		\leq 
		\pi^{2k+2} n^{-2s} \| i - 1 \|_2^{2k+2}.
	\end{align*}
With this inequality,
\begin{align}
	\nonumber
	\sum_{i \in [N]^d} \frac{1}{\xi_i + a} 
	&\ge 
	\sum_{i \in [N]^d} \frac{1}{\pi^{2k+2} n^{-2s}\| i - 1 \|_2^{2k+2} + a}\\
	&
	\label{eq:ktf_inv_eigval_sum_int_bd}
	\ge
	c \int_{r=0}^{N} \frac{1}{\pi^{2k+2} n^{-2s} r^{2k+2} + a} r^{d-1} \; dr.
\end{align}
In the second inequality is obtained as follows.
Consider axis-parallel unit cubes with corners located at integer coordinates.
Let $A_i \subset \R^d$ be the cube with its farthest corner from origin 
located at $i$, for \smash{$i \in [N]^d$}. Clearly,
$$
	\frac{1}{\pi^{2k+2} n^{-2s}\| i - 1 \|_2^{2k+2} + a} 
	\geq 
	\int_{A_i}	\frac{1}{\pi^{2k+2} n^{-2s}\| x \|_2^{2k+2} + a} \; dx.
$$
Next observe that the set
\smash{$\{ x \in \R^d : \| x \|_2 \leq N, x \geq 0 \}$} 
is contained in the cube \smash{$\{ x \in \R^d : \| x \|_\infty \leq N, x 
\geq 0 \}$} 
The former set is the non-negative orthant of the $\ell_2$ ball of radius $N$ 
in \smash{$\R^d$}. So, for radially symmetric functions $f$, integral of $f$ 
over this set is $2^{-d}$ times its integral over the $\ell_2$ ball. This 
justifies \eqref{eq:ktf_inv_eigval_sum_int_bd} after a change to polar 
coordinates.
In the integral \eqref{eq:ktf_inv_eigval_sum_int_bd}, noting that 
$s = 1/2, 2k+2 = d$, put $u = r^d$ to get
\[
		\sum_{i \in [N]^d} \frac{1}{\xi_i + a} 
		\ge 
		c \int_0^{n} \frac{1}{\pi^{2k+2} u/n + a} \; du
		=
		c n \log \big( 1 + \pi^{2k+2} a^{-1}\big).\qedhere
\]
\end{proof}

\begin{lemma}
	\label{lem:spectral_dist_kron}
	Let \smash{$\{ \xi_{i} : i = (i_1,\cdots,i_d) \in [N]^d 
	\}$} be 
	the 
	eigenvalues of 
	\smash{${\kronmatk}^\T \kronmatk$} for $k \ge 0, d\ge 1, 
	N \ge 1, n = 
	N^d$.
	Let $\alpha_i, i \in [N]$ be the eigenvalues of $L$, 
	the Laplacian of 
	chain 
	graph of length $N$.
	Define 
	$$
	F(t) = \frac{1}{n} \sum_{i \in [N]^d} 1\{ \lambda_i 
	\leq t \}, \quad 
	\text{ for } t \in [0,\lambda_n].
	$$ 
	Then there exist constants \smash{$c_1, c_2, c_3 >0$} 
	independent of 
	$n, t$ 
	such 
	that
	\begin{align*}
		c_1 t^{\frac{d}{2k+2}} \leq F(t) \leq  c_2 + c_3 
		t^{\frac{d}{2k+2}}
	\end{align*}
	for all \smash{$t\in [0,\lambda_n]$}.
\end{lemma}
\begin{proof}[Proof of Lemma \ref{lem:spectral_dist_kron}]
	Using the notation in the proof of 
	Lemma~\ref{lem:kron_tf_evals},
	$$
	\lambda_{i_1,\dots,i_d} = \rho_{i_1} + \dots + 
	\rho_{i_d}, \quad \text{ 
		for } 
	(i_1,\dots,i_d) \in [N]^d
	$$ 
	where \smash{$\rho_\ell = \beta_{\ell - k - 1}, \ell 
	\in [N]$} with the 
	convention that \smash{$\beta_\ell = 0$ for $\ell \leq 
	0$}.
	From \eqref{eq:cauchy_interlace}, 
	\eqref{eq:gtf_lap_eigval_rel} and the 
	fact 
	that the 
	eigenvalues of chain Laplacian are given by $4 \sin^2 
	\frac{\pi(\ell-1)}{2N}$ 
	for 
	$\ell\in[N]$, we have
	\begin{align}
		\label{eq:rho_interlacing}
		\left( 4 \sin^2 \frac{\pi(\ell-k-2)_+}{2N} 
		\right)^{k+1} 
		\leq  \rho_i \leq 
		\left( 4 \sin^2 \frac{\pi(\ell-1)}{2N} \right)^{k+1}, 
		\quad 
		\text{ for } 
		\ell \in [N]
	\end{align}
	where $(x)_+ = \max\{x,0\}$ for $x\in\R$.
	The upper bound can be argued as follows.
	\begin{align*}
		nF(t) &=
		\sum_{i\in [N]^d} 1 \Big\{   \lambda_{i_1,\dots,i_d} 
		\leq t 
		\Big\}\\
		&=
		\sum_{i\in [N]^d} 1 \Big\{   \sum_{j=1}^d \rho_{i_j} 
		\leq t 
		\Big\}\\
		&\leq 
		\sum_{i\in [N]^d} 1\Big\{   \sum_{j=1}^d 4^{k+1} 
		\sin^{2k+2} 
		\frac{\pi 
			(i_j-k-2)_+}{2N}  \leq t \Big\}\\
		&\leq 
		\sum_{i\in [N]^d} 1\Big\{   \sum_{j=1}^d 
		\left(\frac{\pi}{2}\right)^{2k+2} 
		(i_j-k-2)_+^{2k+2}  \leq t N^{2k+2} \Big\}
	\end{align*}
	In the third line, we use \eqref{eq:rho_interlacing} 
	and in the fourth 
	line, we use the fact that $\sin x \geq x/2$ for $x\in 
	[0,\pi/2]$.
	Observe that 
  \[
		\left(\frac{\pi}{2}\right)^{2k+2} \sum_{j=1}^d 
		(i_j-k-2)_+^{2k+2}  
		\leq tN^{2k+2}
		\Rightarrow 
		\| i \|_\infty \leq k + 2 + \frac{2}{\pi} N 
		t^{\frac{1}{2k+2}}.
  \]
	Applying this fact to the previous bound on $n F(t)$, 
	\begin{align*}
		n F(t)
		&\leq 
		\sum_{i\in [N]^d} 1 \Big\{  \| i \|_\infty \leq k + 2 
		+ \frac{2}{\pi} N 
		t^{\frac{1}{2k+2}} \Big\} \\
		&\leq 
		\Big( k + 2 + \frac{2}{\pi} N t^{\frac{1}{2k+2}} 
		\Big)^{d}
		\leq 2^{d-1}(k+2)^d + 2^{2d-1} \pi^{-d} n 
		t^\frac{d}{2k+2}.
	\end{align*}
	where in the last inequality we used the fact that 
	\smash{$(a+b)^d \leq 
		2^{d-1}(a^d + b^d)$} for $a, b \ge 0, d\geq 1$.
	
	The lower bound can be derived as follows. Certainly, 
	$F(t) \geq F(0) = 
	\kappa/n$ for $t \geq 0$. We can write
	\begin{align*}
		nF(t) &=
		\sum_{i\in [N]^d} 1 \big\{   \lambda_{i_1,\dots,i_d} 
		\leq t 
		\big\}\\
		&=
		\sum_{i\in [N]^d} 1 \Big\{   \sum_{j=1}^d \rho_{i_j} 
		\leq t 
		\Big\}\\
		&=
		\sum_{i\in [N]^d} 1 \Big\{   \sum_{j=1}^d 4^{k+1} 
		\sin^{2k+2} 
		\frac{\pi(i_j-1)}{2N}  \leq t \Big\}\\
		&\geq 
		\sum_{i\in [N]^d} 1 \Big\{   \sum_{j=1}^d \pi^{2k+2} 
		(i_j-1)^{2k+2}  \leq t 
		N^{2k+2} \Big\}\\
		&= \sum_{i\in [N]^d} 1 \Big\{   \sum_{j=1}^d \| i - 1 
		\|_{2k+2} 
		\leq 
		r \Big\}
	\end{align*}
	where \smash{$r = \frac 1\pi N t^{\frac{1}{2k+2}} $}.
	In the third line, we used \eqref{eq:rho_interlacing} 
	and in the fourth 
	line, 
	we used the fact that $\sin x \leq x$ for $x\geq 0$. 
	Note that, we can inscribe a cube
	\smash{$\{ i:  \| i - 1 \|_\infty \leq r 
	d^{\frac{-1}{2k+2}} \}$}
	in the $\ell_{2k+2}$ body
	\smash{$\{ i : \| i - 1 \|_{2k+2} \leq r \}$}
	and the cube contains \smash{$\big(1+ \lfloor r 
	d^{\frac{-1}{2k+2}} \rfloor 
		\big)^d$} lattice points in \smash{$[N]^d$}.
	Therefore, continuing to bound from the previous 
	display,
	$$
	n F(t) \geq \big( 1 + \lfloor \frac{1}{\pi} 
	d^{\frac{-1}{2k+2}} N 
	t^{\frac{1}{2k+2}} \rfloor \big)^d
	\ge
	\frac{1}{\pi^d} d^{\frac{-d}{2k+2}} n 
	t^{\frac{d}{2k+2}}.
	$$
	where in the last inequality we used the fact 
	\smash{$(1 + \lfloor x \rfloor)^d \geq x^d$} for 
	\smash{$x \ge 0$}.
\end{proof}

\section{Fast algorithm for degrees of freedom} 
\label{app:df}

From \eqref{eq:ktf_df_unbiased}, given a KTF 
estimate \smash{$\htheta$}, \smash{$\nuli(D_{-A})$} is an unbiased estimate of 
degrees of 
freedom of KTF where $A$ denotes the set of rows 
$r$ in $D$ for which $(D\hat\theta)_r \neq 0.$
We give an algorithm to compute \smash{$\nuli(D_{-A})$} in $O(ndk)$ time.
This is described in Algorithm \ref{alg:df}, with Algorithm
\ref{alg:subroutines} describing its core rubroutines. The notation we use is as
follows: $N' = N-k-1,$ and $w \in \R^{k+2}$ is the order $(k+1)$ difference
vector.  

\begin{algorithm}[htb]
\textbf{Input:}  fitted values $\theta\in \R^n$, trend filtering order $k\geq 
0$\\
\textbf{Output:} estimate of degrees of freedom
\begin{enumerate}[noitemsep, leftmargin=*,align=left]
\item  \textbf{struct} Piece: set = false, knowns = 0, start, end $\in [N]^d$
\item pieces : Piece[],  pieces-containing[i]: Piece[], set[i] : bool for $i\in 
[N]^d$
\item for $d' \in [d]$ for $i'\in [N]^{d-1}$: 
\begin{enumerate}[noitemsep]
	\vspace{-0.5em}
	\item $i \mathord{\leftarrow} i'$  with an extra $1$ inserted before 
$d'$th entry:\\
	$i_j \mathord{=} i'_j$ for $j\mathord{<}d'$, 
	$i_j \mathord{=} 1$ for $j \mathord{=} d'$, $i_j \mathord{=} 
i'_{j\mathord{-}1}$ for $j\mathord{>}d'$
	\item	 \textsc{Make-Polynomial-Pieces}$(\theta, i, d')$
\end{enumerate}

\item df = 0

\item for $p$ in pieces:
	\renewcommand{\theenumi}{\alph{enumi}}
	\begin{enumerate}[noitemsep] 
	\item if $p.set:$ continue
	\item df $\mathord{+}\mathord{=} \max\{0, \min\{p.length, 
k\mathord{+}1\} - p.\mathrm{knowns} \}$
	\item \textsc{Spread}$(p)$
	\end{enumerate}

\item return df
\end{enumerate}
\caption{\textsc{Degrees-of-Freedom}$(\theta,k)$}
\label{alg:df}
\end{algorithm}

\begin{algorithm}[htb]
\textsc{Spread-Vertex}$(i)$

	$\;$ for piece $q$ in pieces-containing$[i]:$
		\begin{enumerate}[noitemsep]
			\vspace{-1em}
			\item if $q.set$ continue
			\item $q.\mathrm{knowns}\mathord{+}\mathord{+}$
			\item if $q.\mathrm{knowns} \mathord{>} k:$ 
\textsc{Spread}$(q)$
		\end{enumerate}

\textsc{Spread}$(p: \text{Piece})$ $\;$ 

$p.set$ = true 

$\;$ for vertex $i$ on the line $[p.start, p.end]$:
	
\begin{enumerate}[noitemsep]
	\vspace{-1em}
		\item if set$[i]:$ continue
		\item set$[i] \leftarrow$ true
		\item \textsc{Spread-Vertex}($i$)
	\end{enumerate}

\textsc{Make-Polynomial-Pieces}$(\theta, i \in [N]^d , d'\in [d])$
makes polynomial pieces on the line containing $i$ along axis $d'$ 
\begin{enumerate}[noitemsep]
	\item $a_j \leftarrow \theta[i$ with $i_{d'}=j]$ for $j\in[N]$
	\item while $j\leq N$
	\renewcommand{\theenumi}{\alph{enumi}}
	\begin{enumerate}[noitemsep,leftmargin=*,align=left]
		\item start = $j$, end = $j$
		\item while ($\smash{j\mathord{\leq} N'}$  and $\smash{\langle 
w, a[j\mathord{:}j\mathord{+}k\mathord{+}1] \rangle \mathord{=}0})$\\
					$ j\mathord{+}\mathord{+}$
		\item if $j \mathord{\neq} $ start: end $\leftarrow 
j\mathord{+}k$
		\item Piece p(i with $i_{d'} $= start, i with $i_{d'}$ = end)
		\item for $j'$ in [start, end]:  pieces-containing$[i $ with 
$i_{d'} = j']$.add(p)
		\item pieces.add(p)
		\end{enumerate}
\end{enumerate}
\caption{Subroutines used in Algorithm \ref{alg:df}} 
\label{alg:subroutines}
\end{algorithm}


\subsection{Time complexity and correctness of the algorithm}
Let $A$ denote the set of rows $r$ in $D$ for which $(D\hat\theta)_r \neq 0$. 
Denote the null space of $D_{-A}$ with $\cN$.

In step 3 of \textsc{Degrees-of-Freedom}, we find line segments on the lattice 
where 
$\theta$ is a $k$th degree polynomial. In a bit more detail, for each straight 
line in the lattice between opposing faces, we find segments along the line 
where $\theta$ is a $k$th degree polynomial. 
We call these line segments (polynomial) \emph{pieces} in the algorithm.
In 2d, \textsc{Make-Polynomial-Pieces} is called on rows and columns 
separately, and a piece is a part of a row or a column.
An important characterization of the null space $\cN$ is the following:
\begin{center}
A $\theta \in \cN$ iff it is a $k$th degree polynomial on all the pieces found 
in step 3.
\end{center}

If a piece has fewer than $k+2$ elements, then any $\theta$ is trivially a 
$k$th degree polynomial on the piece. 
Otherwise, $k+1$ values on the piece determine a polynomial on the piece. 

We pretend to build a vector $\eta \in \cN$ by making sure that $\eta$ is a 
$k$th degree polynomial on the pieces. 
In step 5, we pretend to set the values of $\eta$ in a piece $p$. The number of 
new entries in $\eta$ required to determine a polynomial on piece $p$ is shown 
in 5(b). 
Once the new entries are picked arbitrarily, all the values on the piece are 
determined via the 
constraints in $D_{-A} \eta = 0.$ 
Then we propagate the values from this piece to other adjoining pieces in a 
depth-first fashion.
By the end of the procedure, \texttt{df} accumulates the total number of free 
parameters that we can use to build such a $\eta$.
The dimension of $\cN$ is equal to the number of free parameters in the 
algorithm.

\paragraph{Time complexity.}
It takes $O(nd(k+1))$ time to make the polynomial pieces in line 3 of 
\textsc{Degrees-of-Freedom}. The number of pieces is at most $nd(k+1).$ 
Therefore the for loop in line 5 is run at most as many times.
\textsc{Spread} is called on a piece exactly once and
\textsc{Spread-Vertex} is called on a node exactly once.
A node is contained in a maximum of $(k+2)d$ pieces. 
Therefore, the total time complexity is $O( nd (k+2)).$

\paragraph{Correctness.}
Suppose the number of free parameters returned by the algorithm is $f$. 
Given the values at the $f$ free nodes $F\subset [n]$,
the values are determined at all $n$ nodes.
Further this mapping from $\R^f \mapsto \R^n$ is linear.  
Therefore there exists a matrix $C$ with size $n\times f$ such that
$Cb \in \cN$ for any $b\in \R^f.$ 
Further, $(Cb)_F$ is a permutation of $b$, because the values at free nodes are 
not modified by $C$. Therefore, there are $f$ rows in $C$ corresponding to the 
free nodes $F$, which when vertically stacked together form a permutation of 
$f\times f$ identity matrix. 
Therefore, the column span of $C$ has dimension $f$. Hence $f \leq 
\mathrm{dim}(\cN).$
Conversely, consider any $\eta \in \cN$. Given the entries $b$ of $\eta$ at the 
locations of free parameters, then the rest of the entries of $\eta$ are 
determined by $\eta = Cb.$ Therefore $\eta$ must lie in the column span of $C$. 
Therefore $f = \mathrm{dim}(\cN).$

\section{More details on optimization algorithms}
\label{app:more_opt}

\paragraph{Generic quadratic programming on the dual.}
Recall that KTF solves the following convex optimization problem:
\begin{equation}
\label{eq:ktf_primal}
\htheta = \argmin_{\theta \in \R^n} \; \frac{1}{2} \|y-\theta\|_2^2 +
\lambda  \|\genmat \theta \|_1.
\end{equation}
with \smash{$\genmat = \kronmatk$}.
The corresponding Lagrange dual problem is
\begin{equation}\label{eq:ktf_dual}
\begin{aligned}
\max_{u} &\;\;- \frac{1}{2}u \genmat \genmat^\T u +  y^\T \genmat^\T u\\
\text{subject to }&\;\; -\lambda \leq u \leq \lambda.
\end{aligned}
\end{equation}
Note that the dual problem is a standard quadratic program (QP) and can be 
solved using the interior 
point method (IPM) to high precision. Then the primal solution can be 
constructed using  $\htheta = 
y - \genmat^\T u^*$ using the optimal solution $u^*$. In practice, IPM takes 
only a few iterations to 
converge,\footnote{In theory it could take up to $O(n^{1/2})$ iterations to 
converge.} but each 
iteration involves solving a linear system. This linear system is sparse since 
$\genmat$ is sparse --- it 
has only $O(dkn)$ non-zero elements. However, the condition number of the 
linear system grows  
as the weights on the barrier function increase, which makes it difficult to 
exploit the sparsity using 
methods such as preconditioned conjugates gradient method.  On the other hand, 
direct solvers 
such as Gaussian elimination and Cholesky decomposition can take up to 
$O(n^3)$. Sometimes this 
can be improved by exploiting the banded structure of the linear system, we 
will describe a 
particular version of the interior point method using logarithmic-barrier 
function.

\paragraph{Primal-dual interior point method.}
The primal-dual version of the interior point solver proposed by 
\citep{kim2009trend} for $\ell_1$ trend filtering can be straightforwardly 
applied to any generalized lasso problem, including KTF.  The main idea is to 
trace a ``central path'' using Newton's method with an increasing weights $t$ 
on the logarithmic barrier functions. The computation is dominated by computing 
the search direction of the Newton step, which boils down to solving the 
following system of linear equations
\begin{equation}
\begin{bmatrix}
\genmat \genmat^\T & I & -I\\
I & J_1 & 0\\
-I & 0 & J_2
\end{bmatrix} \begin{bmatrix}
\Delta u\\
\Delta\mu_1 \\
\Delta\mu_2
\end{bmatrix}  =  - \begin{bmatrix}
\genmat \genmat^Tu - \genmat y + \mu_1 - \mu_2\\
f_1 + (1/t)\mu_1^{-1} \\
f_2 + (1/t)\mu_2^{-1}
\end{bmatrix}
\end{equation}
where $\mu_1,\mu_2\in \R^{m}$ are the dual variables of the dual problem 
\eqref{eq:ktf_dual},  $f_1 = u - \lambda \mathbf{1}$, $f_2 = -u 
-\lambda\mathbf{1}$, $J_i = \diag{\mu_i}^{-1}\diag(f_i)$ are diagonal matrices 
and $\mu_i^{-1}$ denotes entrywise inversion.  Following the derivation of 
\citep{kim2009trend}, we can further eliminate $\Delta\mu_1$ and  $\Delta\mu_2$ 
and solve a linear system of the form
\begin{equation}\label{eq:pdipm_key_lin_system}
(\genmat \genmat^\T - J_1^{-1} J_2^{-1})\Delta u  =  - (\genmat \genmat^\T u - 
\genmat y -(1/t)f_1^{-1} + 
(1/t)f_2^{-1}).
\end{equation}
and then construct the remainder of the solutions using
\begin{align*}
\Delta \mu_1 = -(\mu_1 +(1/t)f_1^{-1} + J_1^{-1} \Delta u ),\\
\Delta \mu_2 = -(\mu_2+(1/t)f_2^{-1} - J_2^{-1} \Delta u ).
\end{align*}
Unlike in the trend filtering problems in 1D where 
\eqref{eq:pdipm_key_lin_system} is a banded linear system with a bandwidth at 
most $2k+3$, in a $d$-dimensional grid, the linear system is the following:
\begin{small}
\[
 \begin{bmatrix}
\big( D_{\mathrm{1d}} D_{\mathrm{1d}}^\T\big) 
 \otimes I \otimes ... \otimes I, & D_{\mathrm{1d}}\otimes  
D_{\mathrm{1d}} ^\T \otimes I \otimes ...\otimes I,& \cdots, 
& 
 D_{\mathrm{1d}} \otimes I \otimes ...\otimes I\otimes  
D_{\mathrm{1d}} ^\T,\\\
D_{\mathrm{1d}} ^\T  
\otimes D_{\mathrm{1d}}\otimes I \otimes ...\otimes I, &  I 
\otimes \big( D_{\mathrm{1d}} D_{\mathrm{1d}} 
^\T\big) \otimes I \otimes ... \otimes I,& \cdots, &I \otimes  
D_{\mathrm{1d}} \otimes I\otimes...\otimes I\otimes  
D_{\mathrm{1d}}^\T\\
\vdots & \vdots & \ddots& \vdots\\
D_{\mathrm{1d}}^\T \otimes I \otimes ...\otimes I\otimes  
D_{\mathrm{1d}}&I\otimes D_{\mathrm{1d}}^\T 
\otimes I \otimes ...\otimes I\otimes  D_{\mathrm{1d}}&\cdots,& 
I\otimes... \otimes I  \otimes \big(D_{\mathrm{1d}} 
D_{\mathrm{1d}}^\T\big) 
\end{bmatrix}
- J_1^{-1}J_2^{-1}
\]
\end{small}
where \smash{$D_{\mathrm{1d}} = \dmat_{N,1}^{(k+1)}$}. 
This is still sparse, structured, but the 
bandedness 
is on the order of $O(n^{1-1/d} +  k^2)$. Moreover, the above matrix is not 
full rank, and the condition number of the linear system blows up as the dual 
variables $\mu_1,\mu_2$ converge to $0$ with $t\rightarrow \infty$.

%

\paragraph{Proximal Dykstra's algorithm.}
Proximal Dykstra's algorithm is an operator-splitting method for solving problems of the form
\begin{equation}
\label{eq:dykstra_primal}
\minimize_{\theta \in \R^n} \; \frac{1}{2} \|y-\theta\|_2^2 + r_1(\theta) + r_2(\theta) + ... + r_d(\theta)
\end{equation}
where $r_1,...,r_d$  are convex but possibly non-smooth functions.  We can clearly see that the regularizer in KTF decomposes into this form with
%
\[
r_i(\theta) = \sum_{x \in Z_{n,d}} | (\Delta_{x_j^{k+1}} \theta) (x) |
\]
in the notation of Section~\ref{sec:ktf}.
The proximal Dykstra algorithm \citep[see, e.g.,][]{tibshirani2017dykstra} initializes $\theta^{(0)} = y,z^{(-d+1)}=...=z^{(0)}=0$ and then iteratively applies the following update rule for $t=1,2,3,...$:
\begin{align*}
\theta^{(t)} &= \prox_{r_{t\;\mathsf{ mod } \;d +1}}(\theta^{(t-1)} + z^{(t-d)})\\
z^{(t)} &= \theta^{(t-1)} + z^{(t-d)} - \theta^{(t)}.
\end{align*}
where $\cdot \;\mathsf{ mod }\;\cdot$ is the modulo operator, and the proximal operator
$$\prox_r(u) =  \argmin_{\theta} \half \|u-\theta\|^2 +  r(\theta).$$
Note that this is equivalent to a cyclic block coordinate descent in the dual.

For KTF, each proximal problem can be parallelized \citep{barbero2018modular}. Specifically, on a $d$-dim regular grid, the proximal operator of $r_i$ further splits into solving $O(n^{1-1/d})$ 1D-trend filters of size $n^{1/d}$ in parallel.  Each subproblem can be solved efficiently in $O(n^{1.5/d})$ time  with the primal-dual interior point method for $k\ge 1$ \citep{tibshirani2014adaptive} and in linear time when $k=0$ using dynamic programming \citep{johnson2013dynamic}.

\paragraph{Douglas-Rachford splitting.}
Another operator-splitting method for solving KTF is through the Douglas-Rachford (DR) algorithm \citep{eckstein1992douglas}. For simplicity, we will focus our discussion on the case of 2D grids. The DR algorithm generically solves the following unconstrained problem:
\begin{equation}\label{eq:general_DR}
\minimize_\theta  f(\theta) + g(\theta)
\end{equation}
for convex functions $f,g$. The update rules include initializing an auxiliary variable $z^{(0)} = y$ and applying the following for  $t=0,1,2,...$:
\begin{align*}
\theta^{(t+1)}  &= \prox_f(z^{(t)})\\
z^{(t+1)}  &=  z^{(t)}  +  \prox_g(2 \theta^{(t+1)} -  z^{(t)}) - \theta^{(t+1)}.
\end{align*}
There are multiple ways of applying this to our problem.
We apply the DR algorithm to the dual of the following reformulation according to \citep[Algorithm 9]{barbero2018modular}:
\begin{equation}
\label{eq:ktf_dr_dual}
\minimize_{\theta \in \R^n} \; \frac{1}{2} \|\theta\|_2^2  + \left(\lambda 
\|\dmatk_N \otimes I   \theta \|_1  - \langle \theta, y \rangle   \right) +  
\left( \lambda \|I \otimes \dmatk_N   \theta \|_1 
\right).
\end{equation}
We refer  interested readers to \citep{barbero2018modular} for the derivation of the dual and the conversion of the problem into one that resembles \eqref{eq:general_DR}. Ultimately, the proximal operator of the conjugate function (an indicator on a certain polytope) can be evaluated using the proximal operator of the $r_1$ and $r_2$ as in the proximal Dykstra updates via the Moreau decomposition:
$$
\prox_{r_{i}}(u) + \prox_{r_{i}^*}(u) = u.
$$
In other words,  the Douglas-Rachford algorithm enjoys the same computational benefits of the proximal Dykstra's algorithm  as each proximal operator evaluation involves only solving 1D trend filtering problems in parallel. 


\end{document}